\documentclass{article}
\usepackage{arxiv-a4}

\usepackage{xkeyval}
\usepackage{xstring}
\usepackage{iftex}
\usepackage{microtype}
\usepackage{etoolbox}
\usepackage{booktabs}
\usepackage{refcount}
\usepackage{totpages}
\usepackage{environ}
\usepackage{setspace}
\usepackage{textcase}
\usepackage{natbib}
\usepackage[pageanchor=false,colorlinks=true,allcolors=black,bookmarksnumbered,unicode]{hyperref}
\usepackage{graphicx}
\usepackage[prologue]{xcolor}

\usepackage{latexsym}
\usepackage{amssymb}
\usepackage{amsmath}
\usepackage{amsthm}

\usepackage{amssymb}
\usepackage{bbold}
\usepackage{booktabs}
\usepackage{multirow}
\usepackage{siunitx}
\usepackage{framed}
\usepackage{graphicx}
\usepackage{subcaption} \usepackage{color}

\usepackage{xpatch}
\usepackage{thmtools}
\usepackage{apptools}

\makeatletter
\xpatchcmd{\thmt@restatable}{\csname #2\@xa\endcsname\ifx\@nx#1\@nx\else[{#1}]\fi}{\IfAppendix{\csname #2\@xa\endcsname}{\csname #2\@xa\endcsname\ifx\@nx#1\@nx\else[{#1}]\fi}}{}{} \makeatother

\newcommand{\labelT}[1]{\label{#1}}

\usepackage[normalem]{ulem}
\usepackage[inline]{enumitem}

\usepackage{lineno}   \usepackage{amsmath}  \usepackage{etoolbox} 

\newcommand*\linenomathpatch[1]{\cspreto{#1}{\linenomath}\cspreto{#1*}{\linenomath}\csappto{end#1}{\endlinenomath}\csappto{end#1*}{\endlinenomath}}

\linenomathpatch{equation}
\linenomathpatch{gather}
\linenomathpatch{multline}
\linenomathpatch{align}
\linenomathpatch{alignat}
\linenomathpatch{flalign}

\usepackage{adjustbox}

\usepackage{tikz}
\usepackage{pgfplots}
\usepackage{pgfplotstable}
\usetikzlibrary{arrows.meta}
\graphicspath{{./figures/}}

\usepackage[ruled,vlined,linesnumbered]{algorithm2e}
\SetKwProg{Fct}{Fct}{}{}
\let \oriAlgorithm=\algorithm
\renewcommand{\algorithm}{
  \oriAlgorithm  \DontPrintSemicolon
}

\mathchardef\mhyphen="2D 
\newcommand\simplex[1]{\Delta(#1)}
\def\cS{\mathcal{S}}
\def\cA{\mathcal{A}}

\def\cI{\mathcal{I}}

\def\obs{\text{\sc obs}}
\def\ag{\text{\sc ag}}
\def\pred{\text{pred}}
\def\target{\xi}
\def\Target{\Xi}
\usepackage{cancel}

\def\type{\theta}
\def\Type{\Theta}
\def\target{\psi}
\def\Target{\Psi}
\def\targetOf{\phi}
\def\candidates{C}

\def\Robs{R_{\obs}}
\def\Rag{R_{\ag}}
\def\Rleg{R_{\text{leg}}}
\def\Rexp{R_{\text{exp}}}
\def\Rpred{R_{\bullet\mhyphen\pred}}
\def\RApred{R_{a\mhyphen\pred}}
\def\RSpred{R_{s\mhyphen\pred}}
\def\piobs{\pi_{\obs}}

\usepackage{xspace}
\def\ie{{\em i.e.}\@\xspace}
\def\eg{{\em e.g.}\@\xspace}
\def\cf{{\em cf.}\@\xspace}

\newcommand{\eqdef}     {\triangleq}

\DeclareMathOperator*{\argmax}{arg\,max}
\newcommand{\upb}[1]{\overline{#1}}
\newcommand{\lob}[1]{\underline{#1}}

\newcommand{\norm}[1]{\|#1\|}

\usepackage[noabbrev,capitalise]{cleveref}

\newtheorem{theorem}{Theorem}

\newtheorem{corollary}[theorem]{Corollary}
\newtheorem{proposition}[theorem]{Proposition}

\DeclareRobustCommand{\abbrevcrefs}{\Crefname{appendix}{App.}{Apps.}\Crefname{section}{Sec.}{Secs.}\Crefname{equation}{Eq.}{Eqs.}\Crefname{figure}{Fig.}{Figs.}\Crefname{algorithm}{Alg.}{Algs.}\Crefname{tabular}{Tab.}{Tabs.}\Crefname{lemma}{Lem.}{Lems.}\Crefname{corollary}{Cor.}{Cors.}\Crefname{theorem}{Thm.}{Thms.}\Crefname{proposition}{Prop.}{Props.}\Crefname{line}{L.}{Ls.}\Crefname{item}{It.}{Its.}\crefname{appendix}{app.}{apps.}\crefname{section}{sec.}{secs.}\crefname{equation}{eq.}{eqs.}\crefname{figure}{fig.}{figs.}\crefname{algorithm}{alg.}{algs.}\crefname{tabular}{tab.}{tabs.}\crefname{lemma}{lem.}{lems.}\crefname{corollary}{cor.}{cors.}\crefname{theorem}{thm.}{thms.}\crefname{proposition}{prop.}{props.}\crefname{line}{l.}{ls.}\crefname{item}{it.}{its.}}

\DeclareRobustCommand{\Cshref}[1]{{\abbrevcrefs\Cref{#1}}}

\usepackage{xcolor}
\newcommand{\persComment}[3]{
	\ifmmode
	\text{\textcolor{#3}{[#2] #1}}
	\else
	\textcolor{#3}{[#2] \em #1}
	\fi
}

\def\largBande{1}

\newcommand{\ObservationLine}[2][]{

\addplot[#1,fill,white,opacity=0.5] coordinates {
  (#2+1-\largBande,1)
  (#2+1-\largBande,0)
  (#2+1,0)
  (#2+1,1)
  };

\addplot[#1] coordinates {
  (#2+1-\largBande,1)
  (#2+1-\largBande,0)
  (#2+1,0)
  (#2+1,1)
  };

}

\newcommand{\textBelief}[3]{
  \node at (rel axis cs:#2/#3+0.5/#3,0.1) {#1};
}

\title{ Observer-Aware Probabilistic Planning \linebreak
  under Partial Observability
}

\renewcommand{\headeright}{Preprint - \today}
\author{
  Salomé Lepers$^1$ \and
  Vincent Thomas$^1$
  \and
  Olivier Buffet$^1$
  \\
\begin{minipage}{.99\linewidth}
  \small \centering
  $^{(1)}$Université de Lorraine, CNRS, Inria, LORIA, F-54000 Nancy, France
  \end{minipage}
}

\begin{document}

\maketitle

\begin{abstract}
  In this article, we are interested in planning problems where the agent is aware of the presence of an observer, and where this observer is in a partial observability situation.
The agent has to choose its strategy so as to optimize the information transmitted by observations.
Building on observer-aware Markov decision processes (OAMDPs), we propose a framework to handle this type of problems and thus formalize properties such as legibility, explicability and predictability.
This extension of OAMDPs to partial observability can not only handle more realistic problems, but also permits considering dynamic hidden variables of interest.
These dynamic target variables allow, for instance, working with predictability, or with legibility problems where the goal might change during execution.
We discuss theoretical properties of PO-OAMDPs and, experimenting  with benchmark problems, we analyze HSVI's convergence behavior with dedicated initializations and study the resulting strategies.
\end{abstract}

\section{Introduction}

As explained by \citet{DBLP:journals/expert/KleinWBHF04}, efficient and safe human-agent collaboration requires behaviors that carry information such as intentions, abilities, current status or upcoming actions (see also \citep{SchReiHeyEve-acm21,DBLP:journals/ijrr/SingamaneniBMGSSA24}).
Various works in manipulation or mobile robotics try to derive behaviors with such properties  \citep{DBLP:conf/rss/DraganS13,DraLeeSri-hri13,FisacEtAl-AFR20,DBLP:journals/arobots/BeetzSEFKKMR10,DBLP:conf/iros/Angelopoulos0N022}.
An alternative is to explicitly communicate through language with the human \citep{DBLP:conf/ro-man/GongZ18}.

Here we consider an agent (robot or otherwise) observed by a passive human, as in \Cref{figOAMDP} (left).
In this setting, \citet{DBLP:conf/aips/ChakrabortiKSSK19,DBLP:journals/corr/abs-1811-09722}
build on previous work to derive a taxonomy of the concepts behind such information communication through the behavior.
In particular, they distinguish between 
\begin{enumerate}
	\item transmitting information, with properties such as
{\em legibility} (legible behaviors convey intentions, \ie, actual task at hand, via action choices),
{\em explicability} (explicable behaviors conform to observers’ expectations, \ie, they appear to have some purpose), and
{\em	predictability } (a behavior is predictable if it is easy to guess the end
	of an on-going trajectory); or
\item	hiding information, as through
{\em	obfuscation}, when the agent tries to hide its actual goal.
\end{enumerate}

They propose a general framework for such problems while assuming deterministic dynamics, and work mostly with plans (a sequence of actions, which induces a
sequence of states).
In their approach, the human is modeled by the robot as having a model
of the robot+environment system (including the robot’s possible
tasks), and is thus able to predict the robot's behavior.

\begin{figure}
	\centering
	\includegraphics[width=1.\columnwidth]{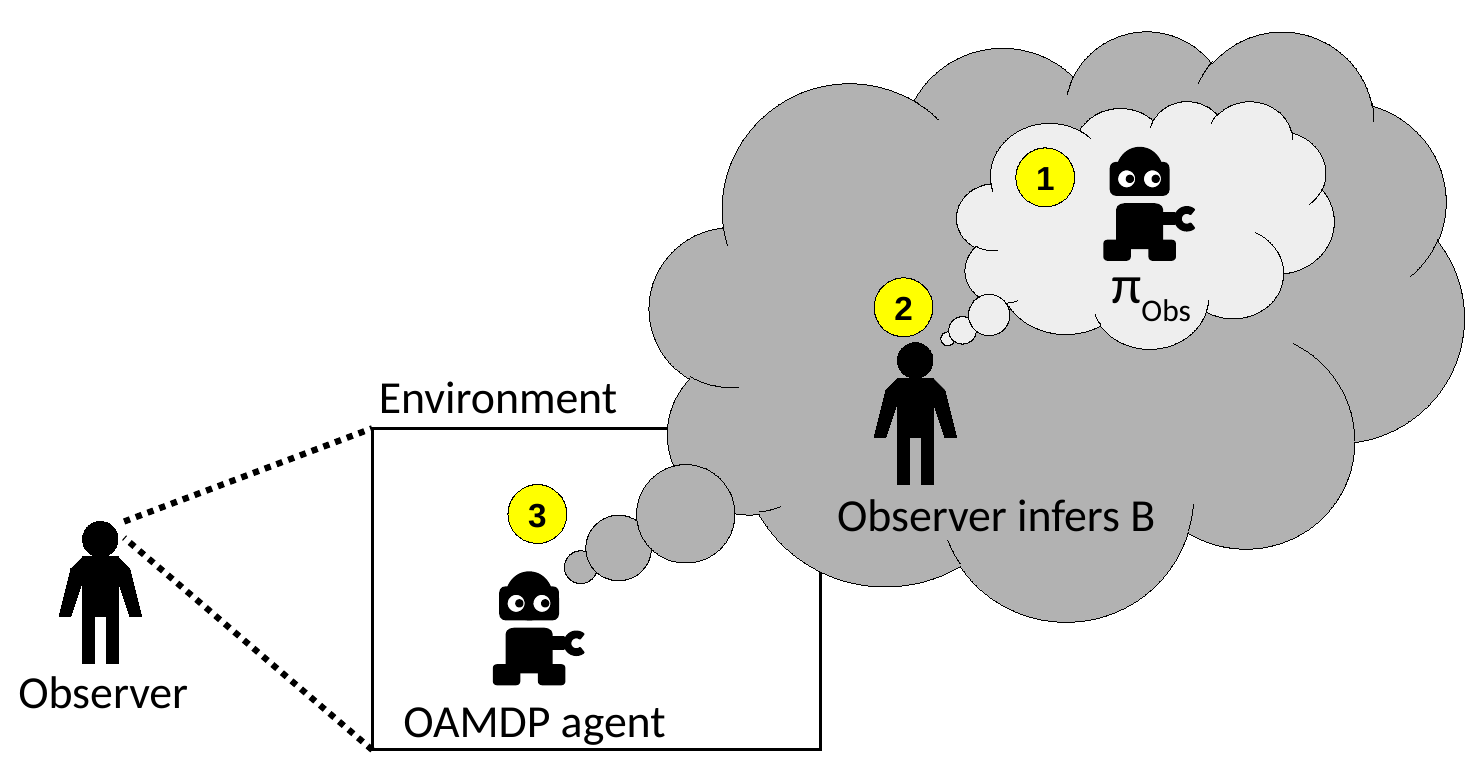}
\caption{An OAMDP agent (3) assumes that the observer expects (2) the agent to behave so as to achieve some task (1).}
	\label{figOAMDP}
\end{figure}

\citet{pmlr-v161-miura21a} build a unifying framework while assuming stochastic transitions, namely {\em observer-aware Markov decision processes} (OAMDPs), adopting a similar approach as \citet{DBLP:conf/aips/ChakrabortiKSSK19}, as illustrated in \Cref{figOAMDP}.
Among other things,  their work also covers legibility, explicability, and predictability.
To better handle predictability, \citet{LepLemThoBuf-arxiv24}
have recently proposed an approach that does not reason with complete trajectories, but with actions or states at each time step, thus being better suited to stochastic dynamics.
This implies reasoning on dynamic target variables, which requires introducing a variant of the OAMDP formalism, namely the pOAMDP (predictable OAMDP).

This paper proposes a formalism that can handle both
\begin{enumerate*}
\item problems with a static type (\eg, legibility, explicability as with OAMDPs) or a dynamic one (predictability as with pOAMDPs), and
\item problems with partial observability.
\end{enumerate*}
In this last situation, the observer may not have access to the state and the action of the agent but to an observation that depends on the transition, but the agent has access to all information, {\em including} the observer's observation.
Introducing partial observability allows considering more diverse and more realistic scenarios.
For example, one can consider settings where the PO-OAMDP agent is not always visible and needs to choose between several paths to be seen by the observer and allow her to better infer the current situation.

\Cref{sec|background} provides background on Markov decision processes and observer-aware MDPs.
The PO-OAMDP formalism is introduced  in \Cshref{sec|contribution}, before discussing theoretical properties of PO-OAMDPs and an example solving algorithm in \Cshref{sec|resolution}, presenting illustrative experiments in \Cshref{sec|XPs}, and concluding in \Cshref{sec|conclusion}.

\section{Background}
\label{sec|background}

\subsection{Markov Decision Processes}
A Markov decision process (MDP) \citep{Bellman-jmm57,Bertsekas-dpoc05} is specified through a tuple {$\langle \cS, \cA, T, R, \gamma, \cS_f \rangle$} where:
$\cS$ is a finite set of states;
$\cA$ is a finite set of actions;
$T: \cS \times  \cA \times \cS \to [0;1]$, the transition function, gives the
  probability $T(s,a,s')$ that action $a$ performed in state $s$ will
  lead to state $s'$;
$R : \cS \times \cA \times \cS \to \mathbb{R}$, the reward function, gives the immediate
  reward $R(s,a,s')$ received upon transition $(s,a,s')$;
$\gamma \in [0,1]$ is a discount factor; and
$\cS_f \subset \cS$ is a set of terminal states:
for all $s,a \in \cS_f \times \cA$, $T(s,a,s)=1$ and $R(s,a,s)=0$.

Then, a (stochastic) policy $\pi: \cS \to \simplex{\cA}$ maps states to distributions over actions, $\pi(a|s)$ denoting the probability to perform $a$ when in $s$.
When a policy is deterministic, $\pi$ denotes the only possible action in $s$.
Assuming $\gamma<1$, the value of a policy $\pi$ is the sum of discounted rewards over an infinite horizon:
\begin{align*}
	V^{\pi}(s)
	& \eqdef \mathbb{E}_\pi \left[ \sum_{t=0}^\infty \gamma^{t}r(S_t, A_t) | S_0=s  \right],
\end{align*}
and an optimal policy $\pi^*$ is such that, for all $s$, $V^{\pi*}(s) = \max_{\pi} V^{\pi}(s)$.

The {\em value iteration} (VI) algorithm approximates $V^*$, the value function common to all optimal policies, by iterating the following computation (where $k$ is the current iteration):
\begin{align*}
  V_{k+1}(s) & \gets \max_a \sum_{s'} T(s,a,s') \cdot \left( R(s,a,s') + \gamma V_k(s') \right).
  \intertext{
    Calculations stop when the {\em Bellman residual} is below a threshold:
$\max_s |V_{k+1}(s) - V_k(s)| \leq \frac{1-\gamma}{\gamma} \epsilon$,
an $\epsilon$-optimal deterministic policy being then obtained by acting greedily with respect to the solution value function $V_k$ , \ie, using
  }
  \pi^*(s) & \gets \argmax_a \sum_{s'} T(s,a,s') \cdot \left( R(s,a,s') + \gamma V^*(s') \right).
\end{align*}

These properties remain valid with $\gamma=1$ if
\begin{enumerate}
	\item $\cS_f$ is not empty; and
	\item $R$ is such that there exists at least one policy that reaches $\cS_f$ with probability $1$ from any state $s$, and
that the value of other policies diverge towards $-\infty$ in states from which $\cS_f$ is reachable with probability $<1$.
\end{enumerate}

Such problems are called {\em Stochastic Shortest Path problems} (SSPs).
In particular, we have an SSP if, for all $(s,a,s') \in (\cS \setminus \cS_f) \times \cA \times \cS$, we have $r(s,a,s')<0$, meaning that we are trying to reach a terminal state at ``minimum cost'' (on average).

Note: SSPs are more general than MDPs because any MDP can be turned into an SSP with, at any time step, a $1-\gamma$ probability to transition to a terminal state.

\subsection{Observer-Aware Markov Decision Processes}

An {\em Observer-Aware MDP} (OAMDP) \citep{pmlr-v161-miura21a} models a situation wherein an agent
is aware of the presence of an external observer, and
interacts with its environment while
attempting to maximize a performance criterion linked to the observer's belief about the agent's ``type'',
the {\em belief} about some variable being the probability distribution over this variable's possible values.

An OAMDP is formalized by a tuple {$\langle \cS, s_0, \cA, T, \gamma, \cS_f, \Type, B, \Rag \rangle$} where:
\begin{itemize}
	\item $\langle \cS, s_0, \cA, T, \gamma, \cS_f \rangle$ is an MDP with an initial state $s_0$ but {\em no reward function};
	\item $\Type$ is a finite set of {\em types} representing a characteristic of the
	agent such as possible goals, intentions or capabilities;
	\item $B: H^* \to \simplex{\Type}$ gives the assumed belief of the observer (about the agent's type) given a state-action history;
	\item $\Rag: \cS \times \cA \times \simplex{\Type} \to \mathbb{R}$ is the {\em ag}ent's reward function.
\end{itemize}

In most of the cases they consider, \citeauthor{pmlr-v161-miura21a} derive $B$ by relying on \citeauthor{BakSaxTen-cog09}'s ``BST'' Bayesian belief update rule \citep{BakSaxTen-cog09}, \ie, considering that, again from the agent’s viewpoint, the observer models the agent’s behavior for a given type $\type$ through an MDP by:
\begin{enumerate}
\item using a corresponding reward function $\Robs^\type$;
	\item solving MDP $\langle \cS, s_0, \cA, T^\type, \Robs^\type, \gamma, \cS_f^\type \rangle$  (where $\cS$, $\cA$ and $\gamma$ are as in the OAMDP definition)
	to obtain $V^{*,\type}_\obs$ for states reachable from $s_0$; and
\item 	building a stochastic “softmax” policy $\piobs^\type$ such that, $\forall(s,a)$,
	\begin{align*}
		\piobs^\type(a|s) & \eqdef
		\frac{
			e^{\frac{1}{\tau}Q^{*,\type}_\obs(s,a)}
		}{
			\sum_{a'} e^{\frac{1}{\tau}Q^{*,\type}_\obs(s,a')}
		},
\end{align*}
         where
         $ Q^{*,\type}_\obs(s,a) \eqdef \sum_{s'} T^\type(s,a,s') \cdot ( \Robs^\type(s,a,s') + \gamma V^{*,\type}_\obs(s') )$,
and temperature $\tau > 0$ allows tuning the policy’s optimality (thus the agent’s assumed rationality for the observer).
\end{enumerate}
The observer belief about the type can thus be obtained by Bayesian inference using $\piobs^\type$.
Note that, unless, for some $\type$, we have $T=T^\type$, $\Rag=\Robs^\type$, and $\cS_f=\cS_f^\type$, then there will likely be no perfect match between the agent's behavior and any of the types.
\citeauthor{pmlr-v161-miura21a} used this framework to formalize various observer-aware problems including legibility, explicability and predictability.

Note:
As done previously, ``\textsc{obs}'' is used to denote quantities associated to the observer viewpoint (as perceived by the agent), such as probabilities, denoted $P_\obs$.
Also, we will sometimes write a function $f(X,Y)$ describing a conditional probability distribution under the form $f(Y|X)$ to exhibit the dependence between variables.

\section{OAMDPs with Partial Observability}
\label{sec|contribution}

This section introduces the PO-OAMDP formalism, shows how the observer's belief about the target variable is maintained, and looks at some typical use cases.

\subsection{Formalism}

We describe the key ingredients of the PO-OAMDP framework before providing a formal definition.
\begin{enumerate*}
\item Within the PO-OAMDP framework, the agent has access to the complete state of the system, while the observer now only has a partial perception.
A set of observations and an observation function are thus added to the OAMDP formalism.
\item In this context, the {\em type} is replaced by a {\em target variable} which can change over time, contrary to OAMDP's static types.
For the definition of the target variable to be as generic as possible, its value at each time step is a function of the transition followed by the system.
The target variable can thus be a part of the system state (\eg, a non-observable variable for the observer), but it can also be linked to the action performed by the agent (for predictability), or to the state transition rather than to the state itself.
This variable can also gather several different variables.
But considering a single variable is without loss of generality.
\item Additionally, we assume that the agent has access not only to the complete state of the system, but also to the observations received by the observer (this is realistic in particular if the observation process is deterministic: in that case the observations received by the observer are easy to predict).
The agent can thus build the mental state of the observer during the execution of its behavior.
By having access to all of the problem's information (the system state, the chosen action and the observations perceived by the observer), the agent can make decisions to control the observer's inference about the target variable.
\end{enumerate*}

Formally, a PO-OAMDP is defined by a tuple $\langle \cS$, $s_0$, $\cA$, $T$, $\gamma$, $\cS_f$, $\Target$, $\Omega$, $O$, $B$, $\Rag$, $\targetOf \rangle$, where:
\begin{itemize}
	\item $\langle  \cS, s_0, \cA, T, \gamma, \cS_f \rangle$ is an MDP with an initial state $s_0$ but {\em no reward function};
	\item $\Target$ denotes both the (dynamic) {\em target variable} and the finite set of values it can take;
	\item $\targetOf: \cS \times \cA \times \cS \to \Target$ is a function that gives the value of the target variable given the transition: $\target_t=\targetOf(s_t,a_t,s_{t+1}$);
	\item $\Omega$ is the finite set of observations;
	\item $O: \cA \times \cS \times \Omega \to \mathbb{R}$  is the observation function;
          $O(a,s',o)$ is the probability of emitting observation $o$ if state $s'$ is reached while performing $a$;
 	\item $B: \Omega^* \to \simplex{\cS}$ gives the observer's belief on the state given an {\em observation} history;
	the belief on the target variable can be deduced from that {\em state belief} (see \Cshref{miseAjourB}), denoted $b$;
\item $\Rag: \cS\times \simplex{\Target} \times \cA \times \cS \times \simplex{\Target} \to \mathbb{R}$ is the agent's reward function under its most general form: $\Rag(s_t,\beta_t,a_t,s_{t+1},\beta_{t+1})$, where $\beta$ denotes a {\em target belief}.
\end{itemize}

Here, we assume that, through her observations, the observer knows at each time step whether a terminal state has been reached or not, without necessarily indicating which terminal state is concerned.

Unlike the OAMDP model, which needs an MDP for each possible type, the PO-OAMDP model is based on a single MDP.
However, within our framework with partial observability, using only one MDP is not restrictive, and,
as discussed in \Cshref{sec|BandRVariousProperties}, any OAMDP can provably be turned into an equivalent PO-OAMDP.

The next subsection describes how the observer's belief (on the state) can be updated and
how the belief on the target variable is deduced, which is used to evaluate the agent's reward attached to a transition.
Then, it illustrates the use of the PO-OAMDP formalism to model different scenarios.

\subsection{State- and Target-Belief Computation}
\label{miseAjourB}

\paragraph{BST belief state update}
Following \mbox{\citeauthor{pmlr-v161-miura21a}}, we employ the BST Bayesian belief update rule \citep{BakSaxTen-cog09},
thus introduce a reward function $\Robs: \cS\times\cA\times\cS \to \mathbb{R}$ assumed to be the agent's reward function according to the observer.
Then, the observer models the agent's behavior for a given task through an MDP by:
\begin{enumerate*}
\item solving the MDP with reward $\Robs$; and
\item deriving a softmax policy $\piobs$.
\end{enumerate*}

Note that, given the dynamics (transition + observation) of the PO-OAMDP and the presumed policy $\piobs$ of the agent, the observer faces a hidden Markov model (HMM) \citep{Rabiner89}:
she solves a {\em filtering} problem, using the observation's history $o_{1:t}$ to estimate her belief on the state $s_t$.
The observer belief can thus be computed with:
\begin{align*}
B(s_{t+1} | o_{1:t+1})
& = P(s_{t+1}|o_{1:t+1}) = \frac{P(s_{t+1},o_{1:t+1})} {\sum_{s_{t+1}} P(s_{t+1},o_{1:t+1})} \\
& = \frac{K(s_{t+1},o_{1:t+1}) \cancel{P(o_{1:t})}} {\sum_{s_{t+1}} K(s_{t+1},o_{1:t+1}) \cancel{P(o_{1:t})} },
\text{ where} \\
K(s_{t+1},o_{1:t+1})
& \eqdef \sum_{a_t}O(o_{t+1}|a_t,s_{t+1})\sum_{s_t}T(s_{t+1}|s_t,a_t) \cdot \\
& \qquad \piobs(a_t|s_t) \cdot B(s_t|o_{1:t}). \end{align*}

\paragraph{Belief on the target variable}
To evaluate the reward received during a transition, we need to evaluate the belief $\beta$ on the value that will be taken by the target value: $\Target_t=\targetOf(S_t,A_t,S_{t+1})$.
This can be done by starting with the observer's belief $b$ on the current state, $S_t$:
\begin{align}
& \beta(\target)
= \sum_{s,a,s'} \mathbb{1}_{\target=\targetOf(s,a,s')} \cdot P_{\obs}(s,a,s'|b) \nonumber \\
& = \sum_{s,a,s'} \mathbb{1}_{\target=\targetOf(s,a,s')} \cdot P_{\obs}(s'|s,a) \cdot P_{\obs}(a|s) \cdot P_{\obs}(s|b) \nonumber \\
& = \sum_{s,a,s'} \mathbb{1}_{\target=\targetOf(s,a,s')} \cdot T(s,a,s') \cdot \piobs(a|s) \cdot b(s), \label{eq|betaDeB}
\end{align}
where $\mathbb{1}_*$ is the indicator function.
\subsection{Relationship with POMDPs}
Despite similarities between POMDP and PO-OAMDP there are some key differences:
\begin{enumerate}
	\item in PO-OAMDPs, the reward function is typically not linear in belief space, 
	\item in PO-OAMDPs, the agent reasons about the observer's belief rather than its own belief, so that POMDPs are not a subclass of PO-OAMDPs. Also, the Bellman optimality operator for PO-OAMDPs does not preserve piecewise linearity and convexity as in POMDPs. The optimal value function may even exhibit local discontinuities (a property inherited from OAMDPs).
\end{enumerate}
\subsection{Implementation on Various Scenarios}
The PO-OAMDP model allows us to generate different behaviors by changing $\Target$ and $R$, and to work with different types of problems.
An important property, formally demonstrated in \Cref{app|generalization}, shows that PO-OAMDPs are at least as expressive as OAMDPs. 

\begin{restatable}[]{proposition}{propOAMDPequivPOOAMDP}
  \labelT{prop|OAMDPequivPOOAMDP}
  Any OAMDP $\mathcal M$ with BST belief update can be turned into an equivalent PO-OAMDP $\mathcal{M}'$, \ie, such that an optimal solution to one problem is optimal for the other problem.
\end{restatable}

A starting point of the proof is to turn the static {\em type} of an OAMDP into a (hidden) {\em target} state variable.
The following shows how to formulate legibility, explicability, and (state/action) predictability
while assuming (for the sake of clarity) that the transition and observation functions do not depend on the type.

\label{sec|BandRVariousProperties}

\paragraph{Legibility}
Assuming several possible objectives for the agent, legibility aims at reducing the observer's uncertainty about the agent's actual objective.

The target variable is thus part of the state, indicating the objective among a finite set of possible objectives,
and the {\em observer} reward function $\Robs$ depends on the target.
For the {\em agent} reward function,
\citeauthor{pmlr-v161-miura21a}
use the opposite of the Euclidean distance to the ``ideal belief''.
The ideal belief being defined by:
$\beta^*(s)= (0,\dots,0,1,0,\dots,0)$ (with a $1$ for component $\target=\targetOf(s)$), we thus have, for $\Rag$:
\begin{align*}
\Rleg(s,\beta,a,s',\beta')
& \eqdef -\norm{\beta-\beta^*(s)}_2.
\end{align*}

\paragraph{Explicability}
Assuming one or several possible objectives, an explicable behavior is a behavior coherent with the observer's expectations.

To express this idea, \citeauthor{pmlr-v161-miura21a} (following \citet{sreedharan2020bayesian}) propose minimizing the probability that the observed behavior corresponds to a random behavior, even if multiple behaviors are still likely.
As they do, we thus introduce a ``virtual'' target value $\target_0$ (in addition to the ones used for legibility) that corresponds to a random behavior (policy) in addition to the other (real) target values.
Then, the explicability criterion described above is obtained using
\begin{align*}
\Rexp(s,\beta,a,s',\beta')
& \eqdef - \beta(\target_0).
\end{align*}

\paragraph{Predictability}
A predictable behavior is typically a behavior whose end of trajectory is easy for the observer to guess.
\citeauthor{pmlr-v161-miura21a}'s discussion on predictability, which relies on work for deterministic settings and thus reasons on complete trajectories, does not provide a very convenient way of formalizing predictability under stochastic dynamics.
We rely instead on Lepers et al.'s work \citep{LepLemThoBuf-arxiv24}, as they propose a more satisfying approach relying on step-by-step predictions.

The starting point is that the observer tries, at each time step, to predict either the next action, or the next state, hence two different types of predictability.
For action predictability, we set $\Target=A$ and $\targetOf(s,a,s')=a$.
For state predictability, we set $\Target=S$ and $\targetOf(s,a,s')=s'$.
In both cases, to act optimally, the observer has to bet on the most probable next target values, and thus pick a value in the set of candidate values
\begin{align*}
\candidates_\Target(\beta_t)
& \eqdef \argmax_\target \beta_t(\target).
\end{align*}
Considering that the observer samples her prediction uniformly at random in the set $\candidates_\Target(\beta_t)$,
the probability that $\target$ is predicted is:
\begin{align*}
\pred(\target|\beta_t) & \eqdef
\frac{1}{|\candidates_\Target(\beta_t)|} \cdot \mathbb{1}_{\target \in \candidates_\Target(\beta_t)} .
\end{align*}

Then, defining
\begin{align*}
  \RApred(s,\beta,a,s',\beta')
  & \eqdef \pred(a|\beta)-1, \text{ or} \\
  \RSpred(s,\beta,a,s',\beta')
  & \eqdef \pred(s'|\beta)-1,
\end{align*}
the immediate reward is the opposite of the probability that, under the current transition, the bet of a rational observer will fail:
$\Rpred(s,\beta,a,s',\beta') = - P(\text{failed rational bet})$.

\medskip

Note: Other example scenarios formalized as PO-OAMDPs are presented in supplementary material, \Cref{app|moreExampleScenarios}.
They illustrate, among other things, the similarities with $\rho$-POMDPs \citep{AraBufThoCha-nips10}, a variant of the POMDP formalism where an agent's reward function depends on its own belief, which permits modeling active information-gathering problems.
Yet, the same differences between OAMDPs and $\rho$-POMDPs pointed out by \citeauthor{pmlr-v161-miura21a}
\citep{pmlr-v161-miura21a} still hold between PO-OAMDPs and $\rho$-POMDPs.

\section{Resolution}
\label{sec|resolution}

\subsection{Sequential Decision-Making Problem}

An OAMDP can be turned into an equivalent MDP using the state-action history $\langle s_{0:t}, a_{1:t} \rangle$ (\ie, all the raw information available to the agent at $t$) as information state, or the state-belief (over type) pair $\langle s,\beta \rangle$ when using the BST update \citep{pmlr-v161-miura21a,miuBufZil-uai24}.

Similarly, in a PO-OAMDP, the raw information available at $t$ is the state-action-observation history $\langle s_{0:t}, a_{1:t}, o_{1:t} \rangle$.
Yet, note that:
\begin{enumerate}[label=(\arabic*)]
\item \label{enum|Markov} the pair $\langle s_t, o_{1:t} \rangle$ induces a Markov process; and
\item the observer's beliefs ($b_t$, thus also $\beta_t$) depend on the observation history $o_{1:t} \equiv \langle o_1,\dots,o_t \rangle$;
\item \label{enum|reward} the reward is a function of the state and the target belief, thus of the observation history, not of the past states and actions.
\end{enumerate}
From \labelcref{enum|Markov,enum|reward}, the state and observation history pair $\langle s_t, o_{1:t} \rangle$ is a sufficient statistic for optimal decision-making.
What is more, when using the BST update, the state belief is Markovian (though not the target belief in general), so that $\langle s_t, b_t \rangle$ can be used instead.

Formally, we obtain an MDP $\langle \cI, i_0, \cA, T', R', \gamma, \cI_f \rangle$, where:
\begin{itemize}
\item $\cI\eqdef\cS\times B$ is the (infinite) set of states, with $i_0 = \langle s_0, b_0 \rangle$ the initial state;
\item $\cA$ is the PO-OAMDP's set of actions;
\item $T': \cI \times  \cA \times \cI \to [0;1]$ is the transition function defined by:
\begin{align*}
    T'(i'|i, a)
    & \eqdef Pr(s',b'|s, b, a) \\
    & = \sum_{o} \mathbb{1}_{b'=B(b,o)} O(o|s',a)P(s'|a,s);
  \end{align*}
\item $R': \cI \times \cA \times \cI \to \mathbb{R}$ is the reward function defined by:
  \begin{align*}
R' (s,b,a,s',b')
    & \eqdef \Rag(s,\beta(b),a,s',\beta(b')),
  \end{align*}
where $\beta(b)$ is the target belief that can be derived from $b$ as seen in \Cref{eq|betaDeB};
\item $\gamma \in [0,1]$ is the discount factor; and
\item $\cI_f \subset \cI$ is the set of elements $\langle s, b \rangle$ in $\cI$ such that $s\in \cS_f$.
\end{itemize}

We assume that $\Rag(s,\beta(b),a,s',\beta(b'))=0$ whenever $s\in \cS_f$.
As a consequence, when $i\equiv\langle s,b \rangle \in \cI_f$ is reached, since the state $s$ does not change anymore, and even if the state belief may still evolve, all future rewards will be null, so that we are in a ``terminal sub-set of states''.

In this setting, Bellman's optimality operator is thus written
\begin{align*}
  V^*(i) & =  \max_{a} \sum_{i'\in \text{nxt}(i,a)} T'(i,a,i') \cdot [ R'(i,a,i') + \gamma V^*(i')],
\end{align*}
where $\text{nxt}(i,a)$ is the (finite) set of possible next state-belief pairs when performing $a$ in $i$.

\subsection{SSPs}

Setting $\gamma=1$ raises the question whether the resulting problem is a valid SSP.
The following proposition answers positively while considering problems with a possibly infinite set of states reachable from initial state $\langle s_0, - \rangle$, where $-$ denotes the empty history.

\begin{proposition}
  \label{thm|POOASSP|valid}
  Assuming that $\Rag$ is bounded from above by $\Rag^{\max}<0$ (in non-terminal states), the PO-OASSP is a valid SSP.
\end{proposition}

\begin{proof}
  First, any reachable pair $\langle s, o_{1:t} \rangle$ with $s\in \cS_f$ is a terminal state of the PO-OASSP.

  Then, let $\hat\pi$ be a proper policy of the observer SSP.
  When in $\langle s, o_{1:t} \rangle$, one can apply $\hat\pi$ (thus ignoring observation histories), thus ensuring that a terminal state of the SSP is reached, which corresponds to a terminal state of the PO-OASSP.

  In the contrary, if, from $\langle s,o_{1:t} \rangle$, one applies a policy $\pi$ that reaches a terminal state only with probability $p<1$, then there is a probability $1-p$ to follow an infinite trajectory with a per-step cost $\Rag^{\max}<0$, so that the value at $\langle s, o_{1:t} \rangle$ diverges to $-\infty$.
\end{proof}

Note that ensuring that $\Rag$ only takes negative values is not sufficient to prove the above lemma, as not all infinite sums of negative values diverge.
For the $\Rag$ functions described in \Cshref{sec|contribution} for legibility, explicability and predictability, the least upper-bound is $0$, so that it is unclear whether all improper policies have diverging values.
In particular, \citeauthor{LepLemThoBuf-arxiv24}'s Proposition~2 in \citep{LepLemThoBuf-arxiv24}, which applies in our setting, states that state predictability can lead to an improper policy.
A simple trick to come back to a valid SSP is to linearly combine the invalid $\Rag$ with a valid $R: \cS\times \cA\times \cS \to \mathbb{R}^-$ using $\Rag' = \Rag + \lambda \cdot R$ for some small $\lambda >0$.

\subsection{Complexity}

\Cref{prop|OAMDPequivPOOAMDP} tells us that PO-OAMDPs cover a larger class of problems than OAMDPs.
Below we establish that PO-OAMDPs inherit the same main complexity results as OAMDPs,
results which require assuming {\em Bayesian updates} for the observer's belief, what we denote by PO-OAMDP$_{BU}$.
Such results are obtained considering the {\em value problem}, \ie, determining whether a policy exists that can achieve some pre-defined value.

\begin{theorem}
  \label{thm|PSpace}
  The ﬁnite-horizon value problem for PO-OAMDP$_{BU}$ is PSPACE as long as $R$ can be evaluated using polynomial space.
\end{theorem}

\begin{proof}

  As for OAMDP$_{BU}$s, a policy's possible outcomes can be expressed as a tree whose depth corresponds to the finite horizon, and
the policy's value can be computed in polynomial space through a tree traversal (provided $R$ can be evaluated in polynomial space as well).
PO-OAMDP$_{BU}$s are thus in PSPACE.
\end{proof}

\begin{theorem}
  \label{thm|PSpaceHard}
The ﬁnite-horizon value problem for PO-OAMDP$_{BU}$ is PSPACE-hard.
\end{theorem}

\begin{proof}
The proof of \Cshref{prop|OAMDPequivPOOAMDP} shows that any OAMDP$_{BU}$ can be turned into a PO-OAMDP$_{BU}$ through a polynomial reduction.
Then, as OAMDP$_{BU}$ is PSPACE-hard \citep{pmlr-v161-miura21a}, so is PO-OAMDP$_{BU}$.
\end{proof}

\begin{corollary}
  \label{thm|PSpaceComplete} The ﬁnite-horizon value problem for PO-OAMDP$_{BU}$ is PSPACE-complete when $R$ can be evaluated using polynomial space.
\end{corollary}

This is a direct consequence of \Cref{thm|PSpace,thm|PSpaceHard}.

\subsection{HSVI}

This section proposes solving discounted PO-OAMDPs ($\gamma<1$) using a variant of \citeauthor{SmiSim-uai04}'s {\em heuristic search value iteration} (HSVI) algorithm \citep{SmiSim-uai04,SmiSim-uai05,Smith-phd07}.
HSVI is generally used to solve POMDPs through equivalent belief MDPs, maintaining an upper and a lower bound of $V^*$, respectively denoted $\upb{V}$ and $\lob{V}$, and whose representations exploit $V^*$'s convexity in belief space.
As illustrated in \Cref{alg|HSVI|infinite} (where it is presented for MDPs, thus reasoning on states),
these bounds are updated (\cref{alg|HSVI|line|updateForward,alg|HSVI|line|updateBackward}) while simulating trajectories in a recursive manner (\cref{alg|HSVI|line|Recur}),
making decisions optimistically (\ie, acting greedily with respect to $\upb{V}$, \cf \cref{alg|HSVI|line|pickAction}) and
picking the next transition so as to provably reduce uncertainty about the value (\cref{alg|HSVI|line|pickState}).
It stops when $\upb{V}(b_0)-\lob{V}(b_0) \leq \epsilon$ for some positive $\epsilon$ (\cref{alg|HSVI|line|stoppingCriterion}).

\begin{algorithm}
    \caption{\texttt{HSVI} for (infinite-horizon) \texttt{MDP}s}
    \label{alg|HSVI|infinite}
    \SetKwInOut{Input}{Input}
    \SetKwInOut{Output}{Output}
    \SetKwFunction{Update}{Update}
    \SetKwFunction{Explore}{Explore}
    \SetKwFunction{Solve}{Solve}
    
    \Input{$s_0$ a state}
    
    \Fct{\Solve}{
      
      Initialize $ \overline{V} : S \to \mathbb{R}$ with optimistic value\;
      Initialize $ \underline{V} : S \to \mathbb{R}$ with pessimistic value\;
      
      \While{$\overline{V}(s_0) - \underline{V}(s_0) \geq \epsilon$ \nllabel{alg|HSVI|line|stoppingCriterion}}{
	\Explore($s_0,0$)\;
      }
      \Return{$\underline{V}$}
    }
    
    \Fct{$\Explore(s,t)$}{
      \If{$\overline{V}(s) - \underline{V}(s) \geq \epsilon\gamma^{-t}$}{ 
        
        \Update($s,t$) \nllabel{alg|HSVI|line|updateForward} \;
	$a^* \gets \argmax_{a} r(s,a) + \gamma \sum_{s'} T(s,a,s') \overline{V}(s')$ \nllabel{alg|HSVI|line|pickAction} \;
	$s^* \gets \argmax_{s'} T(s,a^*,s') \left[ \overline{V}(s') - \underline{V}(s') - \epsilon \cdot \gamma^{-(t+1)} \right] $ \nllabel{alg|HSVI|line|pickState} \;
	\Explore($s^*,t+1$) \nllabel{alg|HSVI|line|Recur} \;
        \Update($s,t$) \nllabel{alg|HSVI|line|updateBackward} \;
      }
      \Return
    }
    
    \Fct{\Update($s,t$)}{
      $\overline{V}(s) \gets \max_{a} r(s,a) + \gamma \sum_{s'} T(s,a,s') \overline{V}(s')$\;
      $\underline{V}(s) \gets \max_{a} r(s,a) + \gamma \sum_{s'} T(s,a,s') \underline{V}(s')$\;
    }
    
\end{algorithm}

Differences between POMDPs and PO-OAMDPs lead to several differences in HSVI.
\begin{enumerate*}
\item Our setting allows for observable terminal states, so that trajectories can be terminated when one is encountered, and $V^*(i)=0$ for $i\in\cS_f$.
\item $\lob{V}$ and $\upb{V}$ are expressed in information space $\cI \equiv \cS \times \simplex{\cS}$, not in $\simplex{\cS}$ alone.
\item PO-OAMDPs inherit local discontinuities in $\simplex{\cS}$ from OAMDPs \cite[Sec.~3.2]{miuBufZil-uai24}, so that we do not attempt to use generalizing representations (which typically rely on continuity properties), but only rely on pointwise representations.
\item Usual bound initializations do not apply, and others need to be introduced, as discussed next.
\end{enumerate*}

\paragraph{Initializing Bounds}
\label{sec|initBounds}

As in a standard discounted MDP, a first way to initialize the bounds is simply with, $\forall i$,
$\lob{V}(i) = {{R'}^{\min}}/{(1-\gamma)}$ and
$\upb{V}(i) = {{R'}^{\max}}/{(1-\gamma)}$,
where ${R'}^{\min} \eqdef \min_{i,a,i'} {R'}(i,a,i')$ and ${R'}^{\max} \eqdef \max_{i,a,i'} {R'}(i,a,i')$.
These {\em naive initializations} are very loose, thus far from informative.
Usual POMDP bounds rely on properties that do not hold in our setting (in particular the linearity of the reward function in belief space), and thus do not apply here.
In the following, we consider that a term that is not belief-dependent (noted $R_s$ and possibly equal to $\Robs$) can be isolated in the reward function, the other term being denoted $R_b$: ${R'}(s,b,a,s',b') = R_s(s,a,s') + R_b(s,b,a,s,b')$, and introduce so-called {\em combined initializations}.

Our decomposition ${R'}=R_s+R_b$ allows lower-bounding ${R'}(i,$ $a$, $i')$ with $R_s(s,a,s')+R_b^{\min}$, where $R_b^{\min} \eqdef \min_{i,a,i'} R_b(i,$ $a,$ $i')$,
so that $V^*(s,b)$ could be lower-bounded by $V^\pi_s(s)+{R_b^{\min}}/{(1-\gamma)}$, with $\pi$ some policy, for instance the solution $\pi_s^*$ of the MDP equipped with $R_s$.
But this lower bound can again be very loose when $\gamma$ is close to $1$.
To avoid the $\frac{1}{1-\gamma}$ term, we can work with some predefined policy $\pi$ and lower-bound its PO-OAMDP value as
\begin{align}
  V^\pi(s,b) & \geq V_s^\pi(s) + R_b^{\min} \cdot V_{costToGo}^\pi(s),
  \label{eq|lob}
\end{align}
where
$V_s^\pi$ evaluates $\pi$ with $R_s$, and
$V_{costToGo}^{\pi}$ evaluates $\pi$ with $R_{costToGo}(s,a,s') \eqdef \mathbb{1}_{s'\not\in \cS_f}$, \ie, $V_{costToGo}^{\pi}(s)$ is the average time before reaching a terminal state under $\pi$ if interpreting $1-\gamma$ as a termination probability at each time step.
Two possible choices for $\pi$, the best one depending on the situation at hand, are $\piobs$, so that $V_s^\pi=V_{\obs}$, and $\pi_s^*$.

To upper-bound $V^*$, a simple approach is to compute $V^*_s$, the optimal value function for the MDP with $R_s$, and then write:
\begin{align}
  V^*(s,b) & \leq V^*_s(s) + \frac{R_b^{\max}}{1-\gamma},
  \label{eq|upb}
\end{align}
where $R_b^{\max} \eqdef \max_{i,a,i'} r_b(i,a,i')$ ($=0$ for all the criteria we presented).

Using $\gamma=1$, thus for PO-OASSPs,
the lower bound (\Cshref{eq|lob}) requires the initializing policy $\pi$ to be proper, which is true when using $\piobs$, $\pi^*_s$, or a uniformly random policy,
and the upper bound (\Cshref{eq|upb}) requires replacing the belief-dependent term with $0$.

\section{Experiments}
\label{sec|XPs}

The conducted experiments will first allow looking at some resulting behaviors, thus demonstrating its possible benefits, and illustrating some encountered phenomena.
Then, they will show the influence of problem types and bound initializations on HSVI's runtime.
The source code is available at \url{https://gitlab.inria.fr/po-oamdp/po-oamdp_aamas25} .

\paragraph{Baseline Policies}
In any given problem instance, we first compute the softmax policy $\piobs$ obtained using value iteration and a softmax with a low temperature ($\tau=0.01$), so that sub-optimal actions are picked with low probability.

Then, this policy not only serves
\begin{enumerate*}
	\item to model the observer's belief update, but also
	\item as a baseline, and
        \item to compute HSVI's initial bounds as described in \Cshref{sec|initBounds}.
\end{enumerate*}

\paragraph{Algorithm Settings}

Our experiments focus on goal-oriented problems.
However we stick to using a (large) discount factor $\gamma=0.99$,
\begin{enumerate*}
\item to illustrate some pathological behaviors, and
\item to allow sticking to the standard HSVI algorithm.\footnote{Adapting \citeauthor{HorBosCha-ijcai18}'s Goal-HSVI \citep{HorBosCha-ijcai18} would allow solving such problems.}
\end{enumerate*}

Any instance of value iteration or policy evaluation \citep{Bertsekas-dpoc05} stops when the Bellman residual is below $\epsilon_{VI}=0.0001$.
Also, in all experiments, HSVI stops when the root gap $\epsilon_{HSVI}=0.001$ or with a 1\,h timeout.

\subsection{Benchmark Problems}

We now describe the underlying MDP, named {\em Maze}.

\subsubsection{Maze problems}
\label{sec|MazeProblems}

As illustrated in \Cshref{fig|legibility|trajectories,fig|legibility|trajectories|partialObs,fig|explicability|trajectories,fig|predictability|trajectories}, a maze is defined by a 4-connected grid world that contains walls (in dark grey), normal cells (in white), hidden cells (in cyan), and goal cells: the current {\em actual} goal (green disk) and {\em alternate} goals (pink diamonds).

More formally, in this SSP:
\begin{enumerate}
	\item each state $s$ in $\cS$ indicates
	\begin{enumerate*}
		\item the $(x, y)$ coordinates of the agent, which can be
		in a normal, hidden, or goal cell, and
		\item which goal cell is the actual goal $(x_g,y_g)$;
	\end{enumerate*}
	\item the only terminal states ($\cS_f$) are states $s$ such that the agent is in the actual goal ($(x,y)=(x_g,y_g)$);
	\item $\cA=\{$up, down, left, right$\}$;
	\item $T(s,a,s')$ encodes the agent's moves:
	an agent in a normal, hidden or alternate-goal cell moves in the direction indicated by its chosen action if no wall prevents it;
	an agent in an actual-goal cell, being in a terminal state, does not move;
	\item  $\Robs$, the observer reward function, returns a default penalty of  $-0.01$ for each move, $-1$ when the agent hits a wall, and $0$ when in a terminal state.
\end{enumerate}
To this SSP we add:
\begin{enumerate}[resume]
	\item a set of observations $\Omega = \{(x,y) | (x,y)$ is a visible or actual-goal cell$\} \cup \{ \text{\texttt{none}}\}$; and
\item an observation function $O$ which returns the agent's location (with probability $p_\obs$, set by default to $1$) when it is in a visible cell  or in a goal cell, and the ``\texttt{none}'' observation otherwise; and
        \item
          the belief $b_0$ uniform over states $s={(x_0, y_0, x_{g_i},y_{g_i})}$, $(x_0,y_0)$ being the initial cell (known by the observer, but hidden) and $i$ indexing possible goals.
\end{enumerate}
Note that, as required,
\begin{enumerate*}
\item the observer knows when a terminal state is reached, \ie, when the agent has reached the actual goal; and
\item there are several possible goals in legibility and explicability scenarios,
but a single one for predictability.
\end{enumerate*}

The observer policy $\piobs$ just quickly reaches the actual goal, and can thus follow any of possibly many shortest paths (with deviations due to the softmax), ignoring the observer's viewpoint.

\subsubsection{Grids used}

\paragraph{Legibility and Explicability}
The maze (\Cshref{fig|legibility|trajectories}) consists of an open space with 3 possible goal states and a single row (\#5) of visible cells which the agent may want to exploit.
We will consider the actual goal being either the left one or the middle one.

\paragraph{Legibility with Stochastic Observations}
We examine stochastic observations ($p_\obs=0.5$) only in a legibility problem (\Cshref{fig|legibility|trajectories|partialObs}) with two paths for the left goal, the longest one (right) having more visible cells.

\paragraph{Predictability}

For action- and state-predictability, the maze (\Cshref{fig|predictability|trajectories}) consists of mostly hidden corridors, only three cells being visible for the observer: $(D,10)$, $(D,2)$ and $(B,2)$, the actual-goal cell.

\subsection{Solution Evaluation}

\Cref{tab|valuesPolicy} provides
\begin{enumerate*}
\item the reward functions in use in each setting, the choice of a combination with $\Robs$ being discussed below, and
\item the value at $i_0$ in each case for the PO-OAMDP policy vs the baseline $\piobs$ vs the underlying MDP optimal policy (estimated through 1\,000 simulated trajectories).
\end{enumerate*}
The main observation is that the PO-OAMDP agent consistently and significantly outperforms the baseline.

\begin{table*}[ht!]
	\caption{$V^\pi(i_0)$ for various problems }
	\label{tab|valuesPolicy}
	\centering

	\adjustbox{max width=1.\linewidth}
	{

\sisetup{
  round-mode = places,
  round-precision = 2,
  table-format=2.2,
}%
\begin{tabular}{ c c S S S S S S}
  \toprule
  && \multicolumn{3}{c}{Legibility} & {Explicability} & {Action pred.} & {State pred.} \\
 criterion& Policies
  & \multicolumn{3}{c}{$\Robs+\Rleg$} & {$\Rexp$} & {$\Robs+\RApred$} & {$\RSpred$} \\
  && {\scriptsize left goal} & {\scriptsize middle goal} & {\scriptsize $p_\obs=0.5$} & \\
  \cmidrule(lr){1-2} \cmidrule(lr){3-5} \cmidrule(lr){6-6} \cmidrule(lr){7-7} \cmidrule(lr){8-8}
 \multirow{3}{*}{$V_\text{PO-OAMDP}$} &PO-OAMDP  & -3.586 &  -3.098 & -3.017& -1.052 & -1.607 & -2.455  \\  
 & MDP ($\piobs$)  & -6.759659049006518 & -7.588678920801781 &-4.317103341563281 & -3.093 & -6.3032299984279065 & -10.991001072113848 \\
   &  MDP  & -4.118667541924231 &  -3.556822164869111  & -3.2578291177308922 & -1.507 & -3.4077520293885484 & -8.39787827021881 \\ 
  \cmidrule(lr){1-2} \cmidrule(lr){3-5} \cmidrule(lr){6-6} \cmidrule(lr){7-7} \cmidrule(lr){8-8}
 \multirow{3}{*}{$V_\obs$} &  PO-OAMDP  & -0.14854222890512436 &  -0.18209306240276898 & -0.09723767328753627& -0.1312541872310219 & -0.1312541872310219& -0.16548623854991223 \\ 
 &  MDP ($\piobs$)   & -0.19521241072754467 &  -0.19835530765305032 & -0.10374775194854641& -0.1965231537635659& -0.16444602061778624 & -0.16681150671459172 \\ 
  &  MDP  & -0.131 &  -0.114  & -0.077 & -0.131 & -0.131 & -0.131 \\ 
  \bottomrule
\end{tabular}

	}
\end{table*}

In the following, we mainly look at example trajectories obtained using a PO-OAMDP agent, along with the evolution of the belief about the target variable.
Corresponding observer MDP policies $\piobs$ are provided in \Cref{app|XP}.
White stripes appearing on belief evolutions (\Cshref{fig|legibility|trajectories,fig|explicability|trajectories,fig|legibility|trajectories|partialObs,fig|predictability|trajectories}) correspond to time steps where the agent has been observed.

\paragraph{Legibility}
In the grid \Cref{fig|legibility|trajectories|left}, if row 5 is crossed on the left, the observer strongly believes  in the left goal, and rewards become small, which makes it harder for HSVI to converge to a proper policy.
We thus combined $\Robs$ with the legibility reward.

For the left actual goal \Cshref{fig|legibility|trajectories|left}, the agent does not go directly up to visible cell $(D,5)$, what would slightly increase the probability of the middle goal.
It goes to the left-most visible cell $(B,5)$, and goes back to it multiple times to increase the belief in the left goal before actually reaching that actual goal.
Disappearing from $(B,5)$ (rather than appearing in $(C,5)$) also increases the belief in the actual goal.
As illustrated in \Cref{app|XP|legibility}, \Cshref{RLegibility}, there are no such ``oscillations'' when the remaining path to the goal is short.

For the middle actual goal \Cshref{fig|legibility|trajectories|middle}, the agent has a similar behavior, but going up to visible cell $(D,5)$ instead of $(B,5)$.

As can be observed in \Cshref{tab|valuesPolicy}, those PO-OAMDP policies have significantly better values than the default MDP policies, which do not attempt to increase the observer's certainty before traversing to the goal.

\def\mycolwidth{0.45\columnwidth}
\def\mazewidth{0.40\columnwidth}
\def\beliefgraphwidth{0.45\columnwidth}

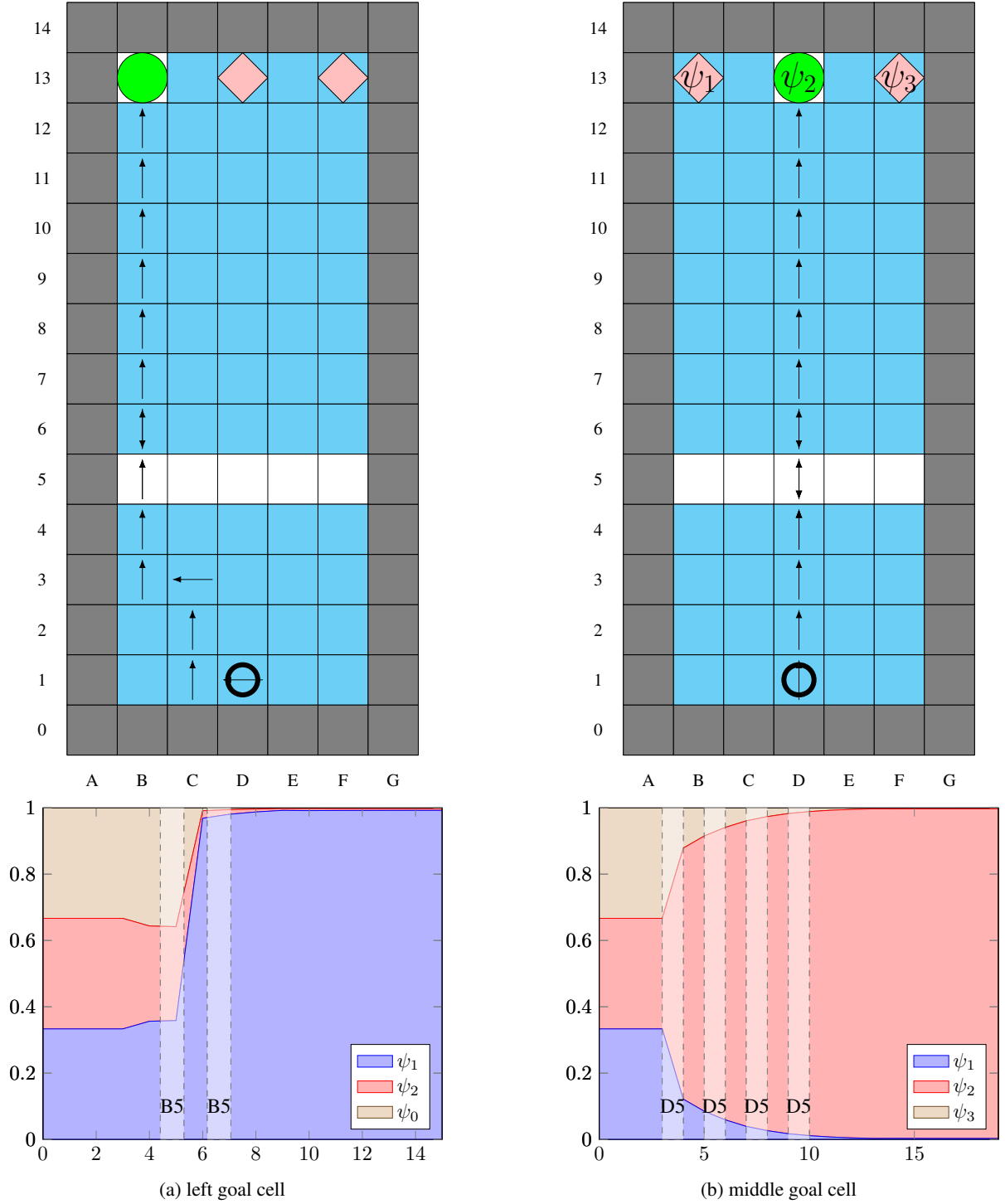
\begin{figure}
\subcaptionbox{
    left goal cell
        \label{fig|legibility|trajectories|left}
  }[\mycolwidth]{
    \adjustbox{width=\mazewidth}{
      \begin{tikzpicture}
\draw (0,0) grid (7,15);
\draw[fill=gray] (0,0) rectangle (1,1);
\draw[fill=gray] (1,0) rectangle (2,1);
\draw[fill=gray] (2,0) rectangle (3,1);
\draw[fill=gray] (3,0) rectangle (4,1);
\draw[fill=gray] (4,0) rectangle (5,1);
\draw[fill=gray] (5,0) rectangle (6,1);
\draw[fill=gray] (6,0) rectangle (7,1);
\draw[fill=gray] (0,1) rectangle (1,2);
\draw[fill=cyan!50] (1,1) rectangle (2,2);
\draw[fill=cyan!50] (2,1) rectangle (3,2);
\draw[fill=cyan!50] (3,1) rectangle (4,2);
\draw[fill=cyan!50] (4,1) rectangle (5,2);
\draw[fill=cyan!50] (5,1) rectangle (6,2);
\draw[fill=gray] (6,1) rectangle (7,2);
\draw[fill=gray] (0,2) rectangle (1,3);
\draw[fill=cyan!50] (1,2) rectangle (2,3);
\draw[fill=cyan!50] (2,2) rectangle (3,3);
\draw[fill=cyan!50] (3,2) rectangle (4,3);
\draw[fill=cyan!50] (4,2) rectangle (5,3);
\draw[fill=cyan!50] (5,2) rectangle (6,3);
\draw[fill=gray] (6,2) rectangle (7,3);
\draw[fill=gray] (0,3) rectangle (1,4);
\draw[fill=cyan!50] (1,3) rectangle (2,4);
\draw[fill=cyan!50] (2,3) rectangle (3,4);
\draw[fill=cyan!50] (3,3) rectangle (4,4);
\draw[fill=cyan!50] (4,3) rectangle (5,4);
\draw[fill=cyan!50] (5,3) rectangle (6,4);
\draw[fill=gray] (6,3) rectangle (7,4);
\draw[fill=gray] (0,4) rectangle (1,5);
\draw[fill=cyan!50] (1,4) rectangle (2,5);
\draw[fill=cyan!50] (2,4) rectangle (3,5);
\draw[fill=cyan!50] (3,4) rectangle (4,5);
\draw[fill=cyan!50] (4,4) rectangle (5,5);
\draw[fill=cyan!50] (5,4) rectangle (6,5);
\draw[fill=gray] (6,4) rectangle (7,5);
\draw[fill=gray] (0,5) rectangle (1,6);
\draw[fill=gray] (6,5) rectangle (7,6);
\draw[fill=gray] (0,6) rectangle (1,7);
\draw[fill=cyan!50] (1,6) rectangle (2,7);
\draw[fill=cyan!50] (2,6) rectangle (3,7);
\draw[fill=cyan!50] (3,6) rectangle (4,7);
\draw[fill=cyan!50] (4,6) rectangle (5,7);
\draw[fill=cyan!50] (5,6) rectangle (6,7);
\draw[fill=gray] (6,6) rectangle (7,7);
\draw[fill=gray] (0,7) rectangle (1,8);
\draw[fill=cyan!50] (1,7) rectangle (2,8);
\draw[fill=cyan!50] (2,7) rectangle (3,8);
\draw[fill=cyan!50] (3,7) rectangle (4,8);
\draw[fill=cyan!50] (4,7) rectangle (5,8);
\draw[fill=cyan!50] (5,7) rectangle (6,8);
\draw[fill=gray] (6,7) rectangle (7,8);
\draw[fill=gray] (0,8) rectangle (1,9);
\draw[fill=cyan!50] (1,8) rectangle (2,9);
\draw[fill=cyan!50] (2,8) rectangle (3,9);
\draw[fill=cyan!50] (3,8) rectangle (4,9);
\draw[fill=cyan!50] (4,8) rectangle (5,9);
\draw[fill=cyan!50] (5,8) rectangle (6,9);
\draw[fill=gray] (6,8) rectangle (7,9);
\draw[fill=gray] (0,9) rectangle (1,10);
\draw[fill=cyan!50] (1,9) rectangle (2,10);
\draw[fill=cyan!50] (2,9) rectangle (3,10);
\draw[fill=cyan!50] (3,9) rectangle (4,10);
\draw[fill=cyan!50] (4,9) rectangle (5,10);
\draw[fill=cyan!50] (5,9) rectangle (6,10);
\draw[fill=gray] (6,9) rectangle (7,10);
\draw[fill=gray] (0,10) rectangle (1,11);
\draw[fill=cyan!50] (1,10) rectangle (2,11);
\draw[fill=cyan!50] (2,10) rectangle (3,11);
\draw[fill=cyan!50] (3,10) rectangle (4,11);
\draw[fill=cyan!50] (4,10) rectangle (5,11);
\draw[fill=cyan!50] (5,10) rectangle (6,11);
\draw[fill=gray] (6,10) rectangle (7,11);
\draw[fill=gray] (0,11) rectangle (1,12);
\draw[fill=cyan!50] (1,11) rectangle (2,12);
\draw[fill=cyan!50] (2,11) rectangle (3,12);
\draw[fill=cyan!50] (3,11) rectangle (4,12);
\draw[fill=cyan!50] (4,11) rectangle (5,12);
\draw[fill=cyan!50] (5,11) rectangle (6,12);
\draw[fill=gray] (6,11) rectangle (7,12);
\draw[fill=gray] (0,12) rectangle (1,13);
\draw[fill=cyan!50] (1,12) rectangle (2,13);
\draw[fill=cyan!50] (2,12) rectangle (3,13);
\draw[fill=cyan!50] (3,12) rectangle (4,13);
\draw[fill=cyan!50] (4,12) rectangle (5,13);
\draw[fill=cyan!50] (5,12) rectangle (6,13);
\draw[fill=gray] (6,12) rectangle (7,13);
\draw[fill=gray] (0,13) rectangle (1,14);
\draw[fill=green] (1.500000,13.500000) circle (0.500000);
\draw[fill=cyan!50] (2,13) rectangle (3,14);
\draw[fill=cyan!50] (3,13) rectangle (4,14);
\begin{scope}[shift={(3.500000 ,13)}] \draw[fill=pink,rotate=45] rectangle (0.707, 0.707);
\end{scope}
\draw[fill=cyan!50] (4,13) rectangle (5,14);
\draw[fill=cyan!50] (5,13) rectangle (6,14);
\begin{scope}[shift={(5.500000 ,13)}] \draw[fill=pink,rotate=45] rectangle (0.707, 0.707);
\end{scope}
\draw[fill=gray] (6,13) rectangle (7,14);
\draw[fill=gray] (0,14) rectangle (1,15);
\draw[fill=gray] (1,14) rectangle (2,15);
\draw[fill=gray] (2,14) rectangle (3,15);
\draw[fill=gray] (3,14) rectangle (4,15);
\draw[fill=gray] (4,14) rectangle (5,15);
\draw[fill=gray] (5,14) rectangle (6,15);
\draw[fill=gray] (6,14) rectangle (7,15);
\node [align=center] at(-0.500000,0.500000) {0};
\node [align=center] at(-0.500000,1.500000) {1};
\node [align=center] at(-0.500000,2.500000) {2};
\node [align=center] at(-0.500000,3.500000) {3};
\node [align=center] at(-0.500000,4.500000) {4};
\node [align=center] at(-0.500000,5.500000) {5};
\node [align=center] at(-0.500000,6.500000) {6};
\node [align=center] at(-0.500000,7.500000) {7};
\node [align=center] at(-0.500000,8.500000) {8};
\node [align=center] at(-0.500000,9.500000) {9};
\node [align=center] at(-0.500000,10.500000) {10};
\node [align=center] at(-0.500000,11.500000) {11};
\node [align=center] at(-0.500000,12.500000) {12};
\node [align=center] at(-0.500000,13.500000) {13};
\node [align=center] at(-0.500000,14.500000) {14};
\node [align=center] at(0.500000,-0.500000) {A};
\node [align=center] at(1.500000,-0.500000) {B};
\node [align=center] at(2.500000,-0.500000) {C};
\node [align=center] at(3.500000,-0.500000) {D};
\node [align=center] at(4.500000,-0.500000) {E};
\node [align=center] at(5.500000,-0.500000) {F};
\node [align=center] at(6.500000,-0.500000) {G};
\draw[line width=1mm] (3.500000,1.500000) circle (0.300000);
\draw[-Latex] (3.900000,1.500000)--(3.100000,1.500000);
\draw[-Latex] (2.500000,1.100000)--(2.500000,1.900000);
\draw[-Latex] (2.500000,2.100000)--(2.500000,2.900000);
\draw[-Latex] (2.900000,3.500000)--(2.100000,3.500000);
\draw[-Latex] (1.500000,3.100000)--(1.500000,3.900000);
\draw[-Latex] (1.500000,4.100000)--(1.500000,4.900000);
\draw[-Latex] (1.500000,5.100000)--(1.500000,5.900000);
\draw[-Latex] (1.500000,6.900000)--(1.500000,6.100000);
\draw[-Latex] (1.500000,5.100000)--(1.500000,5.900000);
\draw[-Latex] (1.500000,6.100000)--(1.500000,6.900000);
\draw[-Latex] (1.500000,7.100000)--(1.500000,7.900000);
\draw[-Latex] (1.500000,8.100000)--(1.500000,8.900000);
\draw[-Latex] (1.500000,9.100000)--(1.500000,9.900000);
\draw[-Latex] (1.500000,10.100000)--(1.500000,10.900000);
\draw[-Latex] (1.500000,11.100000)--(1.500000,11.900000);
\draw[-Latex] (1.500000,12.100000)--(1.500000,12.900000);
\end{tikzpicture}
    }

    \adjustbox{width=\beliefgraphwidth}{
      \begin{tikzpicture}
\begin{axis}[
ymin=0,
ymax=1,
stack plots=y,
area style,
enlarge x limits=false,
legend pos=south east
]
\addplot coordinates
{(0,0.333333) (1,0.333333) (2,0.333333) (3,0.333333) (4,0.356331) (5,0.358380) (6,0.968442) (7,0.980429) (8,0.987508) (9,0.992310) (10,0.991583) (11,0.991975) (12,0.992081) (13,0.992189) (14,0.992261) (15,0.992299) }
\closedcycle;
\addlegendentry{$\psi_1$}\addplot coordinates
{(0,0.333333) (1,0.333333) (2,0.333333) (3,0.333333) (4,0.287338) (5,0.283240) (6,0.022645) (7,0.014034) (8,0.008886) (9,0.005466) (10,0.005991) (11,0.005715) (12,0.005614) (13,0.005552) (14,0.005500) (15,0.005480) }
\closedcycle;
\addlegendentry{$\psi_2$}\addplot coordinates
{(0,0.333333) (1,0.333333) (2,0.333333) (3,0.333333) (4,0.356331) (5,0.358380) (6,0.008913) (7,0.005537) (8,0.003606) (9,0.002223) (10,0.002426) (11,0.002310) (12,0.002305) (13,0.002259) (14,0.002238) (15,0.002221) }
\closedcycle;
\addlegendentry{$\psi_0 $}\end{axis}
\begin{axis}[
	ymin=0,
	ymax=1,
	xmin=0,
	xmax=17,
	stack plots=y,
	area style,
	enlarge x limits=false,
	legend pos=south east,
	axis x line=none,
	]
	\textBelief{B5}{5}{17}
	\ObservationLine[dashed]{5}
	\textBelief{B5}{7}{17}
	\ObservationLine[dashed]{7}
\end{axis}
\end{tikzpicture}
    }
  }
  \hfill
  \subcaptionbox{
    middle goal cell
    \label{fig|legibility|trajectories|middle}
  }[\mycolwidth]{
    \adjustbox{width=\mazewidth}{
      \begin{tikzpicture}
\draw (0,0) grid (7,15);
\draw[fill=gray] (0,0) rectangle (1,1);
\draw[fill=gray] (1,0) rectangle (2,1);
\draw[fill=gray] (2,0) rectangle (3,1);
\draw[fill=gray] (3,0) rectangle (4,1);
\draw[fill=gray] (4,0) rectangle (5,1);
\draw[fill=gray] (5,0) rectangle (6,1);
\draw[fill=gray] (6,0) rectangle (7,1);
\draw[fill=gray] (0,1) rectangle (1,2);
\draw[fill=cyan!50] (1,1) rectangle (2,2);
\draw[fill=cyan!50] (2,1) rectangle (3,2);
\draw[fill=cyan!50] (3,1) rectangle (4,2);
\draw[fill=cyan!50] (4,1) rectangle (5,2);
\draw[fill=cyan!50] (5,1) rectangle (6,2);
\draw[fill=gray] (6,1) rectangle (7,2);
\draw[fill=gray] (0,2) rectangle (1,3);
\draw[fill=cyan!50] (1,2) rectangle (2,3);
\draw[fill=cyan!50] (2,2) rectangle (3,3);
\draw[fill=cyan!50] (3,2) rectangle (4,3);
\draw[fill=cyan!50] (4,2) rectangle (5,3);
\draw[fill=cyan!50] (5,2) rectangle (6,3);
\draw[fill=gray] (6,2) rectangle (7,3);
\draw[fill=gray] (0,3) rectangle (1,4);
\draw[fill=cyan!50] (1,3) rectangle (2,4);
\draw[fill=cyan!50] (2,3) rectangle (3,4);
\draw[fill=cyan!50] (3,3) rectangle (4,4);
\draw[fill=cyan!50] (4,3) rectangle (5,4);
\draw[fill=cyan!50] (5,3) rectangle (6,4);
\draw[fill=gray] (6,3) rectangle (7,4);
\draw[fill=gray] (0,4) rectangle (1,5);
\draw[fill=cyan!50] (1,4) rectangle (2,5);
\draw[fill=cyan!50] (2,4) rectangle (3,5);
\draw[fill=cyan!50] (3,4) rectangle (4,5);
\draw[fill=cyan!50] (4,4) rectangle (5,5);
\draw[fill=cyan!50] (5,4) rectangle (6,5);
\draw[fill=gray] (6,4) rectangle (7,5);
\draw[fill=gray] (0,5) rectangle (1,6);
\draw[fill=gray] (6,5) rectangle (7,6);
\draw[fill=gray] (0,6) rectangle (1,7);
\draw[fill=cyan!50] (1,6) rectangle (2,7);
\draw[fill=cyan!50] (2,6) rectangle (3,7);
\draw[fill=cyan!50] (3,6) rectangle (4,7);
\draw[fill=cyan!50] (4,6) rectangle (5,7);
\draw[fill=cyan!50] (5,6) rectangle (6,7);
\draw[fill=gray] (6,6) rectangle (7,7);
\draw[fill=gray] (0,7) rectangle (1,8);
\draw[fill=cyan!50] (1,7) rectangle (2,8);
\draw[fill=cyan!50] (2,7) rectangle (3,8);
\draw[fill=cyan!50] (3,7) rectangle (4,8);
\draw[fill=cyan!50] (4,7) rectangle (5,8);
\draw[fill=cyan!50] (5,7) rectangle (6,8);
\draw[fill=gray] (6,7) rectangle (7,8);
\draw[fill=gray] (0,8) rectangle (1,9);
\draw[fill=cyan!50] (1,8) rectangle (2,9);
\draw[fill=cyan!50] (2,8) rectangle (3,9);
\draw[fill=cyan!50] (3,8) rectangle (4,9);
\draw[fill=cyan!50] (4,8) rectangle (5,9);
\draw[fill=cyan!50] (5,8) rectangle (6,9);
\draw[fill=gray] (6,8) rectangle (7,9);
\draw[fill=gray] (0,9) rectangle (1,10);
\draw[fill=cyan!50] (1,9) rectangle (2,10);
\draw[fill=cyan!50] (2,9) rectangle (3,10);
\draw[fill=cyan!50] (3,9) rectangle (4,10);
\draw[fill=cyan!50] (4,9) rectangle (5,10);
\draw[fill=cyan!50] (5,9) rectangle (6,10);
\draw[fill=gray] (6,9) rectangle (7,10);
\draw[fill=gray] (0,10) rectangle (1,11);
\draw[fill=cyan!50] (1,10) rectangle (2,11);
\draw[fill=cyan!50] (2,10) rectangle (3,11);
\draw[fill=cyan!50] (3,10) rectangle (4,11);
\draw[fill=cyan!50] (4,10) rectangle (5,11);
\draw[fill=cyan!50] (5,10) rectangle (6,11);
\draw[fill=gray] (6,10) rectangle (7,11);
\draw[fill=gray] (0,11) rectangle (1,12);
\draw[fill=cyan!50] (1,11) rectangle (2,12);
\draw[fill=cyan!50] (2,11) rectangle (3,12);
\draw[fill=cyan!50] (3,11) rectangle (4,12);
\draw[fill=cyan!50] (4,11) rectangle (5,12);
\draw[fill=cyan!50] (5,11) rectangle (6,12);
\draw[fill=gray] (6,11) rectangle (7,12);
\draw[fill=gray] (0,12) rectangle (1,13);
\draw[fill=cyan!50] (1,12) rectangle (2,13);
\draw[fill=cyan!50] (2,12) rectangle (3,13);
\draw[fill=cyan!50] (3,12) rectangle (4,13);
\draw[fill=cyan!50] (4,12) rectangle (5,13);
\draw[fill=cyan!50] (5,12) rectangle (6,13);
\draw[fill=gray] (6,12) rectangle (7,13);
\draw[fill=gray] (0,13) rectangle (1,14);
\draw[fill=cyan!50] (5,13) rectangle (6,14);
\draw[fill=cyan!50] (1,13) rectangle (2,14);
\begin{scope}[shift={(1.500000 ,13)}] \draw[fill=pink,rotate=45] rectangle (0.707, 0.707);
\end{scope}
\draw[fill=cyan!50] (2,13) rectangle (3,14);
\draw[fill=green] (3.500000,13.500000) circle (0.500000);
\draw[fill=cyan!50] (4,13) rectangle (5,14);
\begin{scope}[shift={(5.500000 ,13)}] \draw[fill=pink,rotate=45] rectangle (0.707, 0.707);
\end{scope}
\draw[fill=gray] (6,13) rectangle (7,14);
\draw[fill=gray] (0,14) rectangle (1,15);
\draw[fill=gray] (1,14) rectangle (2,15);
\draw[fill=gray] (2,14) rectangle (3,15);
\draw[fill=gray] (3,14) rectangle (4,15);
\draw[fill=gray] (4,14) rectangle (5,15);
\draw[fill=gray] (5,14) rectangle (6,15);
\draw[fill=gray] (6,14) rectangle (7,15);
\node [align=center] at(-0.500000,0.500000) {0};
\node [align=center] at(-0.500000,1.500000) {1};
\node [align=center] at(-0.500000,2.500000) {2};
\node [align=center] at(-0.500000,3.500000) {3};
\node [align=center] at(-0.500000,4.500000) {4};
\node [align=center] at(-0.500000,5.500000) {5};
\node [align=center] at(-0.500000,6.500000) {6};
\node [align=center] at(-0.500000,7.500000) {7};
\node [align=center] at(-0.500000,8.500000) {8};
\node [align=center] at(-0.500000,9.500000) {9};
\node [align=center] at(-0.500000,10.500000) {10};
\node [align=center] at(-0.500000,11.500000) {11};
\node [align=center] at(-0.500000,12.500000) {12};
\node [align=center] at(-0.500000,13.500000) {13};
\node [align=center] at(-0.500000,14.500000) {14};
\node [align=center] at(0.500000,-0.500000) {A};
\node [align=center] at(1.500000,-0.500000) {B};
\node [align=center] at(2.500000,-0.500000) {C};
\node [align=center] at(3.500000,-0.500000) {D};
\node [align=center] at(4.500000,-0.500000) {E};
\node [align=center] at(5.500000,-0.500000) {F};
\node [align=center] at(6.500000,-0.500000) {G};
\draw[line width=1mm] (3.500000,1.500000) circle (0.300000);
\draw[-Latex] (3.500000,1.100000)--(3.500000,1.900000);
\draw[-Latex] (3.500000,2.100000)--(3.500000,2.900000);
\draw[-Latex] (3.500000,3.100000)--(3.500000,3.900000);
\draw[-Latex] (3.500000,4.100000)--(3.500000,4.900000);
\draw[-Latex] (3.500000,5.900000)--(3.500000,5.100000);
\draw[-Latex] (3.500000,4.100000)--(3.500000,4.900000);
\draw[-Latex] (3.500000,5.900000)--(3.500000,5.100000);
\draw[-Latex] (3.500000,4.100000)--(3.500000,4.900000);
\draw[-Latex] (3.500000,5.100000)--(3.500000,5.900000);
\draw[-Latex] (3.500000,6.900000)--(3.500000,6.100000);
\draw[-Latex] (3.500000,5.900000)--(3.500000,5.100000);
\draw[-Latex] (3.500000,4.100000)--(3.500000,4.900000);
\draw[-Latex] (3.500000,5.100000)--(3.500000,5.900000);
\draw[-Latex] (3.500000,6.100000)--(3.500000,6.900000);
\draw[-Latex] (3.500000,7.100000)--(3.500000,7.900000);
\draw[-Latex] (3.500000,8.100000)--(3.500000,8.900000);
\draw[-Latex] (3.500000,9.100000)--(3.500000,9.900000);
\draw[-Latex] (3.500000,10.100000)--(3.500000,10.900000);
\draw[-Latex] (3.500000,11.100000)--(3.500000,11.900000);
\draw[-Latex] (3.500000,12.100000)--(3.500000,12.900000);

\node at (1.500000,13.500000) {\huge$\psi_1$};
\node at (3.500000,13.500000) {\huge$\psi_2$};
\node at (5.500000,13.500000) {\huge$\psi_3$};
\end{tikzpicture}
    }

    \adjustbox{width=\beliefgraphwidth}{
      \begin{tikzpicture}
\begin{axis}[
ymin=0,
ymax=1,
xmin=0,
xmax=19,
stack plots=y,
area style,
enlarge x limits=false,
legend pos=south east
]
\addplot coordinates
{(0,0.333333) (1,0.333333) (2,0.333333) (3,0.333333) (4,0.121543) (5,0.084971) (6,0.059093) (7,0.039358) (8,0.026485) (9,0.017215) (10,0.011405) (11,0.007331) (12,0.004824) (13,0.003086) (14,0.003325) (15,0.003199) (16,0.003131) (17,0.003111) (18,0.003082) (19,0.003078) }
\closedcycle;
\addlegendentry{$\psi_1$}\addplot coordinates
{(0,0.333333) (1,0.333333) (2,0.333333) (3,0.333333) (4,0.756913) (5,0.830058) (6,0.881813) (7,0.921284) (8,0.947030) (9,0.965571) (10,0.977190) (11,0.985338) (12,0.990353) (13,0.993828) (14,0.993350) (15,0.993602) (16,0.993737) (17,0.993777) (18,0.993837) (19,0.993845) }
\closedcycle;
\addlegendentry{$\psi_2$}\addplot coordinates
{(0,0.333333) (1,0.333333) (2,0.333333) (3,0.333333) (4,0.121543) (5,0.084971) (6,0.059093) (7,0.039358) (8,0.026485) (9,0.017215) (10,0.011405) (11,0.007331) (12,0.004824) (13,0.003086) (14,0.003325) (15,0.003199) (16,0.003131) (17,0.003111) (18,0.003082) (19,0.003078) }
\closedcycle;
\addlegendentry{$\psi_3 $}\end{axis}


\begin{axis}[
ymin=0,
ymax=1,
xmin=0,
xmax=19,
stack plots=y,
area style,
enlarge x limits=false,
legend pos=south east,
  axis x line=none,
]
\textBelief{D5}{3}{19}
\ObservationLine[dashed]{3}

\textBelief{D5}{5}{19}
\ObservationLine[dashed]{5}

\textBelief{D5}{7}{19}
\ObservationLine[dashed]{7}

\textBelief{D5}{9}{19}
\ObservationLine[dashed]{9}
\end{axis}


\end{tikzpicture}
    }
  }

  \caption{PO-OAMDP trajectories and corresponding belief evolutions for the {\em legibility} task
with $p_\obs=1$ (so that the evolution is deterministic) for two different goal cells
    \label{fig|legibility|trajectories}
  }
\end{figure}

\paragraph{Legibility with Stochastic Observability}

As shown in \Cshref{fig|legibility|trajectories}, with the actual goal on the left, the PO-OAMDP policy depends on $p_{\obs}$.
When $p_\obs=1$ (left), the agent prefers the left (and shortest) path to the goal, where it is easily seen in $(B,4)$.
When $p_\obs=0.5$ (right), the agent prefers the less likely right path, where it is more likely to be seen without having to wait.
The belief evolution is less sudden with the right path, which is less likely for the observer.

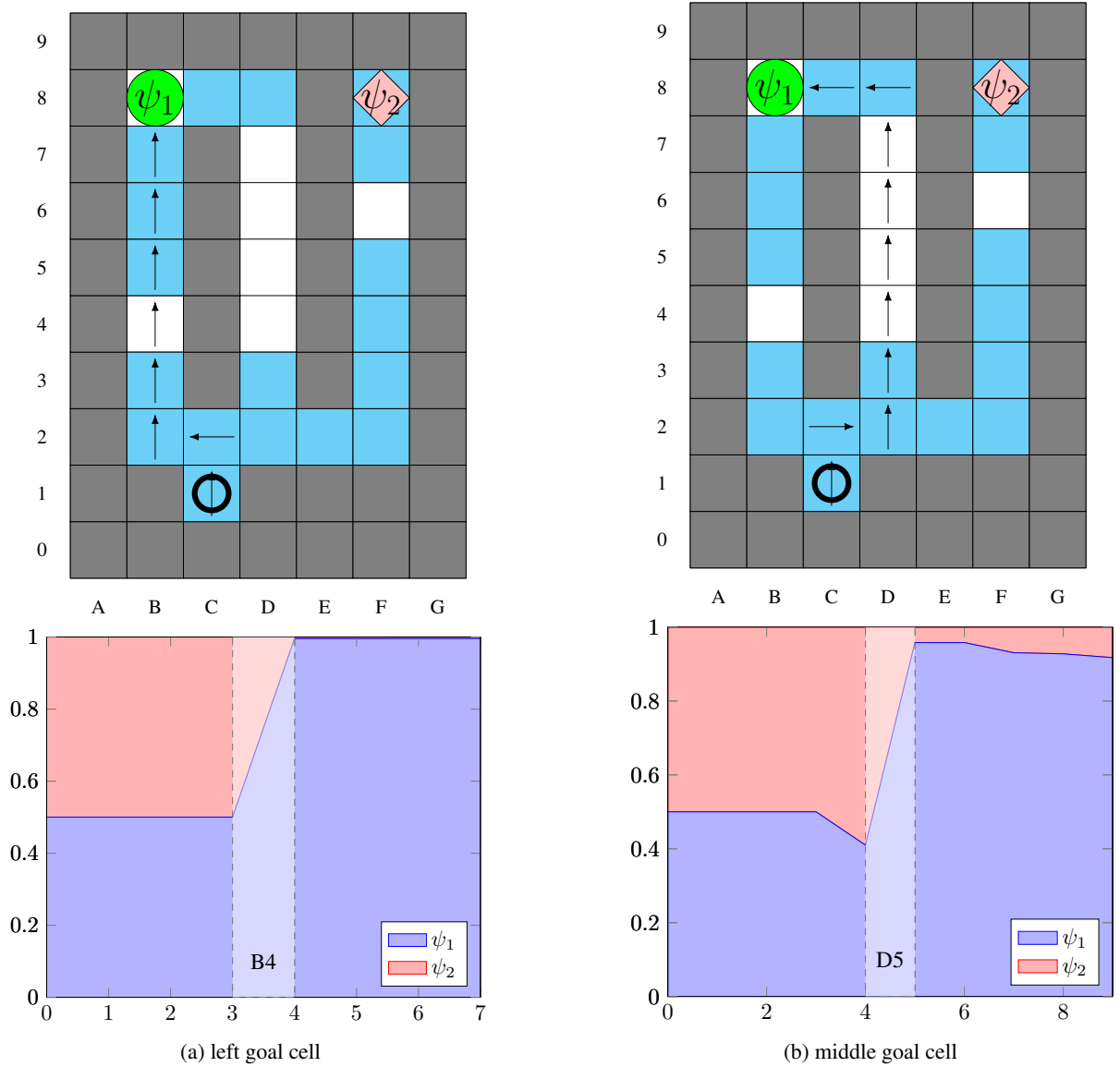
\begin{figure}

	\subcaptionbox{
		left goal cell
	}[\mycolwidth]{
		\adjustbox{width=\mazewidth}{
			\begin{tikzpicture}
\draw (0,0) grid (7,10);
\draw[fill=gray] (0,0) rectangle (1,1);
\draw[fill=gray] (1,0) rectangle (2,1);
\draw[fill=gray] (2,0) rectangle (3,1);
\draw[fill=gray] (3,0) rectangle (4,1);
\draw[fill=gray] (4,0) rectangle (5,1);
\draw[fill=gray] (5,0) rectangle (6,1);
\draw[fill=gray] (6,0) rectangle (7,1);
\draw[fill=gray] (0,1) rectangle (1,2);
\draw[fill=gray] (1,1) rectangle (2,2);
\draw[fill=cyan!50] (2,1) rectangle (3,2);
\draw[fill=gray] (3,1) rectangle (4,2);
\draw[fill=gray] (4,1) rectangle (5,2);
\draw[fill=gray] (5,1) rectangle (6,2);
\draw[fill=gray] (6,1) rectangle (7,2);
\draw[fill=gray] (0,2) rectangle (1,3);
\draw[fill=cyan!50] (1,2) rectangle (2,3);
\draw[fill=cyan!50] (2,2) rectangle (3,3);
\draw[fill=cyan!50] (3,2) rectangle (4,3);
\draw[fill=cyan!50] (4,2) rectangle (5,3);
\draw[fill=cyan!50] (5,2) rectangle (6,3);
\draw[fill=cyan!50] (5,8) rectangle (6,9);
\draw[fill=gray] (6,2) rectangle (7,3);
\draw[fill=gray] (0,3) rectangle (1,4);
\draw[fill=cyan!50] (1,3) rectangle (2,4);
\draw[fill=gray] (2,3) rectangle (3,4);
\draw[fill=cyan!50] (3,3) rectangle (4,4);
\draw[fill=gray] (4,3) rectangle (5,4);
\draw[fill=cyan!50] (5,3) rectangle (6,4);
\draw[fill=gray] (6,3) rectangle (7,4);
\draw[fill=gray] (0,4) rectangle (1,5);
\draw[fill=gray] (2,4) rectangle (3,5);
\draw[fill=gray] (4,4) rectangle (5,5);
\draw[fill=cyan!50] (5,4) rectangle (6,5);
\draw[fill=gray] (6,4) rectangle (7,5);
\draw[fill=gray] (0,5) rectangle (1,6);
\draw[fill=cyan!50] (1,5) rectangle (2,6);
\draw[fill=gray] (2,5) rectangle (3,6);
\draw[fill=gray] (4,5) rectangle (5,6);
\draw[fill=cyan!50] (5,5) rectangle (6,6);
\draw[fill=gray] (6,5) rectangle (7,6);
\draw[fill=gray] (0,6) rectangle (1,7);
\draw[fill=cyan!50] (1,6) rectangle (2,7);
\draw[fill=gray] (2,6) rectangle (3,7);
\draw[fill=gray] (4,6) rectangle (5,7);
\draw[fill=gray] (6,6) rectangle (7,7);
\draw[fill=gray] (0,7) rectangle (1,8);
\draw[fill=cyan!50] (1,7) rectangle (2,8);
\draw[fill=gray] (2,7) rectangle (3,8);
\draw[fill=gray] (4,7) rectangle (5,8);
\draw[fill=cyan!50] (5,7) rectangle (6,8);
\draw[fill=gray] (6,7) rectangle (7,8);
\draw[fill=gray] (0,8) rectangle (1,9);
\draw[fill=green] (1.500000,8.500000) circle (0.500000);
\draw[fill=cyan!50] (2,8) rectangle (3,9);
\draw[fill=cyan!50] (3,8) rectangle (4,9);
\draw[fill=gray] (4,8) rectangle (5,9);
\begin{scope}[shift={(5.500000 ,8)}] \draw[fill=pink,rotate=45] rectangle (0.707, 0.707);
\end{scope}
\draw[fill=gray] (6,8) rectangle (7,9);
\draw[fill=gray] (0,9) rectangle (1,10);
\draw[fill=gray] (1,9) rectangle (2,10);
\draw[fill=gray] (2,9) rectangle (3,10);
\draw[fill=gray] (3,9) rectangle (4,10);
\draw[fill=gray] (4,9) rectangle (5,10);
\draw[fill=gray] (5,9) rectangle (6,10);
\draw[fill=gray] (6,9) rectangle (7,10);
\node [align=center] at(-0.500000,0.500000) {0};
\node [align=center] at(-0.500000,1.500000) {1};
\node [align=center] at(-0.500000,2.500000) {2};
\node [align=center] at(-0.500000,3.500000) {3};
\node [align=center] at(-0.500000,4.500000) {4};
\node [align=center] at(-0.500000,5.500000) {5};
\node [align=center] at(-0.500000,6.500000) {6};
\node [align=center] at(-0.500000,7.500000) {7};
\node [align=center] at(-0.500000,8.500000) {8};
\node [align=center] at(-0.500000,9.500000) {9};
\node [align=center] at(0.500000,-0.500000) {A};
\node [align=center] at(1.500000,-0.500000) {B};
\node [align=center] at(2.500000,-0.500000) {C};
\node [align=center] at(3.500000,-0.500000) {D};
\node [align=center] at(4.500000,-0.500000) {E};
\node [align=center] at(5.500000,-0.500000) {F};
\node [align=center] at(6.500000,-0.500000) {G};
\draw[line width=1mm] (2.500000,1.500000) circle (0.300000);
\draw[-Latex] (2.500000,1.100000)--(2.500000,1.900000);
\draw[-Latex] (2.900000,2.500000)--(2.100000,2.500000);
\draw[-Latex] (1.500000,2.100000)--(1.500000,2.900000);
\draw[-Latex] (1.500000,3.100000)--(1.500000,3.900000);
\draw[-Latex] (1.500000,4.100000)--(1.500000,4.900000);
\draw[-Latex] (1.500000,5.100000)--(1.500000,5.900000);
\draw[-Latex] (1.500000,6.100000)--(1.500000,6.900000);
\draw[-Latex] (1.500000,7.100000)--(1.500000,7.900000);

\node at (1.500000,8.500000) {\huge$\psi_1$};
\node at (5.500000,8.500000) {\huge$\psi_2$};
\end{tikzpicture}
		}

		\adjustbox{width=\beliefgraphwidth}{
			\begin{tikzpicture}
\begin{axis}[
ymin=0,
ymax=1,
stack plots=y,
area style,
enlarge x limits=false,
legend pos=south east
]
\addplot coordinates
{(0,0.500000) (1,0.500000) (2,0.500000) (3,0.500000) (4,0.995652) (5,0.995652) (6,0.995793) (7,0.995793) }
\closedcycle;
\addlegendentry{$\psi_1$}\addplot coordinates
{(0,0.500000) (1,0.500000) (2,0.500000) (3,0.500000) (4,0.004348) (5,0.004348) (6,0.004207) (7,0.004207) }
\closedcycle;
\addlegendentry{$\psi_2 $}\end{axis}

\begin{axis}[
ymin=0,
ymax=1,
xmin=0,
xmax=7,
stack plots=y,
area style,
enlarge x limits=false,
legend pos=south east,
  axis x line=none,
]
\textBelief{B4}{3}{7}
\ObservationLine[dashed]{3}
\end{axis}


\end{tikzpicture}
		}
	}
	\hfill
	\subcaptionbox{
		middle goal cell
	}[\mycolwidth]{
		\adjustbox{width=\mazewidth}{
			\begin{tikzpicture}
\draw (0,0) grid (7,10);
\draw[fill=gray] (0,0) rectangle (1,1);
\draw[fill=gray] (1,0) rectangle (2,1);
\draw[fill=gray] (2,0) rectangle (3,1);
\draw[fill=gray] (3,0) rectangle (4,1);
\draw[fill=gray] (4,0) rectangle (5,1);
\draw[fill=gray] (5,0) rectangle (6,1);
\draw[fill=gray] (6,0) rectangle (7,1);
\draw[fill=gray] (0,1) rectangle (1,2);
\draw[fill=gray] (1,1) rectangle (2,2);
\draw[fill=cyan!50] (2,1) rectangle (3,2);
\draw[fill=gray] (3,1) rectangle (4,2);
\draw[fill=gray] (4,1) rectangle (5,2);
\draw[fill=gray] (5,1) rectangle (6,2);
\draw[fill=gray] (6,1) rectangle (7,2);
\draw[fill=gray] (0,2) rectangle (1,3);
\draw[fill=cyan!50] (1,2) rectangle (2,3);
\draw[fill=cyan!50] (2,2) rectangle (3,3);
\draw[fill=cyan!50] (3,2) rectangle (4,3);
\draw[fill=cyan!50] (4,2) rectangle (5,3);
\draw[fill=cyan!50] (5,2) rectangle (6,3);
\draw[fill=gray] (6,2) rectangle (7,3);
\draw[fill=gray] (0,3) rectangle (1,4);
\draw[fill=cyan!50] (1,3) rectangle (2,4);
\draw[fill=gray] (2,3) rectangle (3,4);
\draw[fill=cyan!50] (3,3) rectangle (4,4);
\draw[fill=gray] (4,3) rectangle (5,4);
\draw[fill=cyan!50] (5,3) rectangle (6,4);
\draw[fill=gray] (6,3) rectangle (7,4);
\draw[fill=gray] (0,4) rectangle (1,5);
\draw[fill=gray] (2,4) rectangle (3,5);
\draw[fill=gray] (4,4) rectangle (5,5);
\draw[fill=cyan!50] (5,4) rectangle (6,5);
\draw[fill=gray] (6,4) rectangle (7,5);
\draw[fill=gray] (0,5) rectangle (1,6);
\draw[fill=cyan!50] (1,5) rectangle (2,6);
\draw[fill=gray] (2,5) rectangle (3,6);
\draw[fill=gray] (4,5) rectangle (5,6);
\draw[fill=cyan!50] (5,5) rectangle (6,6);
\draw[fill=gray] (6,5) rectangle (7,6);
\draw[fill=gray] (0,6) rectangle (1,7);
\draw[fill=cyan!50] (1,6) rectangle (2,7);
\draw[fill=gray] (2,6) rectangle (3,7);
\draw[fill=gray] (4,6) rectangle (5,7);
\draw[fill=gray] (6,6) rectangle (7,7);
\draw[fill=gray] (0,7) rectangle (1,8);
\draw[fill=cyan!50] (1,7) rectangle (2,8);
\draw[fill=gray] (2,7) rectangle (3,8);
\draw[fill=gray] (4,7) rectangle (5,8);
\draw[fill=cyan!50] (5,7) rectangle (6,8);
\draw[fill=cyan!50] (5,8) rectangle (6,9);
\draw[fill=gray] (6,7) rectangle (7,8);
\draw[fill=gray] (0,8) rectangle (1,9);
\draw[fill=green] (1.500000,8.500000) circle (0.500000);
\draw[fill=cyan!50] (2,8) rectangle (3,9);
\draw[fill=cyan!50] (3,8) rectangle (4,9);
\draw[fill=gray] (4,8) rectangle (5,9);
\begin{scope}[shift={(5.500000 ,8)}] \draw[fill=pink,rotate=45] rectangle (0.707, 0.707);
\end{scope}
\draw[fill=gray] (6,8) rectangle (7,9);
\draw[fill=gray] (0,9) rectangle (1,10);
\draw[fill=gray] (1,9) rectangle (2,10);
\draw[fill=gray] (2,9) rectangle (3,10);
\draw[fill=gray] (3,9) rectangle (4,10);
\draw[fill=gray] (4,9) rectangle (5,10);
\draw[fill=gray] (5,9) rectangle (6,10);
\draw[fill=gray] (6,9) rectangle (7,10);
\node [align=center] at(-0.500000,0.500000) {0};
\node [align=center] at(-0.500000,1.500000) {1};
\node [align=center] at(-0.500000,2.500000) {2};
\node [align=center] at(-0.500000,3.500000) {3};
\node [align=center] at(-0.500000,4.500000) {4};
\node [align=center] at(-0.500000,5.500000) {5};
\node [align=center] at(-0.500000,6.500000) {6};
\node [align=center] at(-0.500000,7.500000) {7};
\node [align=center] at(-0.500000,8.500000) {8};
\node [align=center] at(-0.500000,9.500000) {9};
\node [align=center] at(0.500000,-0.500000) {A};
\node [align=center] at(1.500000,-0.500000) {B};
\node [align=center] at(2.500000,-0.500000) {C};
\node [align=center] at(3.500000,-0.500000) {D};
\node [align=center] at(4.500000,-0.500000) {E};
\node [align=center] at(5.500000,-0.500000) {F};
\node [align=center] at(6.500000,-0.500000) {G};
\draw[line width=1mm] (2.500000,1.500000) circle (0.300000);
\draw[-Latex] (2.500000,1.100000)--(2.500000,1.900000);
\draw[-Latex] (2.100000,2.500000)--(2.900000,2.500000);
\draw[-Latex] (3.500000,2.100000)--(3.500000,2.900000);
\draw[-Latex] (3.500000,3.100000)--(3.500000,3.900000);
\draw[-Latex] (3.500000,4.100000)--(3.500000,4.900000);
\draw[-Latex] (3.500000,5.100000)--(3.500000,5.900000);
\draw[-Latex] (3.500000,6.100000)--(3.500000,6.900000);
\draw[-Latex] (3.500000,7.100000)--(3.500000,7.900000);
\draw[-Latex] (3.900000,8.500000)--(3.100000,8.500000);
\draw[-Latex] (2.900000,8.500000)--(2.100000,8.500000);

\node at (1.500000,8.500000) {\huge$\psi_1$};
\node at (5.500000,8.500000) {\huge$\psi_2$};
\end{tikzpicture}
		}

		\adjustbox{width=\beliefgraphwidth}{
			\begin{tikzpicture}
\begin{axis}[
ymin=0,
ymax=1,
stack plots=y,
area style,
enlarge x limits=false,
legend pos=south east
]
\addplot coordinates
{(0,0.500000) (1,0.500000) (2,0.500000) (3,0.500000) (4,0.409873) (5,0.958118) (6,0.958118) (7,0.930649) (8,0.927698) (9,0.917432) }
\closedcycle;
\addlegendentry{$\psi_1$}\addplot coordinates
{(0,0.500000) (1,0.500000) (2,0.500000) (3,0.500000) (4,0.590127) (5,0.041882) (6,0.041882) (7,0.069351) (8,0.072302) (9,0.082568) }
\closedcycle;
\addlegendentry{$\psi_2 $}\end{axis}

\begin{axis}[
ymin=0,
ymax=1,
xmin=0,
xmax=9,
stack plots=y,
area style,
enlarge x limits=false,
legend pos=south east,
  axis x line=none,
]
\textBelief{D5}{4}{9}
\ObservationLine[dashed]{4}
\end{axis}

\end{tikzpicture}
		}
	}

	\caption{PO-OAMDP trajectories and corresponding belief evolutions for the {\em legibility} task
with $p_\obs=1$ and $p_\obs=0.5$  (in this last case, only a sampled belief evolution---in which the agent has been observed in $(D,5)$---is shown).
		\label{fig|legibility|trajectories|partialObs}
	}
\end{figure}

\paragraph{Explicability}

Here, the explicability reward function $\Rexp$ alone was sufficient to obtain proper policies.
This is because a behavior is explicable if it appears to be normal, thus, here, to reach a terminal state as fast as $\piobs$ would do.

As shown in \Cshref{fig|explicability|trajectories}, the agent goes directly to visible cell $(D,5)$, which is consistent with all three possible goals, thus quickly and significantly decreasing the probability of a random behavior (target value $\target_0$), and not trying to bring information about the actual goal (which is then reached as fast as possible).

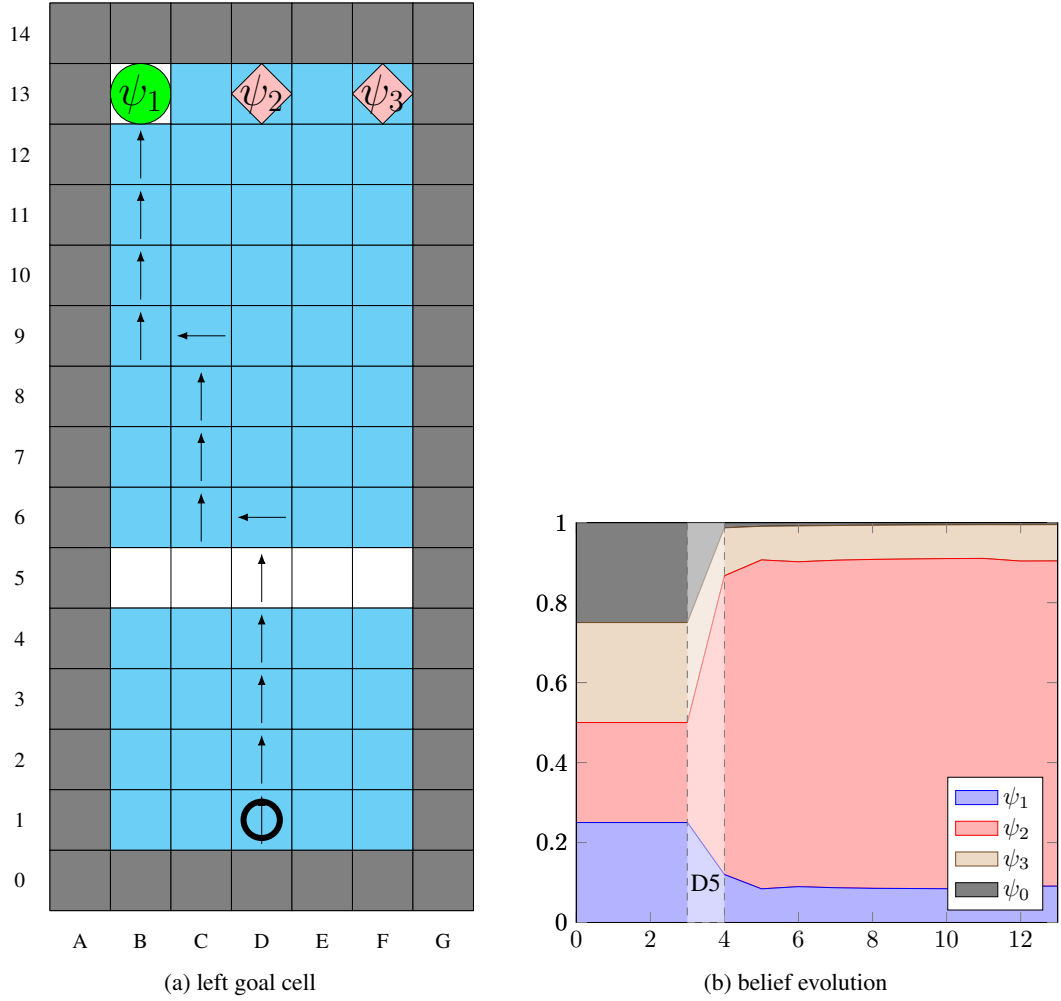
\begin{figure}

	\subcaptionbox{
		left goal cell
	}[\mycolwidth]{
		\adjustbox{width=\mazewidth}{
			\begin{tikzpicture}
\draw (0,0) grid (7,15);
\draw[fill=gray] (0,0) rectangle (1,1);
\draw[fill=gray] (1,0) rectangle (2,1);
\draw[fill=gray] (2,0) rectangle (3,1);
\draw[fill=gray] (3,0) rectangle (4,1);
\draw[fill=gray] (4,0) rectangle (5,1);
\draw[fill=gray] (5,0) rectangle (6,1);
\draw[fill=gray] (6,0) rectangle (7,1);
\draw[fill=gray] (0,1) rectangle (1,2);
\draw[fill=cyan!50] (1,1) rectangle (2,2);
\draw[fill=cyan!50] (2,1) rectangle (3,2);
\draw[fill=cyan!50] (3,1) rectangle (4,2);
\draw[fill=cyan!50] (4,1) rectangle (5,2);
\draw[fill=cyan!50] (5,1) rectangle (6,2);
\draw[fill=gray] (6,1) rectangle (7,2);
\draw[fill=gray] (0,2) rectangle (1,3);
\draw[fill=cyan!50] (1,2) rectangle (2,3);
\draw[fill=cyan!50] (2,2) rectangle (3,3);
\draw[fill=cyan!50] (3,2) rectangle (4,3);
\draw[fill=cyan!50] (4,2) rectangle (5,3);
\draw[fill=cyan!50] (5,2) rectangle (6,3);
\draw[fill=gray] (6,2) rectangle (7,3);
\draw[fill=gray] (0,3) rectangle (1,4);
\draw[fill=cyan!50] (1,3) rectangle (2,4);
\draw[fill=cyan!50] (2,3) rectangle (3,4);
\draw[fill=cyan!50] (3,3) rectangle (4,4);
\draw[fill=cyan!50] (4,3) rectangle (5,4);
\draw[fill=cyan!50] (5,3) rectangle (6,4);
\draw[fill=gray] (6,3) rectangle (7,4);
\draw[fill=gray] (0,4) rectangle (1,5);
\draw[fill=cyan!50] (1,4) rectangle (2,5);
\draw[fill=cyan!50] (2,4) rectangle (3,5);
\draw[fill=cyan!50] (3,4) rectangle (4,5);
\draw[fill=cyan!50] (4,4) rectangle (5,5);
\draw[fill=cyan!50] (5,4) rectangle (6,5);
\draw[fill=gray] (6,4) rectangle (7,5);
\draw[fill=gray] (0,5) rectangle (1,6);
\draw[fill=gray] (6,5) rectangle (7,6);
\draw[fill=gray] (0,6) rectangle (1,7);
\draw[fill=cyan!50] (1,6) rectangle (2,7);
\draw[fill=cyan!50] (2,6) rectangle (3,7);
\draw[fill=cyan!50] (3,6) rectangle (4,7);
\draw[fill=cyan!50] (4,6) rectangle (5,7);
\draw[fill=cyan!50] (5,6) rectangle (6,7);
\draw[fill=gray] (6,6) rectangle (7,7);
\draw[fill=gray] (0,7) rectangle (1,8);
\draw[fill=cyan!50] (1,7) rectangle (2,8);
\draw[fill=cyan!50] (2,7) rectangle (3,8);
\draw[fill=cyan!50] (3,7) rectangle (4,8);
\draw[fill=cyan!50] (4,7) rectangle (5,8);
\draw[fill=cyan!50] (5,7) rectangle (6,8);
\draw[fill=gray] (6,7) rectangle (7,8);
\draw[fill=gray] (0,8) rectangle (1,9);
\draw[fill=cyan!50] (1,8) rectangle (2,9);
\draw[fill=cyan!50] (2,8) rectangle (3,9);
\draw[fill=cyan!50] (3,8) rectangle (4,9);
\draw[fill=cyan!50] (4,8) rectangle (5,9);
\draw[fill=cyan!50] (5,8) rectangle (6,9);
\draw[fill=gray] (6,8) rectangle (7,9);
\draw[fill=gray] (0,9) rectangle (1,10);
\draw[fill=cyan!50] (1,9) rectangle (2,10);
\draw[fill=cyan!50] (2,9) rectangle (3,10);
\draw[fill=cyan!50] (3,9) rectangle (4,10);
\draw[fill=cyan!50] (4,9) rectangle (5,10);
\draw[fill=cyan!50] (5,9) rectangle (6,10);
\draw[fill=gray] (6,9) rectangle (7,10);
\draw[fill=gray] (0,10) rectangle (1,11);
\draw[fill=cyan!50] (1,10) rectangle (2,11);
\draw[fill=cyan!50] (2,10) rectangle (3,11);
\draw[fill=cyan!50] (3,10) rectangle (4,11);
\draw[fill=cyan!50] (4,10) rectangle (5,11);
\draw[fill=cyan!50] (5,10) rectangle (6,11);
\draw[fill=gray] (6,10) rectangle (7,11);
\draw[fill=gray] (0,11) rectangle (1,12);
\draw[fill=cyan!50] (1,11) rectangle (2,12);
\draw[fill=cyan!50] (2,11) rectangle (3,12);
\draw[fill=cyan!50] (3,11) rectangle (4,12);
\draw[fill=cyan!50] (4,11) rectangle (5,12);
\draw[fill=cyan!50] (5,11) rectangle (6,12);
\draw[fill=gray] (6,11) rectangle (7,12);
\draw[fill=gray] (0,12) rectangle (1,13);
\draw[fill=cyan!50] (1,12) rectangle (2,13);
\draw[fill=cyan!50] (2,12) rectangle (3,13);
\draw[fill=cyan!50] (3,12) rectangle (4,13);
\draw[fill=cyan!50] (4,12) rectangle (5,13);
\draw[fill=cyan!50] (5,12) rectangle (6,13);
\draw[fill=cyan!50] (5,13) rectangle (6,14);
\draw[fill=cyan!50] (3,13) rectangle (4,14);
\draw[fill=gray] (6,12) rectangle (7,13);
\draw[fill=gray] (0,13) rectangle (1,14);
\draw[fill=green] (1.500000,13.500000) circle (0.500000);
\draw[fill=cyan!50] (2,13) rectangle (3,14);
\begin{scope}[shift={(3.500000 ,13)}] \draw[fill=pink,rotate=45] rectangle (0.707, 0.707);
\end{scope}
\draw[fill=cyan!50] (4,13) rectangle (5,14);
\begin{scope}[shift={(5.500000 ,13)}] \draw[fill=pink,rotate=45] rectangle (0.707, 0.707);
\end{scope}
\draw[fill=gray] (6,13) rectangle (7,14);
\draw[fill=gray] (0,14) rectangle (1,15);
\draw[fill=gray] (1,14) rectangle (2,15);
\draw[fill=gray] (2,14) rectangle (3,15);
\draw[fill=gray] (3,14) rectangle (4,15);
\draw[fill=gray] (4,14) rectangle (5,15);
\draw[fill=gray] (5,14) rectangle (6,15);
\draw[fill=gray] (6,14) rectangle (7,15);
\node [align=center] at(-0.500000,0.500000) {0};
\node [align=center] at(-0.500000,1.500000) {1};
\node [align=center] at(-0.500000,2.500000) {2};
\node [align=center] at(-0.500000,3.500000) {3};
\node [align=center] at(-0.500000,4.500000) {4};
\node [align=center] at(-0.500000,5.500000) {5};
\node [align=center] at(-0.500000,6.500000) {6};
\node [align=center] at(-0.500000,7.500000) {7};
\node [align=center] at(-0.500000,8.500000) {8};
\node [align=center] at(-0.500000,9.500000) {9};
\node [align=center] at(-0.500000,10.500000) {10};
\node [align=center] at(-0.500000,11.500000) {11};
\node [align=center] at(-0.500000,12.500000) {12};
\node [align=center] at(-0.500000,13.500000) {13};
\node [align=center] at(-0.500000,14.500000) {14};
\node [align=center] at(0.500000,-0.500000) {A};
\node [align=center] at(1.500000,-0.500000) {B};
\node [align=center] at(2.500000,-0.500000) {C};
\node [align=center] at(3.500000,-0.500000) {D};
\node [align=center] at(4.500000,-0.500000) {E};
\node [align=center] at(5.500000,-0.500000) {F};
\node [align=center] at(6.500000,-0.500000) {G};
\draw[line width=1mm] (3.500000,1.500000) circle (0.300000);
\draw[-Latex] (3.500000,1.100000)--(3.500000,1.900000);
\draw[-Latex] (3.500000,2.100000)--(3.500000,2.900000);
\draw[-Latex] (3.500000,3.100000)--(3.500000,3.900000);
\draw[-Latex] (3.500000,4.100000)--(3.500000,4.900000);
\draw[-Latex] (3.500000,5.100000)--(3.500000,5.900000);
\draw[-Latex] (3.900000,6.500000)--(3.100000,6.500000);
\draw[-Latex] (2.500000,6.100000)--(2.500000,6.900000);
\draw[-Latex] (2.500000,7.100000)--(2.500000,7.900000);
\draw[-Latex] (2.500000,8.100000)--(2.500000,8.900000);
\draw[-Latex] (2.900000,9.500000)--(2.100000,9.500000);
\draw[-Latex] (1.500000,9.100000)--(1.500000,9.900000);
\draw[-Latex] (1.500000,10.100000)--(1.500000,10.900000);
\draw[-Latex] (1.500000,11.100000)--(1.500000,11.900000);
\draw[-Latex] (1.500000,12.100000)--(1.500000,12.900000);

\node at (1.500000,13.500000) {\huge$\psi_1$};
\node at (3.500000,13.500000) {\huge$\psi_2$};
\node at (5.500000,13.500000) {\huge$\psi_3$};
\end{tikzpicture}
		}
  }
  \subcaptionbox{
  belief evolution
  }
  [\mycolwidth]{
		\adjustbox{width=\beliefgraphwidth}{
			\begin{tikzpicture}
\begin{axis}[
ymin=0,
ymax=1,
stack plots=y,
area style,
enlarge x limits=false,
legend pos=south east
]
\addplot coordinates
{(0,0.250000) (1,0.250000) (2,0.250000) (3,0.250000) (4,0.119949) (5,0.084188) (6,0.089647) (7,0.086919) (8,0.085443) (9,0.085037) (10,0.084398) (11,0.084334) (12,0.090864) (13,0.090880) }
\closedcycle;
\addlegendentry{$\psi_1$}\addplot coordinates
{(0,0.250000) (1,0.250000) (2,0.250000) (3,0.250000) (4,0.746983) (5,0.822415) (6,0.812344) (7,0.818910) (8,0.822470) (9,0.823879) (10,0.825571) (11,0.826079) (12,0.812915) (13,0.813172) }
\closedcycle;
\addlegendentry{$\psi_2$}\addplot coordinates
{(0,0.250000) (1,0.250000) (2,0.250000) (3,0.250000) (4,0.119949) (5,0.084188) (6,0.089647) (7,0.086919) (8,0.085443) (9,0.085037) (10,0.084398) (11,0.084334) (12,0.090864) (13,0.090880) }
\closedcycle;
\addlegendentry{$\psi_3$}\addplot coordinates
{(0,0.250000) (1,0.250000) (2,0.250000) (3,0.250000) (4,0.013119) (5,0.009208) (6,0.008362) (7,0.007252) (8,0.006645) (9,0.006046) (10,0.005634) (11,0.005253) (12,0.005356) (13,0.005068) }
\closedcycle;
\addlegendentry{$\psi_0 $}\end{axis}

\begin{axis}[
ymin=0,
ymax=1,
xmin=0,
xmax=13,
stack plots=y,
area style,
enlarge x limits=false,
legend pos=south east,
  axis x line=none,
]
\textBelief{D5}{3}{13}
\ObservationLine[dashed]{3}

\end{axis}

\end{tikzpicture}
		}
	}
	\caption{PO-OAMDP trajectory and corresponding belief evolution for the {\em explicability} task
with $p_\obs=1$ (so that the evolution is deterministic) for the left goal cell
		\label{fig|explicability|trajectories}
	}
\end{figure}

\paragraph{Action Predictability}
As shown in \Cshref{fig|predictability|trajectories|actions|noRobs,fig|predictability|trajectories|actions}, action predictability here requires combining $\RApred$ with $\Robs$ to obtain a proper policy.
Indeed, without $\Robs$, the observer can keep on correctly predicting action $down$, believing it is most probably in cell $(B,2)$, while it is actually in $(F,1)$ (\Cshref{fig|predictability|trajectories|actions|noRobs}, truncated trajectory).
This rightmost trajectory if preferred over going through the empty room, where the action sequence is less predictable.
Adding $\Robs$ fixes this issue by making sure that an infinite trajectory induces an infinite cost.
Then, the best option is to go through $(D,10)$ so as to reduce the uncertainty about the trajectory early on, despite the traversal of the empty room afterwards.
Note: In a smaller version of this environment, with no empty room, adding $\Robs$ is not required, \cf \Cshref{app|XP|predictability}, \Cshref{RPredictability}.

\paragraph{State Predictability}
\Cshref{fig|predictability|trajectories|states} shows a behavior very similar to \Cshref{fig|predictability|trajectories|actions}, but for coming back once in $(B,5)$ after reaching $(B,6)$ for the first time (see also $(B,2)$).
This is because $\piobs$ makes it likely enough that the agent randomly moved backward at some point.

\def\mycolwidthC{0.32\columnwidth}
\def\mazewidthC{0.32\columnwidth}
\def\beliefgraphwidthC{0.32\columnwidth}

\begin{figure}

  \subcaptionbox{
    action pred.\\
    without $\Robs$
    \label{fig|predictability|trajectories|actions|noRobs}
  }[\mycolwidthC]{
    \adjustbox{width=\mazewidthC}{
      \begin{tikzpicture}
\draw (0,0) grid (7,13);
\draw[fill=gray] (0,0) rectangle (1,1);
\draw[fill=gray] (1,0) rectangle (2,1);
\draw[fill=gray] (2,0) rectangle (3,1);
\draw[fill=gray] (3,0) rectangle (4,1);
\draw[fill=gray] (4,0) rectangle (5,1);
\draw[fill=gray] (5,0) rectangle (6,1);
\draw[fill=gray] (6,0) rectangle (7,1);
\draw[fill=gray] (0,1) rectangle (1,2);
\draw[fill=green] (1.500000,1.500000) circle (0.500000);
\draw[fill=cyan!50] (2,1) rectangle (3,2);
\draw[fill=cyan!50] (3,1) rectangle (4,2);
\draw[fill=cyan!50] (4,1) rectangle (5,2);
\draw[fill=cyan!50] (5,1) rectangle (6,2);
\draw[fill=gray] (6,1) rectangle (7,2);
\draw[fill=gray] (0,2) rectangle (1,3);
\draw[fill=cyan!50] (1,2) rectangle (2,3);
\draw[fill=gray] (2,2) rectangle (3,3);
\draw[fill=gray] (4,2) rectangle (5,3);
\draw[fill=cyan!50] (5,2) rectangle (6,3);
\draw[fill=gray] (6,2) rectangle (7,3);
\draw[fill=gray] (0,3) rectangle (1,4);
\draw[fill=cyan!50] (1,3) rectangle (2,4);
\draw[fill=gray] (2,3) rectangle (3,4);
\draw[fill=cyan!50] (3,3) rectangle (4,4);
\draw[fill=cyan!50] (4,3) rectangle (5,4);
\draw[fill=cyan!50] (5,3) rectangle (6,4);
\draw[fill=gray] (6,3) rectangle (7,4);
\draw[fill=gray] (0,4) rectangle (1,5);
\draw[fill=cyan!50] (1,4) rectangle (2,5);
\draw[fill=gray] (2,4) rectangle (3,5);
\draw[fill=gray] (3,4) rectangle (4,5);
\draw[fill=gray] (4,4) rectangle (5,5);
\draw[fill=cyan!50] (5,4) rectangle (6,5);
\draw[fill=gray] (6,4) rectangle (7,5);
\draw[fill=gray] (0,5) rectangle (1,6);
\draw[fill=cyan!50] (1,5) rectangle (2,6);
\draw[fill=cyan!50] (2,5) rectangle (3,6);
\draw[fill=cyan!50] (3,5) rectangle (4,6);
\draw[fill=gray] (4,5) rectangle (5,6);
\draw[fill=cyan!50] (5,5) rectangle (6,6);
\draw[fill=gray] (6,5) rectangle (7,6);
\draw[fill=gray] (0,6) rectangle (1,7);
\draw[fill=cyan!50] (1,6) rectangle (2,7);
\draw[fill=cyan!50] (2,6) rectangle (3,7);
\draw[fill=cyan!50] (3,6) rectangle (4,7);
\draw[fill=gray] (4,6) rectangle (5,7);
\draw[fill=cyan!50] (5,6) rectangle (6,7);
\draw[fill=gray] (6,6) rectangle (7,7);
\draw[fill=gray] (0,7) rectangle (1,8);
\draw[fill=cyan!50] (1,7) rectangle (2,8);
\draw[fill=cyan!50] (2,7) rectangle (3,8);
\draw[fill=cyan!50] (3,7) rectangle (4,8);
\draw[fill=gray] (4,7) rectangle (5,8);
\draw[fill=cyan!50] (5,7) rectangle (6,8);
\draw[fill=gray] (6,7) rectangle (7,8);
\draw[fill=gray] (0,8) rectangle (1,9);
\draw[fill=cyan!50] (1,8) rectangle (2,9);
\draw[fill=cyan!50] (2,8) rectangle (3,9);
\draw[fill=cyan!50] (3,8) rectangle (4,9);
\draw[fill=gray] (4,8) rectangle (5,9);
\draw[fill=cyan!50] (5,8) rectangle (6,9);
\draw[fill=gray] (6,8) rectangle (7,9);
\draw[fill=gray] (0,9) rectangle (1,10);
\draw[fill=cyan!50] (1,9) rectangle (2,10);
\draw[fill=cyan!50] (2,9) rectangle (3,10);
\draw[fill=cyan!50] (3,9) rectangle (4,10);
\draw[fill=gray] (4,9) rectangle (5,10);
\draw[fill=cyan!50] (5,9) rectangle (6,10);
\draw[fill=gray] (6,9) rectangle (7,10);
\draw[fill=gray] (0,10) rectangle (1,11);
\draw[fill=cyan!50] (1,10) rectangle (2,11);
\draw[fill=gray] (2,10) rectangle (3,11);
\draw[fill=gray] (4,10) rectangle (5,11);
\draw[fill=cyan!50] (5,10) rectangle (6,11);
\draw[fill=gray] (6,10) rectangle (7,11);
\draw[fill=gray] (0,11) rectangle (1,12);
\draw[fill=cyan!50] (1,11) rectangle (2,12);
\draw[fill=cyan!50] (2,11) rectangle (3,12);
\draw[fill=cyan!50] (3,11) rectangle (4,12);
\draw[fill=cyan!50] (4,11) rectangle (5,12);
\draw[fill=cyan!50] (5,11) rectangle (6,12);
\draw[fill=gray] (6,11) rectangle (7,12);
\draw[fill=gray] (0,12) rectangle (1,13);
\draw[fill=gray] (1,12) rectangle (2,13);
\draw[fill=gray] (2,12) rectangle (3,13);
\draw[fill=gray] (3,12) rectangle (4,13);
\draw[fill=gray] (4,12) rectangle (5,13);
\draw[fill=gray] (5,12) rectangle (6,13);
\draw[fill=gray] (6,12) rectangle (7,13);
\node [align=center] at(-0.500000,0.500000) {0};
\node [align=center] at(-0.500000,1.500000) {1};
\node [align=center] at(-0.500000,2.500000) {2};
\node [align=center] at(-0.500000,3.500000) {3};
\node [align=center] at(-0.500000,4.500000) {4};
\node [align=center] at(-0.500000,5.500000) {5};
\node [align=center] at(-0.500000,6.500000) {6};
\node [align=center] at(-0.500000,7.500000) {7};
\node [align=center] at(-0.500000,8.500000) {8};
\node [align=center] at(-0.500000,9.500000) {9};
\node [align=center] at(-0.500000,10.500000) {10};
\node [align=center] at(-0.500000,11.500000) {11};
\node [align=center] at(-0.500000,12.500000) {12};
\node [align=center] at(0.500000,-0.500000) {A};
\node [align=center] at(1.500000,-0.500000) {B};
\node [align=center] at(2.500000,-0.500000) {C};
\node [align=center] at(3.500000,-0.500000) {D};
\node [align=center] at(4.500000,-0.500000) {E};
\node [align=center] at(5.500000,-0.500000) {F};
\node [align=center] at(6.500000,-0.500000) {G};
\draw[line width=1mm] (5.500000,11.500000) circle (0.300000);
\draw[-Latex] (5.500000,11.900000)--(5.500000,11.100000);
\draw[-Latex] (5.500000,10.900000)--(5.500000,10.100000);
\draw[-Latex] (5.500000,9.900000)--(5.500000,9.100000);
\draw[-Latex] (5.500000,8.900000)--(5.500000,8.100000);
\draw[-Latex] (5.500000,7.900000)--(5.500000,7.100000);
\draw[-Latex] (5.500000,6.900000)--(5.500000,6.100000);
\draw[-Latex] (5.500000,5.900000)--(5.500000,5.100000);
\draw[-Latex] (5.500000,4.900000)--(5.500000,4.100000);
\draw[-Latex] (5.500000,3.900000)--(5.500000,3.100000);
\draw[-Latex] (5.500000,2.900000)--(5.500000,2.100000);
\draw[-Latex] (5.500000,1.900000)--(5.500000,1.100000);
\draw[-Latex] (5.500000,1.900000)--(5.500000,1.100000);
\draw[-Latex] (5.500000,1.900000)--(5.500000,1.100000);
\draw[-Latex] (5.500000,1.900000)--(5.500000,1.100000);
\draw[-Latex] (5.500000,1.900000)--(5.500000,1.100000);
\draw[-Latex] (5.500000,1.900000)--(5.500000,1.100000);
\draw[-Latex] (5.500000,1.900000)--(5.500000,1.100000);
\draw[-Latex] (5.500000,1.900000)--(5.500000,1.100000);
\draw[-Latex] (5.500000,1.900000)--(5.500000,1.100000);
\draw[-Latex] (5.500000,1.900000)--(5.500000,1.100000);
\end{tikzpicture}
    }
    \adjustbox{width=\beliefgraphwidthC}{
      \pgfplotstableread{
Label up down left right
0 2.1209224212338978E-44 0.5 0.5 2.1209224212338978E-44 
1 0.07425933903635351 0.4257406609636465 0.4257406609636465 0.07425933903635351 
2 0.062286637229939934 0.6337830516499655 0.2703290279162589 0.0336012832038356 
3 0.09036840192179904 0.5315760540638989 0.32298721953906157 0.05506832447524064 
4 0.08559199481624659 0.7816278288130167 0.0913722022163842 0.04140797415435256 
5 0.12138019463958422 0.7360299802146985 0.12188056588230635 0.020709259263411134 
6 0.11218672448315092 0.8190073942220453 0.028473136937879805 0.040332744356924105 
7 0.1267215876134224 0.7971318090293172 0.05209302277412693 0.02405358058313358 
8 0.10734909105719201 0.7262137826585235 0.12909245852538717 0.03734466775889713 
9 0.11152895573588535 0.7161007956212284 0.1293318981998892 0.04303835044299715 
10 0.09567963754304136 0.6348582316429456 0.22226480277484134 0.04719732803917152 
11 0.09506468326874089 0.6068434428250686 0.23946421849946609 0.05862765540672446 
12 0.09446986935770588 0.5570637221374637 0.29195628565555104 0.05651012284927935 
13 0.08228245390777465 0.5214063816675447 0.3292219924474568 0.06708917197722342 
14 0.09723248178559661 0.5187579193470419 0.32286982680078663 0.06113977206657493 
15 0.07700323733091632 0.48746514714083594 0.36684134696493653 0.06869026856331142 
16 0.10052534761961097 0.509433555429464 0.32744378003599756 0.06259731691492745 
17 0.07661291891701026 0.4865819713544796 0.37012122598809555 0.06668388374041444 
18 0.10348059504253171 0.5175457917979317 0.3166795792627402 0.062294033896796615 
19 0.07881234195879351 0.5028770752874915 0.3552715930993853 0.0630389896543295 
}\testdata\begin{tikzpicture}
\begin{axis}[
ybar stacked,
 ymin=0,
ymax=1,
 xtick=data,
legend style={cells={anchor=west}, legend pos=north west},
reverse legend=true,
xticklabels from table={\testdata}{Label},
xticklabel style={text width=2cm,align=center},
]
\addplot [fill=green!80] table [y=up, meta=Label, x expr=\coordindex] {\testdata};
\addlegendentry{up}
\addplot [fill=blue!60] table [y=down, meta=Label, x expr=\coordindex] {\testdata};
\addlegendentry{down}
\addplot [fill=red!60] table [y=left, meta=Label, x expr=\coordindex] {\testdata};
\addlegendentry{left}
\addplot [fill=yellow!60] table [y=right, meta=Label, x expr=\coordindex] {\testdata};
\addlegendentry{right}
\end{axis}\end{tikzpicture}
    }
  }
  \hfill
  \subcaptionbox{
    action pred.
    \label{fig|predictability|trajectories|actions}
  }[\mycolwidthC]{
    \adjustbox{width=\mazewidthC}{
      \begin{tikzpicture}
\draw (0,0) grid (7,13);
\draw[fill=gray] (0,0) rectangle (1,1);
\draw[fill=gray] (1,0) rectangle (2,1);
\draw[fill=gray] (2,0) rectangle (3,1);
\draw[fill=gray] (3,0) rectangle (4,1);
\draw[fill=gray] (4,0) rectangle (5,1);
\draw[fill=gray] (5,0) rectangle (6,1);
\draw[fill=gray] (6,0) rectangle (7,1);
\draw[fill=gray] (0,1) rectangle (1,2);
\draw[fill=green] (1.500000,1.500000) circle (0.500000);
\draw[fill=cyan!50] (2,1) rectangle (3,2);
\draw[fill=cyan!50] (3,1) rectangle (4,2);
\draw[fill=cyan!50] (4,1) rectangle (5,2);
\draw[fill=cyan!50] (5,1) rectangle (6,2);
\draw[fill=gray] (6,1) rectangle (7,2);
\draw[fill=gray] (0,2) rectangle (1,3);
\draw[fill=cyan!50] (1,2) rectangle (2,3);
\draw[fill=gray] (2,2) rectangle (3,3);
\draw[fill=gray] (4,2) rectangle (5,3);
\draw[fill=cyan!50] (5,2) rectangle (6,3);
\draw[fill=gray] (6,2) rectangle (7,3);
\draw[fill=gray] (0,3) rectangle (1,4);
\draw[fill=cyan!50] (1,3) rectangle (2,4);
\draw[fill=gray] (2,3) rectangle (3,4);
\draw[fill=cyan!50] (3,3) rectangle (4,4);
\draw[fill=cyan!50] (4,3) rectangle (5,4);
\draw[fill=cyan!50] (5,3) rectangle (6,4);
\draw[fill=gray] (6,3) rectangle (7,4);
\draw[fill=gray] (0,4) rectangle (1,5);
\draw[fill=cyan!50] (1,4) rectangle (2,5);
\draw[fill=gray] (2,4) rectangle (3,5);
\draw[fill=gray] (3,4) rectangle (4,5);
\draw[fill=gray] (4,4) rectangle (5,5);
\draw[fill=cyan!50] (5,4) rectangle (6,5);
\draw[fill=gray] (6,4) rectangle (7,5);
\draw[fill=gray] (0,5) rectangle (1,6);
\draw[fill=cyan!50] (1,5) rectangle (2,6);
\draw[fill=cyan!50] (2,5) rectangle (3,6);
\draw[fill=cyan!50] (3,5) rectangle (4,6);
\draw[fill=gray] (4,5) rectangle (5,6);
\draw[fill=cyan!50] (5,5) rectangle (6,6);
\draw[fill=gray] (6,5) rectangle (7,6);
\draw[fill=gray] (0,6) rectangle (1,7);
\draw[fill=cyan!50] (1,6) rectangle (2,7);
\draw[fill=cyan!50] (2,6) rectangle (3,7);
\draw[fill=cyan!50] (3,6) rectangle (4,7);
\draw[fill=gray] (4,6) rectangle (5,7);
\draw[fill=cyan!50] (5,6) rectangle (6,7);
\draw[fill=gray] (6,6) rectangle (7,7);
\draw[fill=gray] (0,7) rectangle (1,8);
\draw[fill=cyan!50] (1,7) rectangle (2,8);
\draw[fill=cyan!50] (2,7) rectangle (3,8);
\draw[fill=cyan!50] (3,7) rectangle (4,8);
\draw[fill=gray] (4,7) rectangle (5,8);
\draw[fill=cyan!50] (5,7) rectangle (6,8);
\draw[fill=gray] (6,7) rectangle (7,8);
\draw[fill=gray] (0,8) rectangle (1,9);
\draw[fill=cyan!50] (1,8) rectangle (2,9);
\draw[fill=cyan!50] (2,8) rectangle (3,9);
\draw[fill=cyan!50] (3,8) rectangle (4,9);
\draw[fill=gray] (4,8) rectangle (5,9);
\draw[fill=cyan!50] (5,8) rectangle (6,9);
\draw[fill=gray] (6,8) rectangle (7,9);
\draw[fill=gray] (0,9) rectangle (1,10);
\draw[fill=cyan!50] (1,9) rectangle (2,10);
\draw[fill=cyan!50] (2,9) rectangle (3,10);
\draw[fill=cyan!50] (3,9) rectangle (4,10);
\draw[fill=gray] (4,9) rectangle (5,10);
\draw[fill=cyan!50] (5,9) rectangle (6,10);
\draw[fill=gray] (6,9) rectangle (7,10);
\draw[fill=gray] (0,10) rectangle (1,11);
\draw[fill=cyan!50] (1,10) rectangle (2,11);
\draw[fill=gray] (2,10) rectangle (3,11);
\draw[fill=gray] (4,10) rectangle (5,11);
\draw[fill=cyan!50] (5,10) rectangle (6,11);
\draw[fill=gray] (6,10) rectangle (7,11);
\draw[fill=gray] (0,11) rectangle (1,12);
\draw[fill=cyan!50] (1,11) rectangle (2,12);
\draw[fill=cyan!50] (2,11) rectangle (3,12);
\draw[fill=cyan!50] (3,11) rectangle (4,12);
\draw[fill=cyan!50] (4,11) rectangle (5,12);
\draw[fill=cyan!50] (5,11) rectangle (6,12);
\draw[fill=gray] (6,11) rectangle (7,12);
\draw[fill=gray] (0,12) rectangle (1,13);
\draw[fill=gray] (1,12) rectangle (2,13);
\draw[fill=gray] (2,12) rectangle (3,13);
\draw[fill=gray] (3,12) rectangle (4,13);
\draw[fill=gray] (4,12) rectangle (5,13);
\draw[fill=gray] (5,12) rectangle (6,13);
\draw[fill=gray] (6,12) rectangle (7,13);
\node [align=center] at(-0.500000,0.500000) {0};
\node [align=center] at(-0.500000,1.500000) {1};
\node [align=center] at(-0.500000,2.500000) {2};
\node [align=center] at(-0.500000,3.500000) {3};
\node [align=center] at(-0.500000,4.500000) {4};
\node [align=center] at(-0.500000,5.500000) {5};
\node [align=center] at(-0.500000,6.500000) {6};
\node [align=center] at(-0.500000,7.500000) {7};
\node [align=center] at(-0.500000,8.500000) {8};
\node [align=center] at(-0.500000,9.500000) {9};
\node [align=center] at(-0.500000,10.500000) {10};
\node [align=center] at(-0.500000,11.500000) {11};
\node [align=center] at(-0.500000,12.500000) {12};
\node [align=center] at(0.500000,-0.500000) {A};
\node [align=center] at(1.500000,-0.500000) {B};
\node [align=center] at(2.500000,-0.500000) {C};
\node [align=center] at(3.500000,-0.500000) {D};
\node [align=center] at(4.500000,-0.500000) {E};
\node [align=center] at(5.500000,-0.500000) {F};
\node [align=center] at(6.500000,-0.500000) {G};
\draw[line width=1mm] (5.500000,11.500000) circle (0.300000);
\draw[-Latex] (5.900000,11.500000)--(5.100000,11.500000);
\draw[-Latex] (4.900000,11.500000)--(4.100000,11.500000);
\draw[-Latex] (3.500000,11.900000)--(3.500000,11.100000);
\draw[-Latex] (3.500000,10.900000)--(3.500000,10.100000);
\draw[-Latex] (3.500000,9.900000)--(3.500000,9.100000);
\draw[-Latex] (3.500000,8.900000)--(3.500000,8.100000);
\draw[-Latex] (3.500000,7.900000)--(3.500000,7.100000);
\draw[-Latex] (3.500000,6.900000)--(3.500000,6.100000);
\draw[-Latex] (3.900000,5.500000)--(3.100000,5.500000);
\draw[-Latex] (2.900000,5.500000)--(2.100000,5.500000);
\draw[-Latex] (1.500000,5.900000)--(1.500000,5.100000);
\draw[-Latex] (1.500000,4.900000)--(1.500000,4.100000);
\draw[-Latex] (1.500000,3.900000)--(1.500000,3.100000);
\draw[-Latex] (1.500000,2.900000)--(1.500000,2.100000);
\end{tikzpicture}
    }

    \adjustbox{width=\beliefgraphwidthC}{
      \pgfplotstableread{
Label up down left right
0 2.1209224212338978E-44 0.5 0.5 2.1209224212338978E-44 
1 0.07425933903635351 0.4257406609636465 0.4257406609636465 0.07425933903635351 
2 0.062286637229939934 0.6337830516499655 0.2703290279162589 0.0336012832038356 
3 0.09620747102702867 0.5668279475292698 0.2878816485218043 0.049082932921897374 
4 0.08339264769393778 0.7468477601426343 0.1316162102080164 0.03814338195541162 
5 0.11372017894799637 0.7098480145487641 0.15394407983856184 0.022487726664677698 
6 0.10671401217298966 0.7773591941606546 0.07427125089414634 0.04165554277220887 
7 0.11968284416741962 0.7637113927719823 0.08753846540991501 0.029067297650683485 
8 0.1046161587053023 0.7045491352849218 0.14780305080943912 0.043031655200336826 
9 0.10908185351448457 0.6973582045442289 0.14457795295508688 0.048981988986199834 
10 0.09618366426985556 0.6351496170550601 0.21488798696631514 0.05377873170876888 
11 0.09801458535799021 0.6205656136913287 0.22221770397807158 0.05920209697260946 
12 0.09647627713881365 0.5623726190726385 0.2787561089728161 0.062394994815731804 
13 0.08721220977851685 0.5514905617979337 0.2976456787280396 0.06365154969551026 
}\testdata\begin{tikzpicture}
\begin{axis}[
ybar stacked,
 ymin=0,
ymax=1,
 xtick=data,
legend style={cells={anchor=west}, legend pos=north west},
reverse legend=true,
xticklabels from table={\testdata}{Label},
xticklabel style={text width=2cm,align=center},
]
\addplot [fill=green!80] table [y=up, meta=Label, x expr=\coordindex] {\testdata};
\addlegendentry{up}
\addplot [fill=blue!60] table [y=down, meta=Label, x expr=\coordindex] {\testdata};
\addlegendentry{down}
\addplot [fill=red!60] table [y=left, meta=Label, x expr=\coordindex] {\testdata};
\addlegendentry{left}
\addplot [fill=yellow!60] table [y=right, meta=Label, x expr=\coordindex] {\testdata};
\addlegendentry{right}
\end{axis}\end{tikzpicture}
    }
  }
  \hfill
  \subcaptionbox{
    state pred.
    \label{fig|predictability|trajectories|states}
  }[\mycolwidthC]{
    \adjustbox{width=\mazewidthC}{
      \begin{tikzpicture}
\draw (0,0) grid (7,13);
\draw[fill=gray] (0,0) rectangle (1,1);
\draw[fill=gray] (1,0) rectangle (2,1);
\draw[fill=gray] (2,0) rectangle (3,1);
\draw[fill=gray] (3,0) rectangle (4,1);
\draw[fill=gray] (4,0) rectangle (5,1);
\draw[fill=gray] (5,0) rectangle (6,1);
\draw[fill=gray] (6,0) rectangle (7,1);
\draw[fill=gray] (0,1) rectangle (1,2);
\draw[fill=green] (1.500000,1.500000) circle (0.500000);
\draw[fill=cyan!50] (2,1) rectangle (3,2);
\draw[fill=cyan!50] (3,1) rectangle (4,2);
\draw[fill=cyan!50] (4,1) rectangle (5,2);
\draw[fill=cyan!50] (5,1) rectangle (6,2);
\draw[fill=gray] (6,1) rectangle (7,2);
\draw[fill=gray] (0,2) rectangle (1,3);
\draw[fill=cyan!50] (1,2) rectangle (2,3);
\draw[fill=gray] (2,2) rectangle (3,3);
\draw[fill=gray] (4,2) rectangle (5,3);
\draw[fill=cyan!50] (5,2) rectangle (6,3);
\draw[fill=gray] (6,2) rectangle (7,3);
\draw[fill=gray] (0,3) rectangle (1,4);
\draw[fill=cyan!50] (1,3) rectangle (2,4);
\draw[fill=gray] (2,3) rectangle (3,4);
\draw[fill=cyan!50] (3,3) rectangle (4,4);
\draw[fill=cyan!50] (4,3) rectangle (5,4);
\draw[fill=cyan!50] (5,3) rectangle (6,4);
\draw[fill=gray] (6,3) rectangle (7,4);
\draw[fill=gray] (0,4) rectangle (1,5);
\draw[fill=cyan!50] (1,4) rectangle (2,5);
\draw[fill=gray] (2,4) rectangle (3,5);
\draw[fill=gray] (3,4) rectangle (4,5);
\draw[fill=gray] (4,4) rectangle (5,5);
\draw[fill=cyan!50] (5,4) rectangle (6,5);
\draw[fill=gray] (6,4) rectangle (7,5);
\draw[fill=gray] (0,5) rectangle (1,6);
\draw[fill=cyan!50] (1,5) rectangle (2,6);
\draw[fill=cyan!50] (2,5) rectangle (3,6);
\draw[fill=cyan!50] (3,5) rectangle (4,6);
\draw[fill=gray] (4,5) rectangle (5,6);
\draw[fill=cyan!50] (5,5) rectangle (6,6);
\draw[fill=gray] (6,5) rectangle (7,6);
\draw[fill=gray] (0,6) rectangle (1,7);
\draw[fill=cyan!50] (1,6) rectangle (2,7);
\draw[fill=cyan!50] (2,6) rectangle (3,7);
\draw[fill=cyan!50] (3,6) rectangle (4,7);
\draw[fill=gray] (4,6) rectangle (5,7);
\draw[fill=cyan!50] (5,6) rectangle (6,7);
\draw[fill=gray] (6,6) rectangle (7,7);
\draw[fill=gray] (0,7) rectangle (1,8);
\draw[fill=cyan!50] (1,7) rectangle (2,8);
\draw[fill=cyan!50] (2,7) rectangle (3,8);
\draw[fill=cyan!50] (3,7) rectangle (4,8);
\draw[fill=gray] (4,7) rectangle (5,8);
\draw[fill=cyan!50] (5,7) rectangle (6,8);
\draw[fill=gray] (6,7) rectangle (7,8);
\draw[fill=gray] (0,8) rectangle (1,9);
\draw[fill=cyan!50] (1,8) rectangle (2,9);
\draw[fill=cyan!50] (2,8) rectangle (3,9);
\draw[fill=cyan!50] (3,8) rectangle (4,9);
\draw[fill=gray] (4,8) rectangle (5,9);
\draw[fill=cyan!50] (5,8) rectangle (6,9);
\draw[fill=gray] (6,8) rectangle (7,9);
\draw[fill=gray] (0,9) rectangle (1,10);
\draw[fill=cyan!50] (1,9) rectangle (2,10);
\draw[fill=cyan!50] (2,9) rectangle (3,10);
\draw[fill=cyan!50] (3,9) rectangle (4,10);
\draw[fill=gray] (4,9) rectangle (5,10);
\draw[fill=cyan!50] (5,9) rectangle (6,10);
\draw[fill=gray] (6,9) rectangle (7,10);
\draw[fill=gray] (0,10) rectangle (1,11);
\draw[fill=cyan!50] (1,10) rectangle (2,11);
\draw[fill=gray] (2,10) rectangle (3,11);
\draw[fill=gray] (4,10) rectangle (5,11);
\draw[fill=cyan!50] (5,10) rectangle (6,11);
\draw[fill=gray] (6,10) rectangle (7,11);
\draw[fill=gray] (0,11) rectangle (1,12);
\draw[fill=cyan!50] (1,11) rectangle (2,12);
\draw[fill=cyan!50] (2,11) rectangle (3,12);
\draw[fill=cyan!50] (3,11) rectangle (4,12);
\draw[fill=cyan!50] (4,11) rectangle (5,12);
\draw[fill=cyan!50] (5,11) rectangle (6,12);
\draw[fill=gray] (6,11) rectangle (7,12);
\draw[fill=gray] (0,12) rectangle (1,13);
\draw[fill=gray] (1,12) rectangle (2,13);
\draw[fill=gray] (2,12) rectangle (3,13);
\draw[fill=gray] (3,12) rectangle (4,13);
\draw[fill=gray] (4,12) rectangle (5,13);
\draw[fill=gray] (5,12) rectangle (6,13);
\draw[fill=gray] (6,12) rectangle (7,13);
\node [align=center] at(-0.500000,0.500000) {0};
\node [align=center] at(-0.500000,1.500000) {1};
\node [align=center] at(-0.500000,2.500000) {2};
\node [align=center] at(-0.500000,3.500000) {3};
\node [align=center] at(-0.500000,4.500000) {4};
\node [align=center] at(-0.500000,5.500000) {5};
\node [align=center] at(-0.500000,6.500000) {6};
\node [align=center] at(-0.500000,7.500000) {7};
\node [align=center] at(-0.500000,8.500000) {8};
\node [align=center] at(-0.500000,9.500000) {9};
\node [align=center] at(-0.500000,10.500000) {10};
\node [align=center] at(-0.500000,11.500000) {11};
\node [align=center] at(-0.500000,12.500000) {12};
\node [align=center] at(0.500000,-0.500000) {A};
\node [align=center] at(1.500000,-0.500000) {B};
\node [align=center] at(2.500000,-0.500000) {C};
\node [align=center] at(3.500000,-0.500000) {D};
\node [align=center] at(4.500000,-0.500000) {E};
\node [align=center] at(5.500000,-0.500000) {F};
\node [align=center] at(6.500000,-0.500000) {G};
\draw[line width=1mm] (5.500000,11.500000) circle (0.300000);
\draw[-Latex] (5.900000,11.500000)--(5.100000,11.500000);
\draw[-Latex] (4.900000,11.500000)--(4.100000,11.500000);
\draw[-Latex] (3.500000,11.900000)--(3.500000,11.100000);
\draw[-Latex] (3.500000,10.900000)--(3.500000,10.100000);
\draw[-Latex] (3.500000,9.900000)--(3.500000,9.100000);
\draw[-Latex] (3.900000,8.500000)--(3.100000,8.500000);
\draw[-Latex] (2.900000,8.500000)--(2.100000,8.500000);
\draw[-Latex] (1.500000,8.900000)--(1.500000,8.100000);
\draw[-Latex] (1.500000,7.900000)--(1.500000,7.100000);
\draw[-Latex] (1.500000,6.900000)--(1.500000,6.100000);
\draw[-Latex] (1.500000,5.100000)--(1.500000,5.900000);
\draw[-Latex] (1.500000,6.900000)--(1.500000,6.100000);
\draw[-Latex] (1.500000,5.900000)--(1.500000,5.100000);
\draw[-Latex] (1.500000,4.900000)--(1.500000,4.100000);
\draw[-Latex] (1.500000,3.900000)--(1.500000,3.100000);
\draw[-Latex] (1.500000,2.100000)--(1.500000,2.900000);
\draw[-Latex] (1.500000,3.900000)--(1.500000,3.100000);
\draw[-Latex] (1.500000,2.900000)--(1.500000,2.100000);
\end{tikzpicture}
    }

    \medskip

\noindent\fbox{\begin{minipage}{\dimexpr\mazewidthC-2\fboxsep-2\fboxrule\relax}
      {\scriptsize
        The evolution of the belief over the target variable (\ie, the state) is too complex to display.
      }
      \end{minipage}
    }

    \medskip

  }

  \caption{PO-OAMDP trajectories and corresponding belief evolutions for {\em predictability} tasks
\label{fig|predictability|trajectories}
  }
\end{figure}
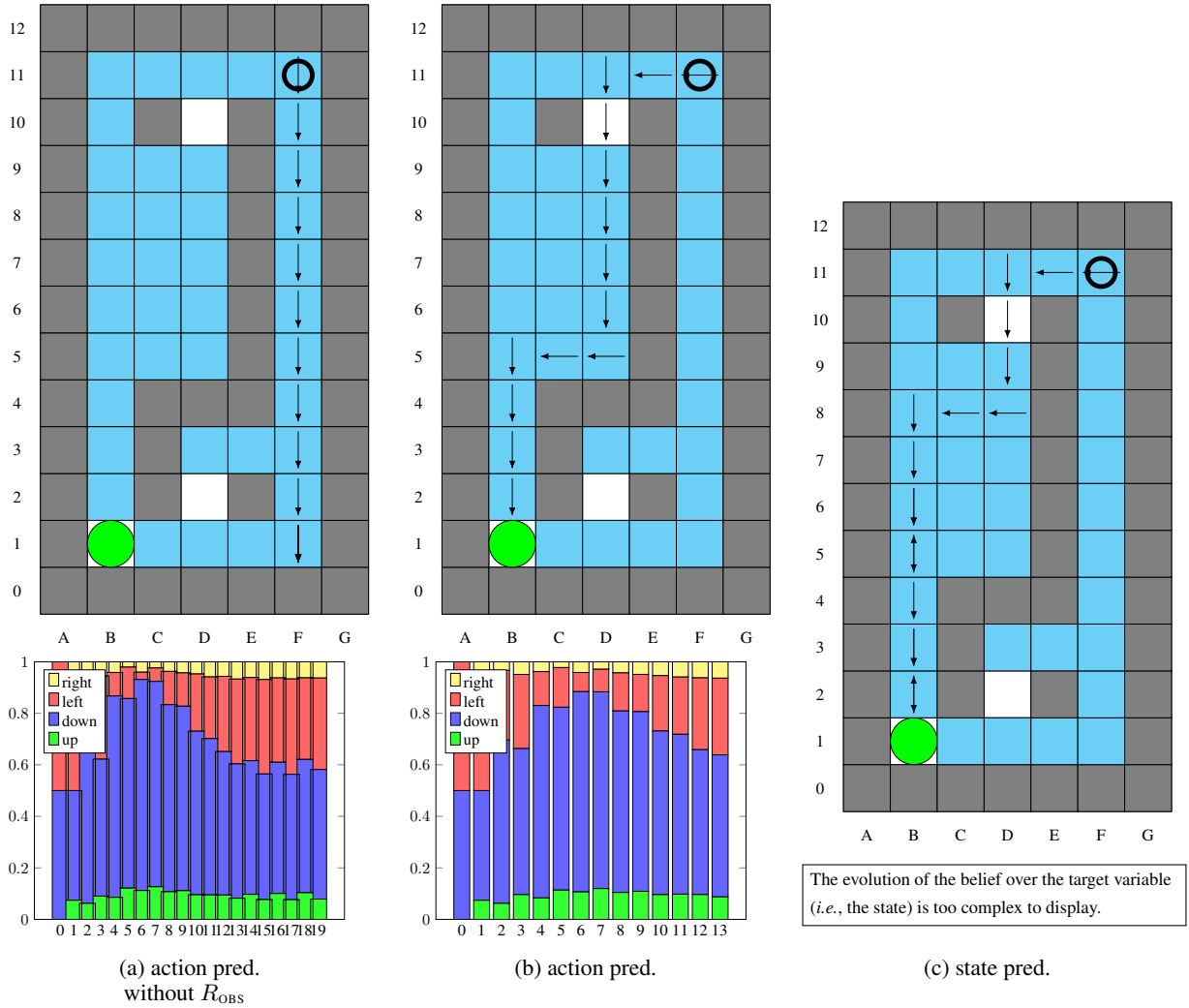

\subsection{Computation Time}
\label{sec|computationTime}

\Cshref{fig|bounds|legibility} shows the evolution of the upper and lower bound during HSVI's convergence on the legibility task, with typical monotonic step-wise behaviors on both sides.
\Cshref{fig|ErrorGap} shows the evolution of the error gap in the four problems involving deterministic observations for the two proposed bound initializations.

\paragraph{Criteria and Grid}
A first comment is that the convergence is much faster for the predictability criteria.
This is likely due
\begin{enumerate*}
\item to the lack of open spaces in the corresponding grid (so that less trajectories need to be considered) and
\item to the reward functions possibly better guiding the decisions.
\end{enumerate*}

\paragraph{Initializations}
Overall, the combined initialization has a better anytime behavior than the naive one.
This is true in particular in the most complex problems (legibility and explicability), even if the curves sometimes cross each other.

\begin{figure}
	\centering

	\includegraphics[width=0.7\columnwidth]{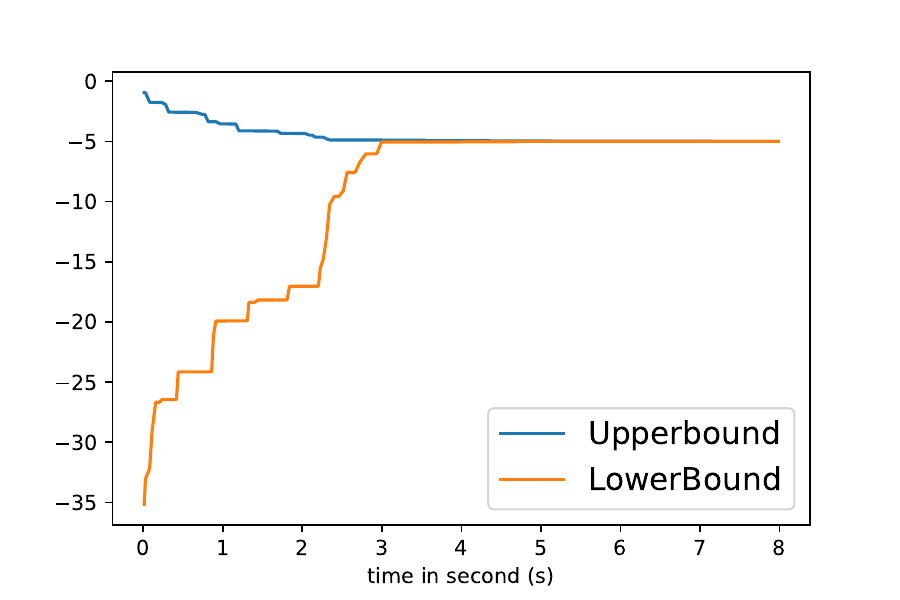}

	\caption{
          Evolution of the upper and lower bounds for legibility on the first maze using the combined initialization
          \label{fig|bounds|legibility}
	}
\end{figure}

\def\mycolwidthB{0.47\columnwidth}
\def\graphwidth{0.47\columnwidth}

\begin{figure}
  \centering

  \subcaptionbox{ Legibility
  }[\mycolwidthB]{
    \centering
    \includegraphics[width=\graphwidth]{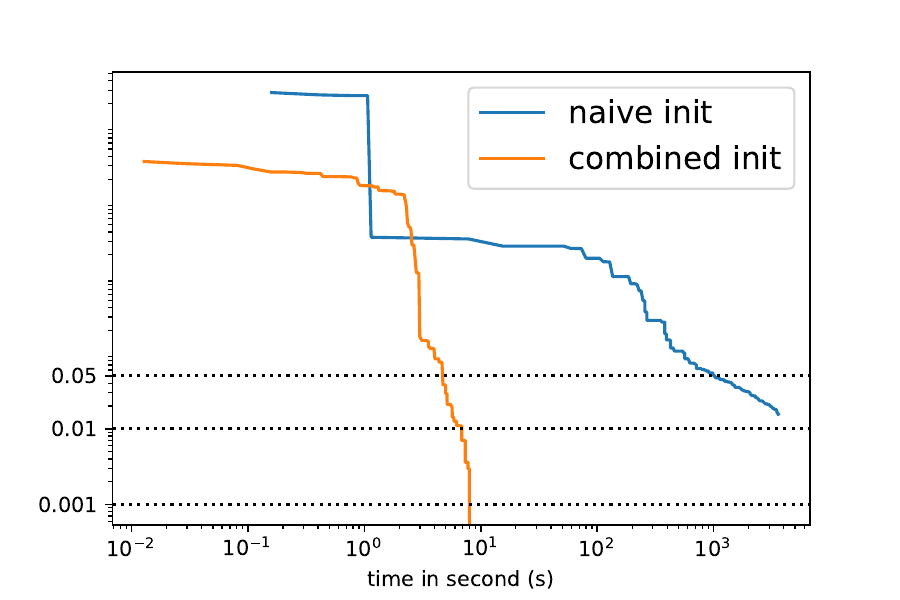}
  }
  \hfill
  \subcaptionbox{
    Explicability
  }[\mycolwidthB]{
    \includegraphics[width=\graphwidth]{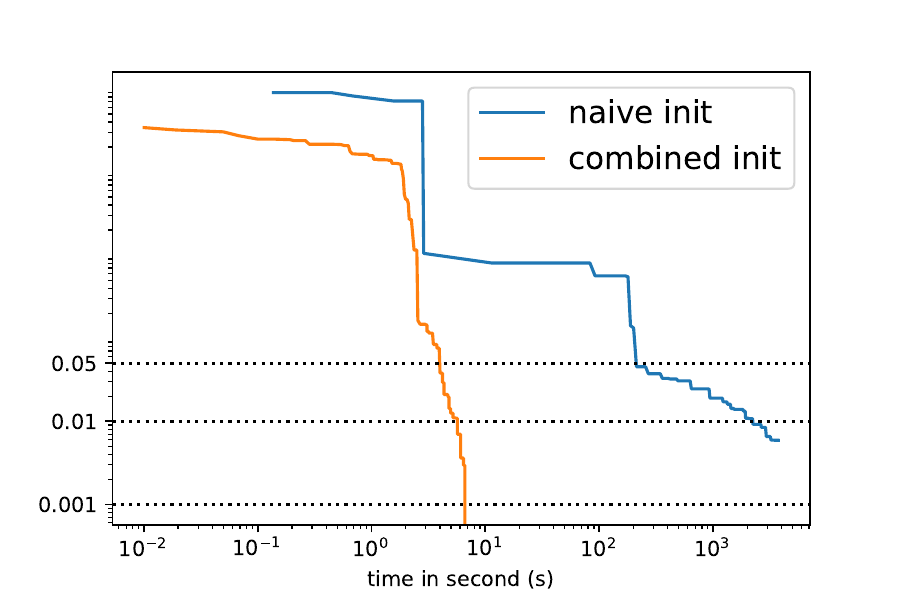}
  }

  \subcaptionbox{
    Action predictability
  }[\mycolwidthB]{
    \includegraphics[width=\graphwidth]{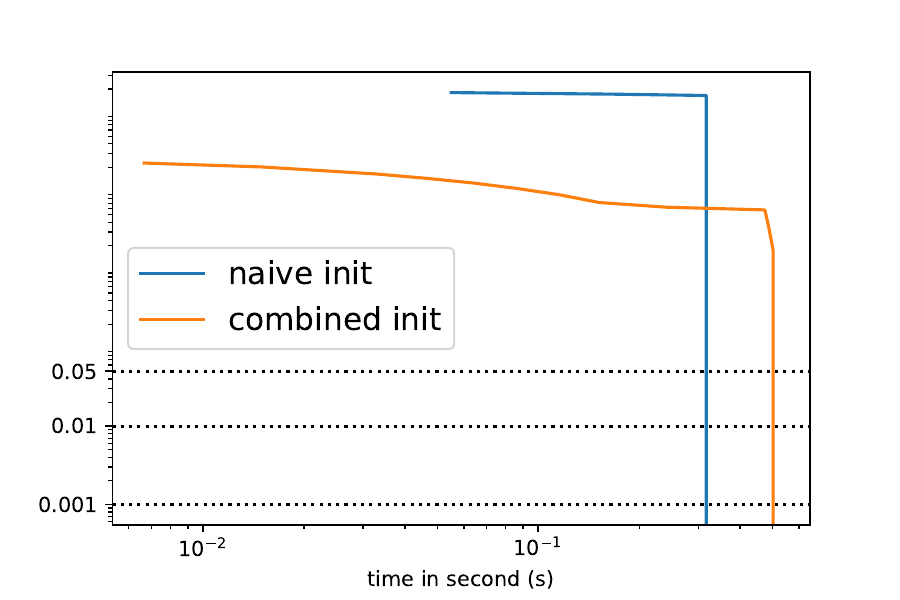}
  }
  \hfill
  \subcaptionbox{
    State predictability
  }[\mycolwidthB]{
    \includegraphics[width=\graphwidth]{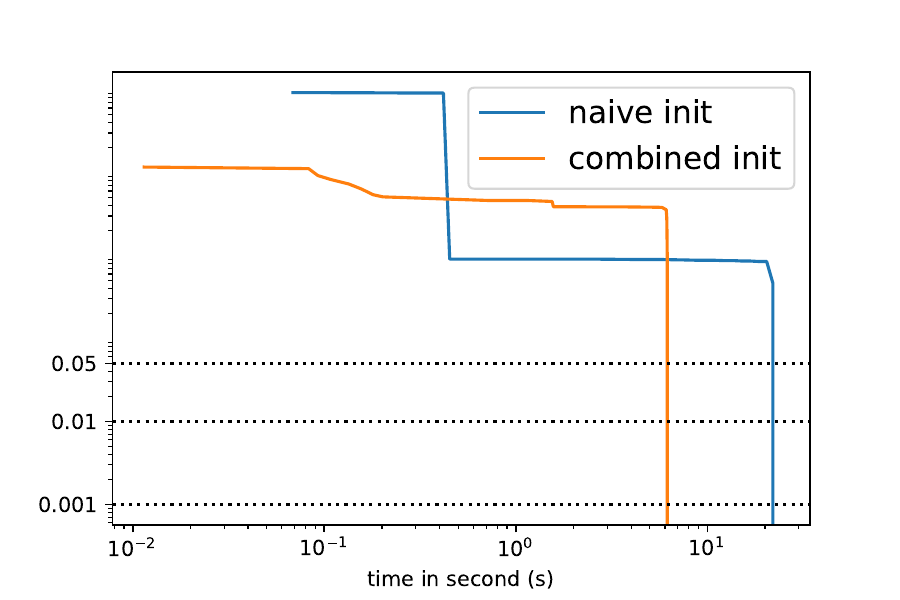}
  }

  \caption{Evolution of the error gap as a function of time (in seconds) for different initializations (log scales)
    \label{fig|ErrorGap}
  }
\end{figure}

\def\mpwidth{.45\columnwidth}
\def\aboxwidth{1.0\textwidth}

\section{Conclusion}
\label{sec|conclusion}

We have introduced the novel framework of observer-aware MDPs under partial observability (PO-OAMDPs), which allows addressing (among other things) legibility, explicability and predictability problems when the observer has only a limited perception of the agent and its environment.
This framework more than just generalizes \citeauthor{pmlr-v161-miura21a}'s OAMDPs (with similar complexity results) as the target variable is now transition-dependent, and can thus be dynamic, contrary to the original agent type.
This enables formalizing a wider range of problems (\eg, action and state predictability as defined by \citet{LepLemThoBuf-arxiv24}, but see also \Cshref{app|moreExampleScenarios}).

Assuming a BST model of the observer, we show how to update her state belief (known to the agent), and derive her belief over the target variable.
This leads to turning a PO-OAMDP into an equivalent state+belief MDP whose virtual state is a (state, state-belief) pair,
allowing to adapt \citeauthor{SmiSim-uai04}'s HSVI algorithm \citep{SmiSim-uai04} with dedicated upper- and lower-bound initializations.

Experiments show the benefits of these initializations and illustrate the PO-OAMDP framework by demonstrating resulting non-trivial behaviors with several criteria (legibility, explicability and predictability), with a significant benefit compared to default policies.
Among other things, they show that these criteria do not necessarily induce valid SSPs, an issue that can be provably be alleviated by simply adding another cost function.

Future work includes further exploring the possibilities offered by PO-OAMDPs, and
improving solution techniques.
In particular, we aim at exploiting the continuity of $V^*$ in belief space, even if it may not be convex and may exhibit local discontinuities \citep{miuBufZil-uai24}.

\bibliographystyle{plainnat}

\renewcommand{\labelT}[1]{} 

\appendix

\section{OAMDPs expressed as PO-OAMDPs}
\label{app|generalization}

This appendix demonstrates that, assuming the BST belief update is used, any OAMDP can be turned into an equivalent PO-OAMDP (cf. \Cref{sec|contribution}).

\propOAMDPequivPOOAMDP

\begin{proof}
  Let $\mathcal{M} \equiv \langle \cS, s_0, \cA, T, \gamma, \cS_f, \Type, B, \Rag \rangle$ be an OAMDP and, for each $\type$ in the set of possible types $\Type$, let $\mathcal{M}^\type \equiv \langle \cS$, $s_0$, $\cA$, $T^\type$, $\gamma$, $\cS_f^\type$, $\Robs^\type \rangle$ be the corresponding MDP.

  Let us now introduce a new type $\tilde\type$ and the MDP $\mathcal{M}^{\tilde\type} \equiv \langle \cS$, $s_0$, $\cA$, $T^{\tilde\type}\eqdef T$, $\gamma$, $\cS_f^{\tilde\type}$ $\eqdef \cS_f$, $\Robs^{\tilde\type} \rangle$, where $\Robs^{\tilde\type}$ is the reward function that returns $-1$ at each time step until a terminal state is reached (to ensure that we have a valid SSP if needed).
  We can now define the PO-OAMDP $\mathcal{M}' \equiv \langle \cS', s'_0, \cA', T', \gamma, \cS'_f, \Robs', (\Target\equiv) \Type', \Omega, O', B', \Rag', \targetOf \rangle$ where:
  \begin{align*}
    \Type'
    &\eqdef \Type \cup \{\tilde\type\}, \\
    \cS'
    & \eqdef \cS \times \Type', \\
    s'_0
    & \eqdef (s_0, \tilde\type), \\
    \cA'
    & \eqdef \cA, \\
    T'((s,\type),a,(s',\type'))
    & \eqdef
    \mathbb{1}_{\type=\type'} \cdot T^\type(s,a,s'), \\
    \cS'_f
    & \eqdef
    \{ (s,\type) | s \in \cS_f^\type \}, \\
    \Robs'((s,\type),a,(s',\type'))
    & \eqdef
    \mathbb{1}_{\type=\type'} \cdot \Robs^\type(s,a,s'), \\
    \Omega
    & \eqdef
    \cS,
    \\
    O'(a,s',o)
    & \eqdef
    \mathbb{1}_{o=s'},
    \\
    B'((s,\type)| - )
    & \eqdef
    \begin{cases}
      B((s,\type)| - ) \cdot \mathbb{1}_{s=s_0} & \text{if } \type \in \Type, \text{ and} \\
      0 & \text{if } \type = \tilde\type,
    \end{cases}
    \intertext{(with beliefs at time steps beyond $t=0$ computed by Bayesian belief updates),}
    \Rag'(s,\beta,a,s',\beta')
    & \eqdef
    \Rag(s,a,\beta_{-\tilde\type}),
    \intertext{where $\beta_{-\tilde\type}$ is the belief vector $\beta$ deprived of the $\tilde\type$ component, and}
    \targetOf((s,\type),a,(s',\type'))
    & \eqdef
    \type.
  \end{align*}
  Note that the ``fake type'' $\tilde\type$ only serves to ensure that the actual dynamics (transition function and set of terminal states) exist, as it could be that none of the ``true types'' is attached to them.

  Then, solving the underlying MDP with reward $\Robs'$ (either for all states, or for all states reachable from the states in $b_0$) is equivalent to solving each MDP $\mathcal{M}^\type$, and the softmax policy $\piobs'$ is equivalent to the softmax policies $\piobs^\type$.
Thus, the BST update is the same in both cases.

  As can be noted, the belief over targets/types will always have value 0 for $\tilde\type$, so that $\beta_{-\tilde\type}$ will correspond to the belief over types for the OAMDP, and, as a consequence, the agent reward is equivalent in both settings.
Then because the initial state of the PO-OAMDP is $(s_0,\tilde\type)$, and because the belief updates are equivalent, the dynamics of the PO-OAMDP are equivalent to those of the original OAMDP, so that solving the one is equivalent to solving the other.
\end{proof}

Note that the fake type $\tilde\type$ is, in a sense, the actual type of the (observer-aware) agent, which is ignored by the observer because the observer does not model the agent as optimizing $\Rag'$.

\section{Complementary Experimental Results}
\label{app|XP}

The following sections present some complementary experimental results.
\Cref{subsc|MDP Policies} shows an illustration of the observer's softmax MDP policies $\piobs$ used in legibility and explicability tasks.
The following subsections present the results obtained with smaller versions of the grids used in the main experiments (\Cref{sec|XPs}).
These subsections illustrate that back and forth movements observed in previous tasks (\Cshref{fig|legibility|trajectories})  are not always necessary and are the consequences of uncertainties on the target variable and induced costs.

\subsection{Observer's Softmax MDP Policies}
\label{subsc|MDP Policies}

\begin{figure}
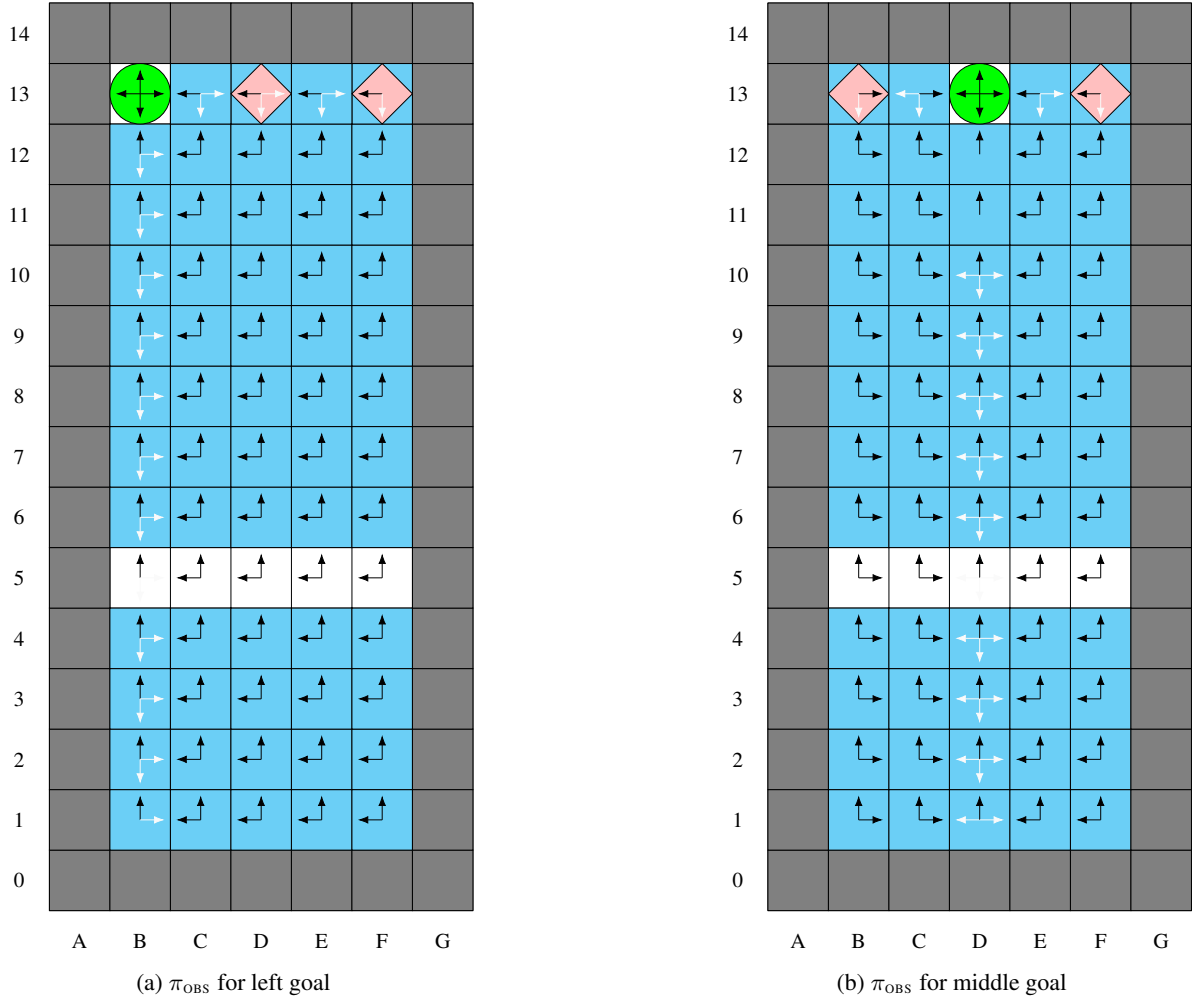

  \subcaptionbox{
    $\piobs$ for left goal
    \label{fig|MDPPolicies|left}
  }{
    \centering
    \adjustbox{width=\mazewidth}{
      \input{XPHSVI/Legibility/7_MDP_policy.tikz}
    }
  }
  \hfill
  \subcaptionbox{
    $\piobs$ for middle goal
    \label{fig|MDPPolicies|middle}
  }{
    \centering
    \adjustbox{width=\mazewidth}{
      \input{XPHSVI/Legibility/10_MDP_policy.tikz}
    }
  }
  \caption{Observer's Softmax MDP policies of the legibility and explicability task  for different actual goals
    \label{fig|MDPPolicies}
  }
\end{figure}

\Cshref{fig|MDPPolicies|left,fig|MDPPolicies|middle} illustrate the $\piobs$ policies that are used for both legibility and explicability tasks (because the task considered by the observer is the same in both scenarios).
In these figures, the higher the probability to select an action, the darker the corresponding arrow, and, when the probability is below a threshold of $0.1$, the corresponding arrow is not shown.

This highlights that the observer models the agent as following a stochastic policy, and it must be noted that, due to negative rewards when hitting a wall, probabilities to go away from the wall are larger than the aforementioned probability threshold.

\subsection{Legibility}
\label{app|XP|legibility}

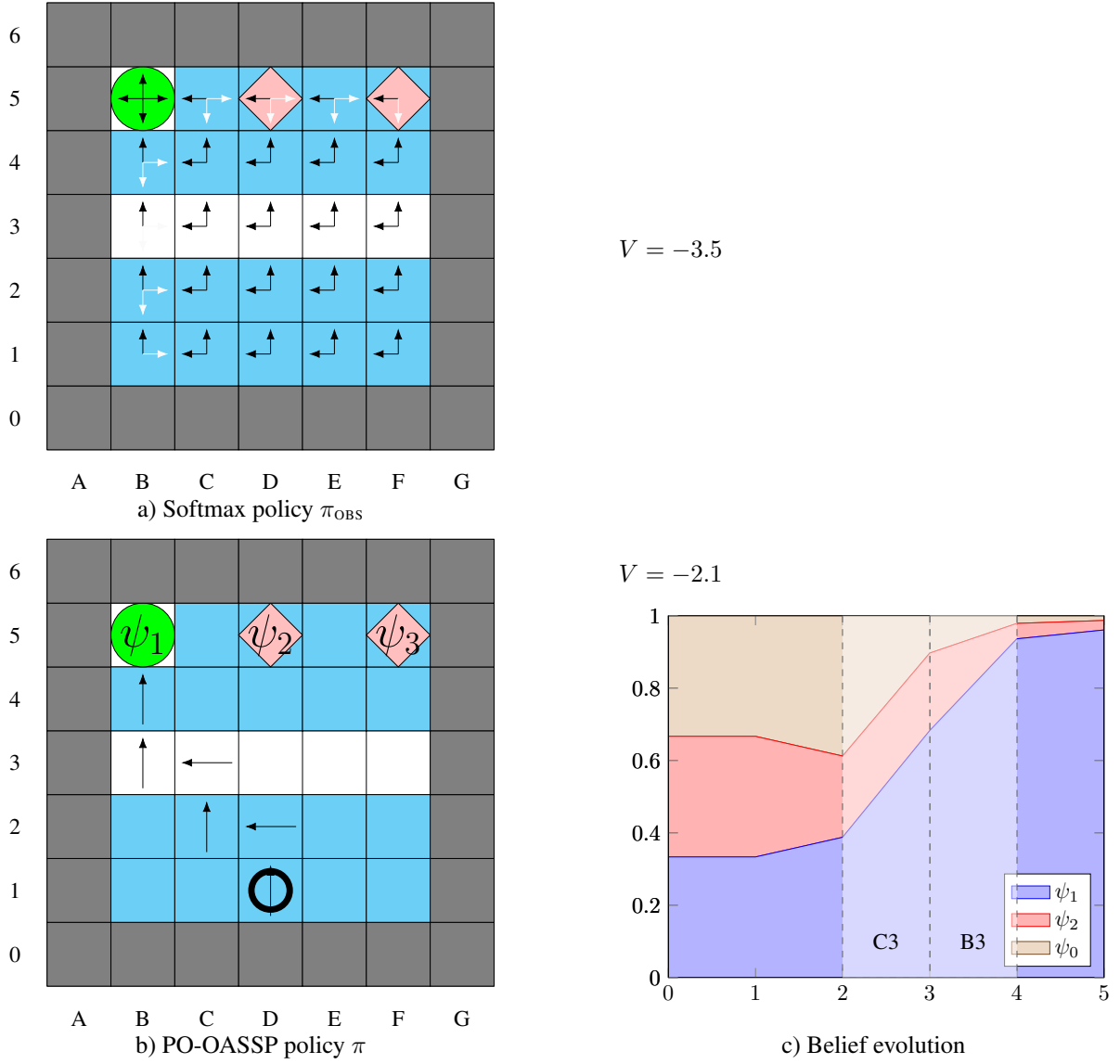
\begin{figure}
	\begin{minipage}{\mpwidth}
		\centering
		\adjustbox{width=\aboxwidth}{
			\begin{tikzpicture}
\draw (0,0) grid (7,7);
\draw[fill=gray] (0,0) rectangle (1,1);
\draw[fill=gray] (1,0) rectangle (2,1);
\draw[fill=gray] (2,0) rectangle (3,1);
\draw[fill=gray] (3,0) rectangle (4,1);
\draw[fill=gray] (4,0) rectangle (5,1);
\draw[fill=gray] (5,0) rectangle (6,1);
\draw[fill=gray] (6,0) rectangle (7,1);
\draw[fill=gray] (0,1) rectangle (1,2);
\draw[fill=cyan!50] (1,1) rectangle (2,2);
\draw[fill=cyan!50] (2,1) rectangle (3,2);
\draw[fill=cyan!50] (3,1) rectangle (4,2);
\draw[fill=cyan!50] (4,1) rectangle (5,2);
\draw[fill=cyan!50] (5,1) rectangle (6,2);
\draw[fill=gray] (6,1) rectangle (7,2);
\draw[fill=gray] (0,2) rectangle (1,3);
\draw[fill=cyan!50] (1,2) rectangle (2,3);
\draw[fill=cyan!50] (2,2) rectangle (3,3);
\draw[fill=cyan!50] (3,2) rectangle (4,3);
\draw[fill=cyan!50] (4,2) rectangle (5,3);
\draw[fill=cyan!50] (5,2) rectangle (6,3);
\draw[fill=gray] (6,2) rectangle (7,3);
\draw[fill=gray] (0,3) rectangle (1,4);
\draw[fill=gray] (6,3) rectangle (7,4);
\draw[fill=gray] (0,4) rectangle (1,5);
\draw[fill=cyan!50] (1,4) rectangle (2,5);
\draw[fill=cyan!50] (2,4) rectangle (3,5);
\draw[fill=cyan!50] (3,4) rectangle (4,5);
\draw[fill=cyan!50] (4,4) rectangle (5,5);
\draw[fill=cyan!50] (5,4) rectangle (6,5);
\draw[fill=cyan!50] (5,5) rectangle (6,6);
\draw[fill=cyan!50] (3,5) rectangle (4,6);
\draw[fill=gray] (6,4) rectangle (7,5);
\draw[fill=gray] (0,5) rectangle (1,6);
\draw[fill=green] (1.500000,5.500000) circle (0.500000);
\draw[fill=cyan!50] (2,5) rectangle (3,6);
\begin{scope}[shift={(3.500000 ,5)}] \draw[fill=pink,rotate=45] rectangle (0.707, 0.707);
\end{scope}
\draw[fill=cyan!50] (4,5) rectangle (5,6);
\begin{scope}[shift={(5.500000 ,5)}] \draw[fill=pink,rotate=45] rectangle (0.707, 0.707);
\end{scope}
\draw[fill=gray] (6,5) rectangle (7,6);
\draw[fill=gray] (0,6) rectangle (1,7);
\draw[fill=gray] (1,6) rectangle (2,7);
\draw[fill=gray] (2,6) rectangle (3,7);
\draw[fill=gray] (3,6) rectangle (4,7);
\draw[fill=gray] (4,6) rectangle (5,7);
\draw[fill=gray] (5,6) rectangle (6,7);
\draw[fill=gray] (6,6) rectangle (7,7);
\node [align=center] at(-0.500000,0.500000) {0};
\node [align=center] at(-0.500000,1.500000) {1};
\node [align=center] at(-0.500000,2.500000) {2};
\node [align=center] at(-0.500000,3.500000) {3};
\node [align=center] at(-0.500000,4.500000) {4};
\node [align=center] at(-0.500000,5.500000) {5};
\node [align=center] at(-0.500000,6.500000) {6};
\node [align=center] at(0.500000,-0.500000) {A};
\node [align=center] at(1.500000,-0.500000) {B};
\node [align=center] at(2.500000,-0.500000) {C};
\node [align=center] at(3.500000,-0.500000) {D};
\node [align=center] at(4.500000,-0.500000) {E};
\node [align=center] at(5.500000,-0.500000) {F};
\node [align=center] at(6.500000,-0.500000) {G};
\draw[-Latex,draw=black!100] (1.500000,5.500000)--(1.100000,5.500000);
\draw[-Latex,draw=black!100] (1.500000,5.500000)--(1.500000,5.900000);
\draw[-Latex,draw=black!100] (1.500000,5.500000)--(1.900000,5.500000);
\draw[-Latex,draw=black!100] (1.500000,5.500000)--(1.500000,5.100000);
\draw[-Latex,draw=black!100] (2.500000,5.500000)--(2.100000,5.500000);
\draw[-Latex,draw=black!1] (2.500000,5.500000)--(2.900000,5.500000);
\draw[-Latex,draw=black!1] (2.500000,5.500000)--(2.500000,5.100000);
\draw[-Latex,draw=black!100] (3.500000,5.500000)--(3.100000,5.500000);
\draw[-Latex,draw=black!2] (3.500000,5.500000)--(3.900000,5.500000);
\draw[-Latex,draw=black!2] (3.500000,5.500000)--(3.500000,5.100000);
\draw[-Latex,draw=black!100] (4.500000,5.500000)--(4.100000,5.500000);
\draw[-Latex,draw=black!2] (4.500000,5.500000)--(4.900000,5.500000);
\draw[-Latex,draw=black!2] (4.500000,5.500000)--(4.500000,5.100000);
\draw[-Latex,draw=black!100] (5.500000,5.500000)--(5.100000,5.500000);
\draw[-Latex,draw=black!2] (5.500000,5.500000)--(5.500000,5.100000);
\draw[-Latex,draw=black!100] (1.500000,4.500000)--(1.500000,4.900000);
\draw[-Latex,draw=black!1] (1.500000,4.500000)--(1.900000,4.500000);
\draw[-Latex,draw=black!1] (1.500000,4.500000)--(1.500000,4.100000);
\draw[-Latex,draw=black!100] (2.500000,4.500000)--(2.100000,4.500000);
\draw[-Latex,draw=black!100] (2.500000,4.500000)--(2.500000,4.900000);
\draw[-Latex,draw=black!100] (3.500000,4.500000)--(3.100000,4.500000);
\draw[-Latex,draw=black!100] (3.500000,4.500000)--(3.500000,4.900000);
\draw[-Latex,draw=black!100] (4.500000,4.500000)--(4.100000,4.500000);
\draw[-Latex,draw=black!100] (4.500000,4.500000)--(4.500000,4.900000);
\draw[-Latex,draw=black!100] (5.500000,4.500000)--(5.100000,4.500000);
\draw[-Latex,draw=black!100] (5.500000,4.500000)--(5.500000,4.900000);
\draw[-Latex,draw=black!100] (1.500000,3.500000)--(1.500000,3.900000);
\draw[-Latex,draw=black!2] (1.500000,3.500000)--(1.900000,3.500000);
\draw[-Latex,draw=black!2] (1.500000,3.500000)--(1.500000,3.100000);
\draw[-Latex,draw=black!100] (2.500000,3.500000)--(2.100000,3.500000);
\draw[-Latex,draw=black!100] (2.500000,3.500000)--(2.500000,3.900000);
\draw[-Latex,draw=black!100] (3.500000,3.500000)--(3.100000,3.500000);
\draw[-Latex,draw=black!100] (3.500000,3.500000)--(3.500000,3.900000);
\draw[-Latex,draw=black!100] (4.500000,3.500000)--(4.100000,3.500000);
\draw[-Latex,draw=black!100] (4.500000,3.500000)--(4.500000,3.900000);
\draw[-Latex,draw=black!100] (5.500000,3.500000)--(5.100000,3.500000);
\draw[-Latex,draw=black!100] (5.500000,3.500000)--(5.500000,3.900000);
\draw[-Latex,draw=black!100] (1.500000,2.500000)--(1.500000,2.900000);
\draw[-Latex,draw=black!2] (1.500000,2.500000)--(1.900000,2.500000);
\draw[-Latex,draw=black!2] (1.500000,2.500000)--(1.500000,2.100000);
\draw[-Latex,draw=black!100] (2.500000,2.500000)--(2.100000,2.500000);
\draw[-Latex,draw=black!100] (2.500000,2.500000)--(2.500000,2.900000);
\draw[-Latex,draw=black!100] (3.500000,2.500000)--(3.100000,2.500000);
\draw[-Latex,draw=black!100] (3.500000,2.500000)--(3.500000,2.900000);
\draw[-Latex,draw=black!100] (4.500000,2.500000)--(4.100000,2.500000);
\draw[-Latex,draw=black!100] (4.500000,2.500000)--(4.500000,2.900000);
\draw[-Latex,draw=black!100] (5.500000,2.500000)--(5.100000,2.500000);
\draw[-Latex,draw=black!100] (5.500000,2.500000)--(5.500000,2.900000);
\draw[-Latex,draw=black!100] (1.500000,1.500000)--(1.500000,1.900000);
\draw[-Latex,draw=black!2] (1.500000,1.500000)--(1.900000,1.500000);
\draw[-Latex,draw=black!100] (2.500000,1.500000)--(2.100000,1.500000);
\draw[-Latex,draw=black!100] (2.500000,1.500000)--(2.500000,1.900000);
\draw[-Latex,draw=black!100] (3.500000,1.500000)--(3.100000,1.500000);
\draw[-Latex,draw=black!100] (3.500000,1.500000)--(3.500000,1.900000);
\draw[-Latex,draw=black!100] (4.500000,1.500000)--(4.100000,1.500000);
\draw[-Latex,draw=black!100] (4.500000,1.500000)--(4.500000,1.900000);
\draw[-Latex,draw=black!100] (5.500000,1.500000)--(5.100000,1.500000);
\draw[-Latex,draw=black!100] (5.500000,1.500000)--(5.500000,1.900000);
\end{tikzpicture}
		}
	\end{minipage}
	\hfill
	\begin{minipage}{\mpwidth}
		$V= -3.5$
	\end{minipage}

	\begin{minipage}{\mpwidth}
		\centering
		a) Softmax policy $\piobs$
	\end{minipage}

	\medskip

	\begin{minipage}{\mpwidth}
		\centering
		\adjustbox{width=\aboxwidth}{
			\begin{tikzpicture}
\draw (0,0) grid (7,7);
\draw[fill=gray] (0,0) rectangle (1,1);
\draw[fill=gray] (1,0) rectangle (2,1);
\draw[fill=gray] (2,0) rectangle (3,1);
\draw[fill=gray] (3,0) rectangle (4,1);
\draw[fill=gray] (4,0) rectangle (5,1);
\draw[fill=gray] (5,0) rectangle (6,1);
\draw[fill=gray] (6,0) rectangle (7,1);
\draw[fill=gray] (0,1) rectangle (1,2);
\draw[fill=cyan!50] (1,1) rectangle (2,2);
\draw[fill=cyan!50] (2,1) rectangle (3,2);
\draw[fill=cyan!50] (3,1) rectangle (4,2);
\draw[fill=cyan!50] (4,1) rectangle (5,2);
\draw[fill=cyan!50] (5,1) rectangle (6,2);
\draw[fill=gray] (6,1) rectangle (7,2);
\draw[fill=gray] (0,2) rectangle (1,3);
\draw[fill=cyan!50] (1,2) rectangle (2,3);
\draw[fill=cyan!50] (2,2) rectangle (3,3);
\draw[fill=cyan!50] (3,2) rectangle (4,3);
\draw[fill=cyan!50] (4,2) rectangle (5,3);
\draw[fill=cyan!50] (5,2) rectangle (6,3);
\draw[fill=gray] (6,2) rectangle (7,3);
\draw[fill=gray] (0,3) rectangle (1,4);
\draw[fill=gray] (6,3) rectangle (7,4);
\draw[fill=gray] (0,4) rectangle (1,5);
\draw[fill=cyan!50] (1,4) rectangle (2,5);
\draw[fill=cyan!50] (2,4) rectangle (3,5);
\draw[fill=cyan!50] (3,4) rectangle (4,5);
\draw[fill=cyan!50] (4,4) rectangle (5,5);
\draw[fill=cyan!50] (5,4) rectangle (6,5);
\draw[fill=cyan!50] (5,5) rectangle (6,6);
\draw[fill=cyan!50] (3,5) rectangle (4,6);
\draw[fill=gray] (6,4) rectangle (7,5);
\draw[fill=gray] (0,5) rectangle (1,6);
\draw[fill=green] (1.500000,5.500000) circle (0.500000);
\draw[fill=cyan!50] (2,5) rectangle (3,6);
\begin{scope}[shift={(3.500000 ,5)}] \draw[fill=pink,rotate=45] rectangle (0.707, 0.707);
\end{scope}
\draw[fill=cyan!50] (4,5) rectangle (5,6);
\begin{scope}[shift={(5.500000 ,5)}] \draw[fill=pink,rotate=45] rectangle (0.707, 0.707);
\end{scope}
\draw[fill=gray] (6,5) rectangle (7,6);
\draw[fill=gray] (0,6) rectangle (1,7);
\draw[fill=gray] (1,6) rectangle (2,7);
\draw[fill=gray] (2,6) rectangle (3,7);
\draw[fill=gray] (3,6) rectangle (4,7);
\draw[fill=gray] (4,6) rectangle (5,7);
\draw[fill=gray] (5,6) rectangle (6,7);
\draw[fill=gray] (6,6) rectangle (7,7);
\node [align=center] at(-0.500000,0.500000) {0};
\node [align=center] at(-0.500000,1.500000) {1};
\node [align=center] at(-0.500000,2.500000) {2};
\node [align=center] at(-0.500000,3.500000) {3};
\node [align=center] at(-0.500000,4.500000) {4};
\node [align=center] at(-0.500000,5.500000) {5};
\node [align=center] at(-0.500000,6.500000) {6};
\node [align=center] at(0.500000,-0.500000) {A};
\node [align=center] at(1.500000,-0.500000) {B};
\node [align=center] at(2.500000,-0.500000) {C};
\node [align=center] at(3.500000,-0.500000) {D};
\node [align=center] at(4.500000,-0.500000) {E};
\node [align=center] at(5.500000,-0.500000) {F};
\node [align=center] at(6.500000,-0.500000) {G};
\draw[line width=1mm] (3.500000,1.500000) circle (0.300000);
\draw[-Latex] (3.500000,1.100000)--(3.500000,1.900000);
\draw[-Latex] (3.900000,2.500000)--(3.100000,2.500000);
\draw[-Latex] (2.500000,2.100000)--(2.500000,2.900000);
\draw[-Latex] (2.900000,3.500000)--(2.100000,3.500000);
\draw[-Latex] (1.500000,3.100000)--(1.500000,3.900000);
\draw[-Latex] (1.500000,4.100000)--(1.500000,4.900000);

\node at (1.500000,5.500000) {\huge$\psi_1$};
\node at (3.500000,5.500000) {\huge$\psi_2$};
\node at (5.500000,5.500000) {\huge$\psi_3$};
\end{tikzpicture}
		}
	\end{minipage}
	\hfill
	\begin{minipage}{\mpwidth}
		$V = -2.1$ 

		\medskip

		\centering
		\adjustbox{width=\aboxwidth}{
			\begin{tikzpicture}
\begin{axis}[
ymin=0,
ymax=1,
stack plots=y,
area style,
enlarge x limits=false,
legend pos=south east
]
\addplot coordinates
{(0,0.333333) (1,0.333333) (2,0.387745) (3,0.682604) (4,0.936174) (5,0.960612) }
\closedcycle;
\addlegendentry{$\psi_1$}\addplot coordinates
{(0,0.333333) (1,0.333333) (2,0.224510) (3,0.214173) (4,0.042597) (5,0.026267) }
\closedcycle;
\addlegendentry{$\psi_2$}\addplot coordinates
{(0,0.333333) (1,0.333333) (2,0.387745) (3,0.103223) (4,0.021229) (5,0.013121) }
\closedcycle;
\addlegendentry{$\psi_0 $}\end{axis}

\begin{axis}[
ymin=0,
ymax=1,
xmin=0,
xmax=5,
stack plots=y,
area style,
enlarge x limits=false,
legend pos=south east,
  axis x line=none,
]
\textBelief{C3}{2}{5}
\ObservationLine[dashed]{2}

\textBelief{B3}{3}{5}
\ObservationLine[dashed]{3}

\end{axis}

\end{tikzpicture}
		}
	\end{minipage}

	\begin{minipage}{\mpwidth}
		\centering
		b) PO-OASSP  policy $\pi$
	\end{minipage}
	\hfill
	\begin{minipage}{\mpwidth}
		\centering
		c) Belief evolution
	\end{minipage}

		\caption{Results for legibility
			\label{RLegibility}}
	\end{figure}

When addressing legibility task in a smaller grid (\Cshref{RLegibility}), the PO-OAMDP agent simply reaches cell $(C,3)$, stays visible by moving left to $(B,3)$, then goes to the actual-goal cell.
It must be noted that at time step $2$, the agent decides not to appear in cell $(D,3)$.
Since the observer would have expected the agent to go in cell $(D,3)$ if its actual goal had been $\target_2$, her belief in the middle goal $\target_2$ decreases just by not observing the agent in that cell.
When the agent appears in cell $(C,3)$, the observer belief in goal $\target_1$ increases but, since $\piobs$ is stochastic, she cannot be sure of the actual goal of the agent.

Then, the agent prefers to quickly reach its actual goal and a terminal state (with future cumulated rewards of $0$) than take time to reduce uncertainties as seen in previous experiment (\Cshref{fig|legibility|trajectories|left}).
This is because, in this new maze, the goal is easily reached and the agent is only penalized one time step by the remaining uncertainty.
This is not the case when the goal is far from the visible cells, which explains back and forth movements observed in \Cshref{fig|legibility|trajectories} to reduce uncertainties regarding the actual goal before following a hidden path to reach that goal.

It must also be noted that, even in this simple problem, the PO-OAMDP policy performs better than the naive $\piobs$ policy.

\subsection{Explicability}
\label{app|XP|explicability}

\begin{figure}
	\begin{minipage}{\mpwidth}
		\centering
		\adjustbox{width=\aboxwidth}{
			\begin{tikzpicture}
\draw (0,0) grid (7,7);
\draw[fill=gray] (0,0) rectangle (1,1);
\draw[fill=gray] (1,0) rectangle (2,1);
\draw[fill=gray] (2,0) rectangle (3,1);
\draw[fill=gray] (3,0) rectangle (4,1);
\draw[fill=gray] (4,0) rectangle (5,1);
\draw[fill=gray] (5,0) rectangle (6,1);
\draw[fill=gray] (6,0) rectangle (7,1);
\draw[fill=gray] (0,1) rectangle (1,2);
\draw[fill=cyan!50] (1,1) rectangle (2,2);
\draw[fill=cyan!50] (2,1) rectangle (3,2);
\draw[fill=cyan!50] (3,1) rectangle (4,2);
\draw[fill=cyan!50] (4,1) rectangle (5,2);
\draw[fill=cyan!50] (5,1) rectangle (6,2);
\draw[fill=gray] (6,1) rectangle (7,2);
\draw[fill=gray] (0,2) rectangle (1,3);
\draw[fill=cyan!50] (1,2) rectangle (2,3);
\draw[fill=cyan!50] (2,2) rectangle (3,3);
\draw[fill=cyan!50] (3,2) rectangle (4,3);
\draw[fill=cyan!50] (4,2) rectangle (5,3);
\draw[fill=cyan!50] (5,2) rectangle (6,3);
\draw[fill=gray] (6,2) rectangle (7,3);
\draw[fill=gray] (0,3) rectangle (1,4);
\draw[fill=gray] (6,3) rectangle (7,4);
\draw[fill=gray] (0,4) rectangle (1,5);
\draw[fill=cyan!50] (1,4) rectangle (2,5);
\draw[fill=cyan!50] (2,4) rectangle (3,5);
\draw[fill=cyan!50] (3,4) rectangle (4,5);
\draw[fill=cyan!50] (4,4) rectangle (5,5);
\draw[fill=cyan!50] (5,4) rectangle (6,5);
\draw[fill=gray] (6,4) rectangle (7,5);
\draw[fill=gray] (0,5) rectangle (1,6);
\draw[fill=cyan!50] (5,5) rectangle (6,6);
\draw[fill=cyan!50] (3,5) rectangle (4,6);
\draw[fill=green] (1.500000,5.500000) circle (0.500000);
\draw[fill=cyan!50] (2,5) rectangle (3,6);
\begin{scope}[shift={(3.500000 ,5)}] \draw[fill=pink,rotate=45] rectangle (0.707, 0.707);
\end{scope}
\draw[fill=cyan!50] (4,5) rectangle (5,6);
\begin{scope}[shift={(5.500000 ,5)}] \draw[fill=pink,rotate=45] rectangle (0.707, 0.707);
\end{scope}
\draw[fill=gray] (6,5) rectangle (7,6);
\draw[fill=gray] (0,6) rectangle (1,7);
\draw[fill=gray] (1,6) rectangle (2,7);
\draw[fill=gray] (2,6) rectangle (3,7);
\draw[fill=gray] (3,6) rectangle (4,7);
\draw[fill=gray] (4,6) rectangle (5,7);
\draw[fill=gray] (5,6) rectangle (6,7);
\draw[fill=gray] (6,6) rectangle (7,7);
\node [align=center] at(-0.500000,0.500000) {0};
\node [align=center] at(-0.500000,1.500000) {1};
\node [align=center] at(-0.500000,2.500000) {2};
\node [align=center] at(-0.500000,3.500000) {3};
\node [align=center] at(-0.500000,4.500000) {4};
\node [align=center] at(-0.500000,5.500000) {5};
\node [align=center] at(-0.500000,6.500000) {6};
\node [align=center] at(0.500000,-0.500000) {A};
\node [align=center] at(1.500000,-0.500000) {B};
\node [align=center] at(2.500000,-0.500000) {C};
\node [align=center] at(3.500000,-0.500000) {D};
\node [align=center] at(4.500000,-0.500000) {E};
\node [align=center] at(5.500000,-0.500000) {F};
\node [align=center] at(6.500000,-0.500000) {G};
\draw[-Latex,draw=black!100] (1.500000,5.500000)--(1.100000,5.500000);
\draw[-Latex,draw=black!100] (1.500000,5.500000)--(1.500000,5.900000);
\draw[-Latex,draw=black!100] (1.500000,5.500000)--(1.900000,5.500000);
\draw[-Latex,draw=black!100] (1.500000,5.500000)--(1.500000,5.100000);
\draw[-Latex,draw=black!100] (2.500000,5.500000)--(2.100000,5.500000);
\draw[-Latex,draw=black!1] (2.500000,5.500000)--(2.900000,5.500000);
\draw[-Latex,draw=black!1] (2.500000,5.500000)--(2.500000,5.100000);
\draw[-Latex,draw=black!100] (3.500000,5.500000)--(3.100000,5.500000);
\draw[-Latex,draw=black!2] (3.500000,5.500000)--(3.900000,5.500000);
\draw[-Latex,draw=black!2] (3.500000,5.500000)--(3.500000,5.100000);
\draw[-Latex,draw=black!100] (4.500000,5.500000)--(4.100000,5.500000);
\draw[-Latex,draw=black!2] (4.500000,5.500000)--(4.900000,5.500000);
\draw[-Latex,draw=black!2] (4.500000,5.500000)--(4.500000,5.100000);
\draw[-Latex,draw=black!100] (5.500000,5.500000)--(5.100000,5.500000);
\draw[-Latex,draw=black!2] (5.500000,5.500000)--(5.500000,5.100000);
\draw[-Latex,draw=black!100] (1.500000,4.500000)--(1.500000,4.900000);
\draw[-Latex,draw=black!1] (1.500000,4.500000)--(1.900000,4.500000);
\draw[-Latex,draw=black!1] (1.500000,4.500000)--(1.500000,4.100000);
\draw[-Latex,draw=black!100] (2.500000,4.500000)--(2.100000,4.500000);
\draw[-Latex,draw=black!100] (2.500000,4.500000)--(2.500000,4.900000);
\draw[-Latex,draw=black!100] (3.500000,4.500000)--(3.100000,4.500000);
\draw[-Latex,draw=black!100] (3.500000,4.500000)--(3.500000,4.900000);
\draw[-Latex,draw=black!100] (4.500000,4.500000)--(4.100000,4.500000);
\draw[-Latex,draw=black!100] (4.500000,4.500000)--(4.500000,4.900000);
\draw[-Latex,draw=black!100] (5.500000,4.500000)--(5.100000,4.500000);
\draw[-Latex,draw=black!100] (5.500000,4.500000)--(5.500000,4.900000);
\draw[-Latex,draw=black!100] (1.500000,3.500000)--(1.500000,3.900000);
\draw[-Latex,draw=black!2] (1.500000,3.500000)--(1.900000,3.500000);
\draw[-Latex,draw=black!2] (1.500000,3.500000)--(1.500000,3.100000);
\draw[-Latex,draw=black!100] (2.500000,3.500000)--(2.100000,3.500000);
\draw[-Latex,draw=black!100] (2.500000,3.500000)--(2.500000,3.900000);
\draw[-Latex,draw=black!100] (3.500000,3.500000)--(3.100000,3.500000);
\draw[-Latex,draw=black!100] (3.500000,3.500000)--(3.500000,3.900000);
\draw[-Latex,draw=black!100] (4.500000,3.500000)--(4.100000,3.500000);
\draw[-Latex,draw=black!100] (4.500000,3.500000)--(4.500000,3.900000);
\draw[-Latex,draw=black!100] (5.500000,3.500000)--(5.100000,3.500000);
\draw[-Latex,draw=black!100] (5.500000,3.500000)--(5.500000,3.900000);
\draw[-Latex,draw=black!100] (1.500000,2.500000)--(1.500000,2.900000);
\draw[-Latex,draw=black!2] (1.500000,2.500000)--(1.900000,2.500000);
\draw[-Latex,draw=black!2] (1.500000,2.500000)--(1.500000,2.100000);
\draw[-Latex,draw=black!100] (2.500000,2.500000)--(2.100000,2.500000);
\draw[-Latex,draw=black!100] (2.500000,2.500000)--(2.500000,2.900000);
\draw[-Latex,draw=black!100] (3.500000,2.500000)--(3.100000,2.500000);
\draw[-Latex,draw=black!100] (3.500000,2.500000)--(3.500000,2.900000);
\draw[-Latex,draw=black!100] (4.500000,2.500000)--(4.100000,2.500000);
\draw[-Latex,draw=black!100] (4.500000,2.500000)--(4.500000,2.900000);
\draw[-Latex,draw=black!100] (5.500000,2.500000)--(5.100000,2.500000);
\draw[-Latex,draw=black!100] (5.500000,2.500000)--(5.500000,2.900000);
\draw[-Latex,draw=black!100] (1.500000,1.500000)--(1.500000,1.900000);
\draw[-Latex,draw=black!2] (1.500000,1.500000)--(1.900000,1.500000);
\draw[-Latex,draw=black!100] (2.500000,1.500000)--(2.100000,1.500000);
\draw[-Latex,draw=black!100] (2.500000,1.500000)--(2.500000,1.900000);
\draw[-Latex,draw=black!100] (3.500000,1.500000)--(3.100000,1.500000);
\draw[-Latex,draw=black!100] (3.500000,1.500000)--(3.500000,1.900000);
\draw[-Latex,draw=black!100] (4.500000,1.500000)--(4.100000,1.500000);
\draw[-Latex,draw=black!100] (4.500000,1.500000)--(4.500000,1.900000);
\draw[-Latex,draw=black!100] (5.500000,1.500000)--(5.100000,1.500000);
\draw[-Latex,draw=black!100] (5.500000,1.500000)--(5.500000,1.900000);
\end{tikzpicture}
		}
	\end{minipage}
	\hfill
	\begin{minipage}{\mpwidth}
		$V=-1.96$
	\end{minipage}

	\begin{minipage}{\mpwidth}
		\centering
		a) Softmax policy $\piobs$
	\end{minipage}

	\medskip

	\begin{minipage}{\mpwidth}
		\centering
		\adjustbox{width=\aboxwidth}{
			\begin{tikzpicture}
\draw (0,0) grid (7,7);
\draw[fill=gray] (0,0) rectangle (1,1);
\draw[fill=gray] (1,0) rectangle (2,1);
\draw[fill=gray] (2,0) rectangle (3,1);
\draw[fill=gray] (3,0) rectangle (4,1);
\draw[fill=gray] (4,0) rectangle (5,1);
\draw[fill=gray] (5,0) rectangle (6,1);
\draw[fill=gray] (6,0) rectangle (7,1);
\draw[fill=gray] (0,1) rectangle (1,2);
\draw[fill=cyan!50] (1,1) rectangle (2,2);
\draw[fill=cyan!50] (2,1) rectangle (3,2);
\draw[fill=cyan!50] (3,1) rectangle (4,2);
\draw[fill=cyan!50] (4,1) rectangle (5,2);
\draw[fill=cyan!50] (5,1) rectangle (6,2);
\draw[fill=gray] (6,1) rectangle (7,2);
\draw[fill=gray] (0,2) rectangle (1,3);
\draw[fill=cyan!50] (1,2) rectangle (2,3);
\draw[fill=cyan!50] (2,2) rectangle (3,3);
\draw[fill=cyan!50] (3,2) rectangle (4,3);
\draw[fill=cyan!50] (4,2) rectangle (5,3);
\draw[fill=cyan!50] (5,2) rectangle (6,3);
\draw[fill=gray] (6,2) rectangle (7,3);
\draw[fill=gray] (0,3) rectangle (1,4);
\draw[fill=gray] (6,3) rectangle (7,4);
\draw[fill=gray] (0,4) rectangle (1,5);
\draw[fill=cyan!50] (1,4) rectangle (2,5);
\draw[fill=cyan!50] (2,4) rectangle (3,5);
\draw[fill=cyan!50] (3,4) rectangle (4,5);
\draw[fill=cyan!50] (4,4) rectangle (5,5);
\draw[fill=cyan!50] (5,4) rectangle (6,5);
\draw[fill=cyan!50] (5,5) rectangle (6,6);
\draw[fill=cyan!50] (3,5) rectangle (4,6);
\draw[fill=gray] (6,4) rectangle (7,5);
\draw[fill=gray] (0,5) rectangle (1,6);
\draw[fill=green] (1.500000,5.500000) circle (0.500000);
\draw[fill=cyan!50] (2,5) rectangle (3,6);
\begin{scope}[shift={(3.500000 ,5)}] \draw[fill=pink,rotate=45] rectangle (0.707, 0.707);
\end{scope}
\draw[fill=cyan!50] (4,5) rectangle (5,6);
\begin{scope}[shift={(5.500000 ,5)}] \draw[fill=pink,rotate=45] rectangle (0.707, 0.707);
\end{scope}
\draw[fill=gray] (6,5) rectangle (7,6);
\draw[fill=gray] (0,6) rectangle (1,7);
\draw[fill=gray] (1,6) rectangle (2,7);
\draw[fill=gray] (2,6) rectangle (3,7);
\draw[fill=gray] (3,6) rectangle (4,7);
\draw[fill=gray] (4,6) rectangle (5,7);
\draw[fill=gray] (5,6) rectangle (6,7);
\draw[fill=gray] (6,6) rectangle (7,7);
\node [align=center] at(-0.500000,0.500000) {0};
\node [align=center] at(-0.500000,1.500000) {1};
\node [align=center] at(-0.500000,2.500000) {2};
\node [align=center] at(-0.500000,3.500000) {3};
\node [align=center] at(-0.500000,4.500000) {4};
\node [align=center] at(-0.500000,5.500000) {5};
\node [align=center] at(-0.500000,6.500000) {6};
\node [align=center] at(0.500000,-0.500000) {A};
\node [align=center] at(1.500000,-0.500000) {B};
\node [align=center] at(2.500000,-0.500000) {C};
\node [align=center] at(3.500000,-0.500000) {D};
\node [align=center] at(4.500000,-0.500000) {E};
\node [align=center] at(5.500000,-0.500000) {F};
\node [align=center] at(6.500000,-0.500000) {G};
\draw[line width=1mm] (3.500000,1.500000) circle (0.300000);
\draw[-Latex] (3.500000,1.100000)--(3.500000,1.900000);
\draw[-Latex] (3.500000,2.100000)--(3.500000,2.900000);
\draw[-Latex] (3.500000,3.100000)--(3.500000,3.900000);
\draw[-Latex] (3.500000,4.100000)--(3.500000,4.900000);
\draw[-Latex] (3.900000,5.500000)--(3.100000,5.500000);
\draw[-Latex] (2.900000,5.500000)--(2.100000,5.500000);

\node at (1.500000,5.500000) {\huge$\psi_1$};
\node at (3.500000,5.500000) {\huge$\psi_2$};
\node at (5.500000,5.500000) {\huge$\psi_3$};

\end{tikzpicture}
		}
	\end{minipage}
	\hfill
	\begin{minipage}{\mpwidth}
		$V= -0.75$

		\medskip

		\centering
		\adjustbox{width=\aboxwidth}{
			\begin{tikzpicture}
\begin{axis}[
ymin=0,
ymax=1,
stack plots=y,
area style,
enlarge x limits=false,
legend pos=south east
]
\addplot coordinates
{(0,0.250000) (1,0.250000) (2,0.200993) (3,0.152023) (4,0.305880) (5,0.312895) }
\closedcycle;
\addlegendentry{$\psi_1$}\addplot coordinates
{(0,0.250000) (1,0.250000) (2,0.535752) (3,0.648862) (4,0.308291) (5,0.301539) }
\closedcycle;
\addlegendentry{$\psi_2$}\addplot coordinates
{(0,0.250000) (1,0.250000) (2,0.200993) (3,0.152023) (4,0.305880) (5,0.312895) }
\closedcycle;
\addlegendentry{$\psi_3$}\addplot coordinates
{(0,0.250000) (1,0.250000) (2,0.062261) (3,0.047091) (4,0.079949) (5,0.072670) }
\closedcycle;
\addlegendentry{$\psi_0 $}\end{axis}

\begin{axis}[
ymin=0,
ymax=1,
xmin=0,
xmax=5,
stack plots=y,
area style,
enlarge x limits=false,
legend pos=south east,
  axis x line=none,
]
\textBelief{D3}{1}{5}
\ObservationLine[dashed]{1}
\end{axis}

\end{tikzpicture}
		}
	\end{minipage}

	\begin{minipage}{\mpwidth}
		\centering
		b)   PO-OASSP  policy $\pi$
	\end{minipage}
	\hfill
	\begin{minipage}{\mpwidth}
		\centering
		c) Belief evolution
	\end{minipage}

		\caption{Results for explicability
			\label{RExplicability}}
	\end{figure}
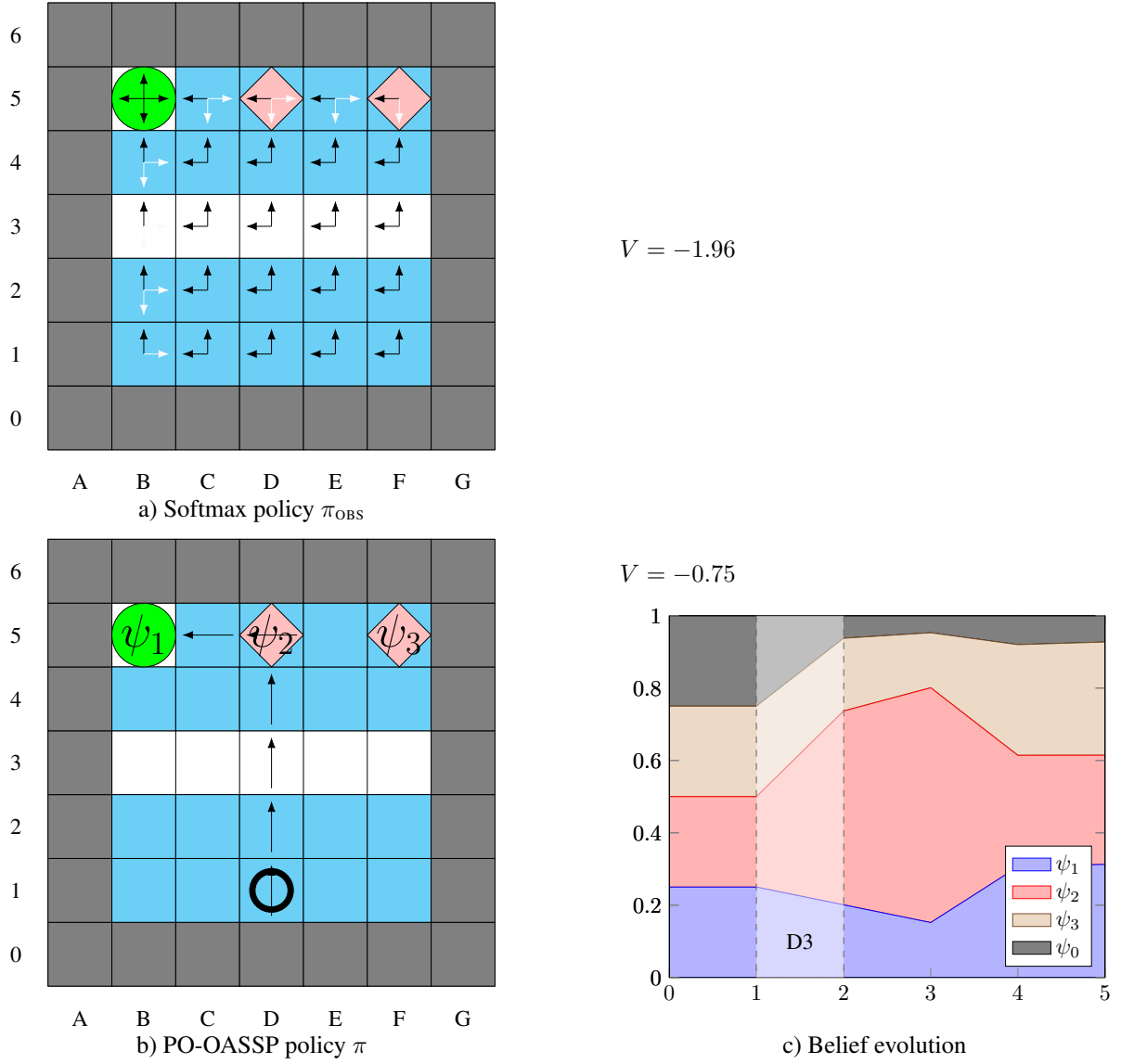

When addressing explicability in a smaller grid, the PO-OAMDP agent directly crosses the visible line and reaches its actual goal as fast as possible (\Cshref{RExplicability}).

By appearing in cell $(D,3)$, it maintains an ambiguity regarding its actual goal, but reduces the probability of the random policy $\target_0$ since the probability to reach that cell when acting randomly is lower than when trying to reach one of the goals.
Being not observed at the next time step reduces the probability of acting randomly (when acting randomly, the agent would have a $0.5$ probability to appear in a visible a cell), and also increases the belief in the middle goal (because left- and right-goal policies have a non-negligible probability to maintain in the visible row 3).

The uncertainty regarding the actual goal is maintained until the actual goal of the agent is reached and the agent is observed in a terminal state.
The fact that the agent is not observed at time step $4$ also reduces the probability of the middle goal ($\target_2$) to be the actual goal.

\subsection{Predictability}
\label{app|XP|predictability}

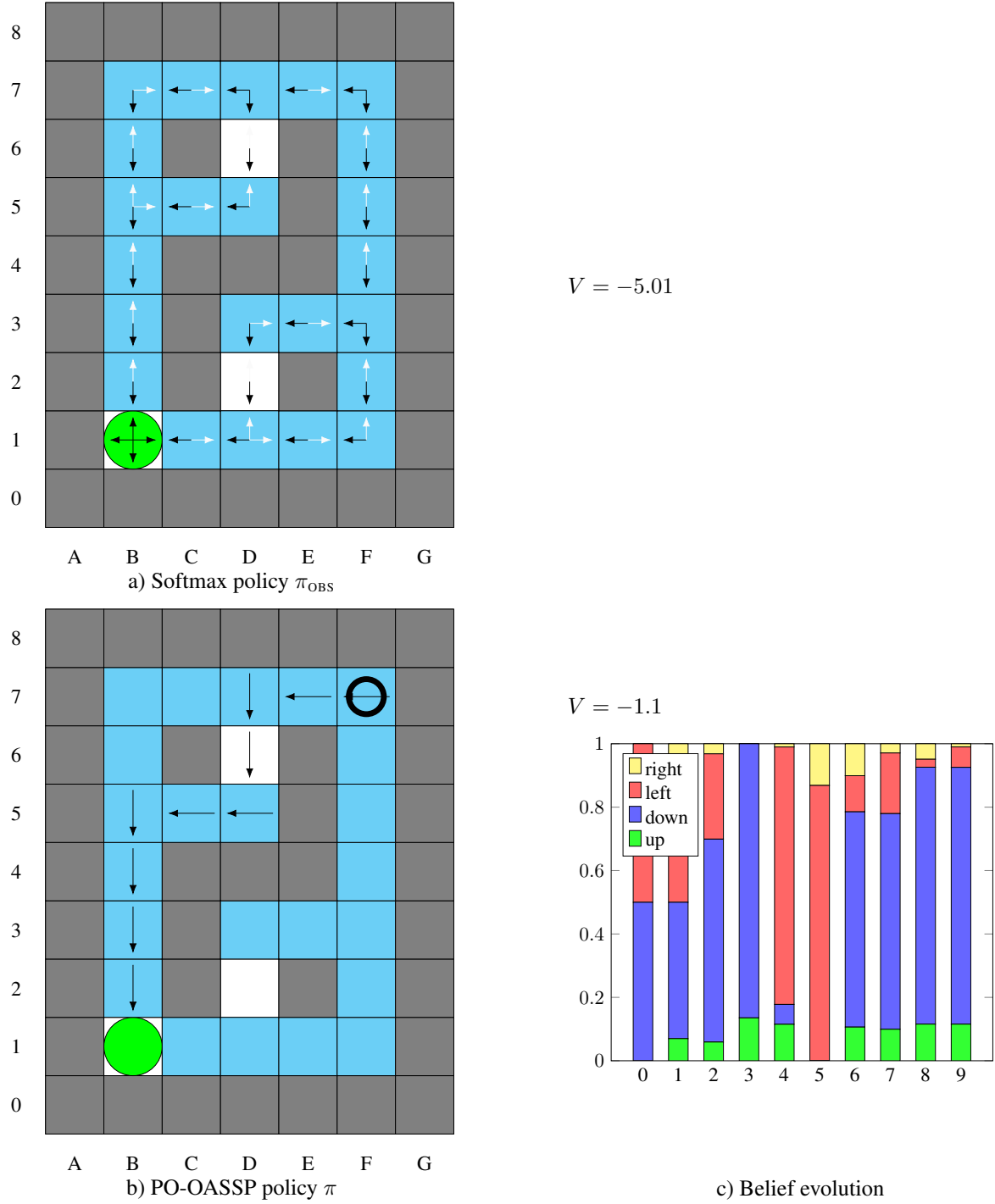
\begin{figure}
	\begin{minipage}{\mpwidth}
		\centering
		\adjustbox{width=\aboxwidth}{
			\begin{tikzpicture}
\draw (0,0) grid (7,9);
\draw[fill=gray] (0,0) rectangle (1,1);
\draw[fill=gray] (1,0) rectangle (2,1);
\draw[fill=gray] (2,0) rectangle (3,1);
\draw[fill=gray] (3,0) rectangle (4,1);
\draw[fill=gray] (4,0) rectangle (5,1);
\draw[fill=gray] (5,0) rectangle (6,1);
\draw[fill=gray] (6,0) rectangle (7,1);
\draw[fill=gray] (0,1) rectangle (1,2);
\draw[fill=green] (1.500000,1.500000) circle (0.500000);
\draw[fill=cyan!50] (2,1) rectangle (3,2);
\draw[fill=cyan!50] (3,1) rectangle (4,2);
\draw[fill=cyan!50] (4,1) rectangle (5,2);
\draw[fill=cyan!50] (5,1) rectangle (6,2);
\draw[fill=gray] (6,1) rectangle (7,2);
\draw[fill=gray] (0,2) rectangle (1,3);
\draw[fill=cyan!50] (1,2) rectangle (2,3);
\draw[fill=gray] (2,2) rectangle (3,3);
\draw[fill=gray] (4,2) rectangle (5,3);
\draw[fill=cyan!50] (5,2) rectangle (6,3);
\draw[fill=gray] (6,2) rectangle (7,3);
\draw[fill=gray] (0,3) rectangle (1,4);
\draw[fill=cyan!50] (1,3) rectangle (2,4);
\draw[fill=gray] (2,3) rectangle (3,4);
\draw[fill=cyan!50] (3,3) rectangle (4,4);
\draw[fill=cyan!50] (4,3) rectangle (5,4);
\draw[fill=cyan!50] (5,3) rectangle (6,4);
\draw[fill=gray] (6,3) rectangle (7,4);
\draw[fill=gray] (0,4) rectangle (1,5);
\draw[fill=cyan!50] (1,4) rectangle (2,5);
\draw[fill=gray] (2,4) rectangle (3,5);
\draw[fill=gray] (3,4) rectangle (4,5);
\draw[fill=gray] (4,4) rectangle (5,5);
\draw[fill=cyan!50] (5,4) rectangle (6,5);
\draw[fill=gray] (6,4) rectangle (7,5);
\draw[fill=gray] (0,5) rectangle (1,6);
\draw[fill=cyan!50] (1,5) rectangle (2,6);
\draw[fill=cyan!50] (2,5) rectangle (3,6);
\draw[fill=cyan!50] (3,5) rectangle (4,6);
\draw[fill=gray] (4,5) rectangle (5,6);
\draw[fill=cyan!50] (5,5) rectangle (6,6);
\draw[fill=gray] (6,5) rectangle (7,6);
\draw[fill=gray] (0,6) rectangle (1,7);
\draw[fill=cyan!50] (1,6) rectangle (2,7);
\draw[fill=gray] (2,6) rectangle (3,7);
\draw[fill=gray] (4,6) rectangle (5,7);
\draw[fill=cyan!50] (5,6) rectangle (6,7);
\draw[fill=gray] (6,6) rectangle (7,7);
\draw[fill=gray] (0,7) rectangle (1,8);
\draw[fill=cyan!50] (1,7) rectangle (2,8);
\draw[fill=cyan!50] (2,7) rectangle (3,8);
\draw[fill=cyan!50] (3,7) rectangle (4,8);
\draw[fill=cyan!50] (4,7) rectangle (5,8);
\draw[fill=cyan!50] (5,7) rectangle (6,8);
\draw[fill=gray] (6,7) rectangle (7,8);
\draw[fill=gray] (0,8) rectangle (1,9);
\draw[fill=gray] (1,8) rectangle (2,9);
\draw[fill=gray] (2,8) rectangle (3,9);
\draw[fill=gray] (3,8) rectangle (4,9);
\draw[fill=gray] (4,8) rectangle (5,9);
\draw[fill=gray] (5,8) rectangle (6,9);
\draw[fill=gray] (6,8) rectangle (7,9);
\node [align=center] at(-0.500000,0.500000) {0};
\node [align=center] at(-0.500000,1.500000) {1};
\node [align=center] at(-0.500000,2.500000) {2};
\node [align=center] at(-0.500000,3.500000) {3};
\node [align=center] at(-0.500000,4.500000) {4};
\node [align=center] at(-0.500000,5.500000) {5};
\node [align=center] at(-0.500000,6.500000) {6};
\node [align=center] at(-0.500000,7.500000) {7};
\node [align=center] at(-0.500000,8.500000) {8};
\node [align=center] at(0.500000,-0.500000) {A};
\node [align=center] at(1.500000,-0.500000) {B};
\node [align=center] at(2.500000,-0.500000) {C};
\node [align=center] at(3.500000,-0.500000) {D};
\node [align=center] at(4.500000,-0.500000) {E};
\node [align=center] at(5.500000,-0.500000) {F};
\node [align=center] at(6.500000,-0.500000) {G};
\draw[-Latex,draw=black!2] (1.500000,7.500000)--(1.900000,7.500000);
\draw[-Latex,draw=black!100] (1.500000,7.500000)--(1.500000,7.100000);
\draw[-Latex,draw=black!100] (2.500000,7.500000)--(2.100000,7.500000);
\draw[-Latex,draw=black!2] (2.500000,7.500000)--(2.900000,7.500000);
\draw[-Latex,draw=black!100] (3.500000,7.500000)--(3.100000,7.500000);
\draw[-Latex,draw=black!100] (3.500000,7.500000)--(3.500000,7.100000);
\draw[-Latex,draw=black!100] (4.500000,7.500000)--(4.100000,7.500000);
\draw[-Latex,draw=black!2] (4.500000,7.500000)--(4.900000,7.500000);
\draw[-Latex,draw=black!100] (5.500000,7.500000)--(5.100000,7.500000);
\draw[-Latex,draw=black!100] (5.500000,7.500000)--(5.500000,7.100000);
\draw[-Latex,draw=black!2] (1.500000,6.500000)--(1.500000,6.900000);
\draw[-Latex,draw=black!100] (1.500000,6.500000)--(1.500000,6.100000);
\draw[-Latex,draw=black!2] (3.500000,6.500000)--(3.500000,6.900000);
\draw[-Latex,draw=black!100] (3.500000,6.500000)--(3.500000,6.100000);
\draw[-Latex,draw=black!2] (5.500000,6.500000)--(5.500000,6.900000);
\draw[-Latex,draw=black!100] (5.500000,6.500000)--(5.500000,6.100000);
\draw[-Latex,draw=black!2] (1.500000,5.500000)--(1.500000,5.900000);
\draw[-Latex,draw=black!2] (1.500000,5.500000)--(1.900000,5.500000);
\draw[-Latex,draw=black!100] (1.500000,5.500000)--(1.500000,5.100000);
\draw[-Latex,draw=black!100] (2.500000,5.500000)--(2.100000,5.500000);
\draw[-Latex,draw=black!2] (2.500000,5.500000)--(2.900000,5.500000);
\draw[-Latex,draw=black!100] (3.500000,5.500000)--(3.100000,5.500000);
\draw[-Latex,draw=black!2] (3.500000,5.500000)--(3.500000,5.900000);
\draw[-Latex,draw=black!2] (5.500000,5.500000)--(5.500000,5.900000);
\draw[-Latex,draw=black!100] (5.500000,5.500000)--(5.500000,5.100000);
\draw[-Latex,draw=black!2] (1.500000,4.500000)--(1.500000,4.900000);
\draw[-Latex,draw=black!100] (1.500000,4.500000)--(1.500000,4.100000);
\draw[-Latex,draw=black!2] (5.500000,4.500000)--(5.500000,4.900000);
\draw[-Latex,draw=black!100] (5.500000,4.500000)--(5.500000,4.100000);
\draw[-Latex,draw=black!2] (1.500000,3.500000)--(1.500000,3.900000);
\draw[-Latex,draw=black!100] (1.500000,3.500000)--(1.500000,3.100000);
\draw[-Latex,draw=black!2] (3.500000,3.500000)--(3.900000,3.500000);
\draw[-Latex,draw=black!100] (3.500000,3.500000)--(3.500000,3.100000);
\draw[-Latex,draw=black!100] (4.500000,3.500000)--(4.100000,3.500000);
\draw[-Latex,draw=black!2] (4.500000,3.500000)--(4.900000,3.500000);
\draw[-Latex,draw=black!100] (5.500000,3.500000)--(5.100000,3.500000);
\draw[-Latex,draw=black!100] (5.500000,3.500000)--(5.500000,3.100000);
\draw[-Latex,draw=black!1] (1.500000,2.500000)--(1.500000,2.900000);
\draw[-Latex,draw=black!100] (1.500000,2.500000)--(1.500000,2.100000);
\draw[-Latex,draw=black!2] (3.500000,2.500000)--(3.500000,2.900000);
\draw[-Latex,draw=black!100] (3.500000,2.500000)--(3.500000,2.100000);
\draw[-Latex,draw=black!2] (5.500000,2.500000)--(5.500000,2.900000);
\draw[-Latex,draw=black!100] (5.500000,2.500000)--(5.500000,2.100000);
\draw[-Latex,draw=black!100] (1.500000,1.500000)--(1.100000,1.500000);
\draw[-Latex,draw=black!100] (1.500000,1.500000)--(1.500000,1.900000);
\draw[-Latex,draw=black!100] (1.500000,1.500000)--(1.900000,1.500000);
\draw[-Latex,draw=black!100] (1.500000,1.500000)--(1.500000,1.100000);
\draw[-Latex,draw=black!100] (2.500000,1.500000)--(2.100000,1.500000);
\draw[-Latex,draw=black!1] (2.500000,1.500000)--(2.900000,1.500000);
\draw[-Latex,draw=black!100] (3.500000,1.500000)--(3.100000,1.500000);
\draw[-Latex,draw=black!2] (3.500000,1.500000)--(3.500000,1.900000);
\draw[-Latex,draw=black!2] (3.500000,1.500000)--(3.900000,1.500000);
\draw[-Latex,draw=black!100] (4.500000,1.500000)--(4.100000,1.500000);
\draw[-Latex,draw=black!2] (4.500000,1.500000)--(4.900000,1.500000);
\draw[-Latex,draw=black!100] (5.500000,1.500000)--(5.100000,1.500000);
\draw[-Latex,draw=black!2] (5.500000,1.500000)--(5.500000,1.900000);
\end{tikzpicture}
		}
	\end{minipage}
	\hfill
	\begin{minipage}{\mpwidth}
		$V= -5.01$
	\end{minipage}

	\begin{minipage}{\mpwidth}
		\centering
		a) Softmax policy $\piobs$
	\end{minipage}

	\medskip

	\begin{minipage}{\mpwidth}
		\centering
		\adjustbox{width=\aboxwidth}{
			\begin{tikzpicture}
\draw (0,0) grid (7,9);
\draw[fill=gray] (0,0) rectangle (1,1);
\draw[fill=gray] (1,0) rectangle (2,1);
\draw[fill=gray] (2,0) rectangle (3,1);
\draw[fill=gray] (3,0) rectangle (4,1);
\draw[fill=gray] (4,0) rectangle (5,1);
\draw[fill=gray] (5,0) rectangle (6,1);
\draw[fill=gray] (6,0) rectangle (7,1);
\draw[fill=gray] (0,1) rectangle (1,2);
\draw[fill=green] (1.500000,1.500000) circle (0.500000);
\draw[fill=cyan!50] (2,1) rectangle (3,2);
\draw[fill=cyan!50] (3,1) rectangle (4,2);
\draw[fill=cyan!50] (4,1) rectangle (5,2);
\draw[fill=cyan!50] (5,1) rectangle (6,2);
\draw[fill=gray] (6,1) rectangle (7,2);
\draw[fill=gray] (0,2) rectangle (1,3);
\draw[fill=cyan!50] (1,2) rectangle (2,3);
\draw[fill=gray] (2,2) rectangle (3,3);
\draw[fill=gray] (4,2) rectangle (5,3);
\draw[fill=cyan!50] (5,2) rectangle (6,3);
\draw[fill=gray] (6,2) rectangle (7,3);
\draw[fill=gray] (0,3) rectangle (1,4);
\draw[fill=cyan!50] (1,3) rectangle (2,4);
\draw[fill=gray] (2,3) rectangle (3,4);
\draw[fill=cyan!50] (3,3) rectangle (4,4);
\draw[fill=cyan!50] (4,3) rectangle (5,4);
\draw[fill=cyan!50] (5,3) rectangle (6,4);
\draw[fill=gray] (6,3) rectangle (7,4);
\draw[fill=gray] (0,4) rectangle (1,5);
\draw[fill=cyan!50] (1,4) rectangle (2,5);
\draw[fill=gray] (2,4) rectangle (3,5);
\draw[fill=gray] (3,4) rectangle (4,5);
\draw[fill=gray] (4,4) rectangle (5,5);
\draw[fill=cyan!50] (5,4) rectangle (6,5);
\draw[fill=gray] (6,4) rectangle (7,5);
\draw[fill=gray] (0,5) rectangle (1,6);
\draw[fill=cyan!50] (1,5) rectangle (2,6);
\draw[fill=cyan!50] (2,5) rectangle (3,6);
\draw[fill=cyan!50] (3,5) rectangle (4,6);
\draw[fill=gray] (4,5) rectangle (5,6);
\draw[fill=cyan!50] (5,5) rectangle (6,6);
\draw[fill=gray] (6,5) rectangle (7,6);
\draw[fill=gray] (0,6) rectangle (1,7);
\draw[fill=cyan!50] (1,6) rectangle (2,7);
\draw[fill=gray] (2,6) rectangle (3,7);
\draw[fill=gray] (4,6) rectangle (5,7);
\draw[fill=cyan!50] (5,6) rectangle (6,7);
\draw[fill=gray] (6,6) rectangle (7,7);
\draw[fill=gray] (0,7) rectangle (1,8);
\draw[fill=cyan!50] (1,7) rectangle (2,8);
\draw[fill=cyan!50] (2,7) rectangle (3,8);
\draw[fill=cyan!50] (3,7) rectangle (4,8);
\draw[fill=cyan!50] (4,7) rectangle (5,8);
\draw[fill=cyan!50] (5,7) rectangle (6,8);
\draw[fill=gray] (6,7) rectangle (7,8);
\draw[fill=gray] (0,8) rectangle (1,9);
\draw[fill=gray] (1,8) rectangle (2,9);
\draw[fill=gray] (2,8) rectangle (3,9);
\draw[fill=gray] (3,8) rectangle (4,9);
\draw[fill=gray] (4,8) rectangle (5,9);
\draw[fill=gray] (5,8) rectangle (6,9);
\draw[fill=gray] (6,8) rectangle (7,9);
\node [align=center] at(-0.500000,0.500000) {0};
\node [align=center] at(-0.500000,1.500000) {1};
\node [align=center] at(-0.500000,2.500000) {2};
\node [align=center] at(-0.500000,3.500000) {3};
\node [align=center] at(-0.500000,4.500000) {4};
\node [align=center] at(-0.500000,5.500000) {5};
\node [align=center] at(-0.500000,6.500000) {6};
\node [align=center] at(-0.500000,7.500000) {7};
\node [align=center] at(-0.500000,8.500000) {8};
\node [align=center] at(0.500000,-0.500000) {A};
\node [align=center] at(1.500000,-0.500000) {B};
\node [align=center] at(2.500000,-0.500000) {C};
\node [align=center] at(3.500000,-0.500000) {D};
\node [align=center] at(4.500000,-0.500000) {E};
\node [align=center] at(5.500000,-0.500000) {F};
\node [align=center] at(6.500000,-0.500000) {G};
\draw[line width=1mm] (5.500000,7.500000) circle (0.300000);
\draw[-Latex] (5.900000,7.500000)--(5.100000,7.500000);
\draw[-Latex] (4.900000,7.500000)--(4.100000,7.500000);
\draw[-Latex] (3.500000,7.900000)--(3.500000,7.100000);
\draw[-Latex] (3.500000,6.900000)--(3.500000,6.100000);
\draw[-Latex] (3.900000,5.500000)--(3.100000,5.500000);
\draw[-Latex] (2.900000,5.500000)--(2.100000,5.500000);
\draw[-Latex] (1.500000,5.900000)--(1.500000,5.100000);
\draw[-Latex] (1.500000,4.900000)--(1.500000,4.100000);
\draw[-Latex] (1.500000,3.900000)--(1.500000,3.100000);
\draw[-Latex] (1.500000,2.900000)--(1.500000,2.100000);
\end{tikzpicture}
		}
	\end{minipage}
	\hfill
	\begin{minipage}{\mpwidth}
		$V= -1.1$

		\medskip

		\centering
		\adjustbox{width=\aboxwidth}{
			\pgfplotstableread{
Label up down left right
0 2.046671543220787E-44 0.5 0.5 2.046671543220787E-44 
1 0.06984306428702557 0.4301569357129744 0.4301569357129744 0.06984306428702557 
2 0.059143991786248865 0.6400569456493197 0.26904400172259413 0.0317550608418373 
3 0.13530913145380427 0.8646908685461958 3.4428359316920585E-44 3.4428359316920585E-44 
4 0.11511803484075357 0.06266016793269948 0.8122330016381417 0.009988795588405335 
5 4.357466815965515E-44 3.541874241262539E-44 0.868599746491959 0.13140025350804083 
6 0.10629521496297636 0.6790382070756311 0.11396046419335704 0.10070611376803573 
7 0.09941602808266961 0.6802978799577257 0.19133342401113798 0.028952667948466567 
8 0.11561226389275903 0.8104459397840177 0.024976329771565133 0.04896546655165821 
9 0.11550774296368418 0.8101321795266003 0.06458615463454781 0.009773922875167932 
}\testdata\begin{tikzpicture}
\begin{axis}[
ybar stacked,
 ymin=0,
ymax=1,
 xtick=data,
legend style={cells={anchor=west}, legend pos=north west},
reverse legend=true,
xticklabels from table={\testdata}{Label},
xticklabel style={text width=2cm,align=center},
]
\addplot [fill=green!80] table [y=up, meta=Label, x expr=\coordindex] {\testdata};
\addlegendentry{up}
\addplot [fill=blue!60] table [y=down, meta=Label, x expr=\coordindex] {\testdata};
\addlegendentry{down}
\addplot [fill=red!60] table [y=left, meta=Label, x expr=\coordindex] {\testdata};
\addlegendentry{left}
\addplot [fill=yellow!60] table [y=right, meta=Label, x expr=\coordindex] {\testdata};
\addlegendentry{right}
\end{axis}\end{tikzpicture}
		}
	\end{minipage}

	\begin{minipage}{\mpwidth}
		\centering
		b)   PO-OASSP  policy $\pi$
	\end{minipage}
	\hfill
	\begin{minipage}{\mpwidth}
		\centering
		c)  Belief evolution
	\end{minipage}

		\caption{Results for action predictability
			\label{RPredictability}}
	\end{figure}

The predictability task described in this section (\Cshref{RPredictability}) is simpler than the one in the body of the article  depicted in \Cshref{fig|predictability|trajectories}.

In the maze presented in \Cref{sec|MazeProblems}, the agent had to cross an empty area, which induces a lot of uncertainties regarding its next action (as many trajectories cross that area) or its actual state (which depends on how he decided to cross the area).
It required essentially adding a small negative reward $\delta$ (through $\Robs$) to ensure that the problem corresponds to a valid SSP and the agent reaches its actual goal (otherwise, the agent might get stuck in a situation where its next action is the most probable one for the observer, but its state does not change, having a cumulated $0$ value).

In the setting of \Cshref{RPredictability}, once the agent has appeared in $(D,6)$, the remainder of its trajectory (as modeled by the observer through $\piobs$) is much less ambiguous, except for randomly sampled sub-optimal moves.
There is thus no need for an additional negative reward.

\section{More Example Scenarios}
\label{app|moreExampleScenarios}

\subsection{Obfuscation}

Opposite problems can also be considered.
The agent then attempts to hide information. 
It may, for instance, have multiple possible goals, and try to not reveal its actual goal to the observer.
Obfuscation in the PO-OAMDP setting presents the same difficulties as in the OAMDP setting:
\begin{itemize}
\item if the objective is only about obfuscating information, but not on achieving a task, then the agent may simply not do anything to hide its goal, and
\item to derive the observer's belief over the goals, one assumes that the observer does not known that the agent may be trying to trick her.
\end{itemize}
Relaxing the last assumption would typically require considering a game-theoretic setting, which is out of scope of this paper.

\subsection{Broadening to Other Types of Problems}

For now we have discussed problems already modeled in the OAMDP and p-OAMDP frameworks by considering the observer's partial observability.
Yet, PO-OAMDPs allow modeling other problems in which the agent will not try to exhibit a legible, explicable or predictible behavior (for instance), but could attempt to convey as much information as possible about the state of the world currently watched by the observer.

\paragraph{Scenario \#1:}

Let us consider an office-like environment (see \Cref{portes}) with doors that are either opened or locked up, and an agent trying to let an external observer know the state of the doors through its actions.
This can of courses ve achieved by opening the  doors visible to the observer, but also by showing up in certain zones that can be reached only by opening certain doors.
Thus, even if these doors are never seen by the observer, the agent's presence may allow inferring that some doors are open.

\begin{figure}
\begin{center}
\resizebox{0.48\columnwidth}{!}{
	\begin{tikzpicture}[xscale=1,yscale=1]

\tikzstyle{mur}=[fill, draw=white]
\tikzstyle{porte}=[fill=brown, draw=black]
\tikzstyle{test}=[fill=green, draw=white]
\tikzstyle{hidden}=[fill=black!40, opacity=0.5, draw=black]

\def\t{8}

\foreach \x in {0,...,\t} {
  \draw[mur] (0,\x) rectangle (1,\x+1);
  \draw[mur] (\t,\x) rectangle (\t+1,\x+1);
  \draw[mur] (\x,0) rectangle (\x+1,1);
  \draw[mur] (\x,\t) rectangle (\x+1,\t+1);
}

\draw[mur] (1,3) rectangle (2,4);
\draw[mur] (3,3) rectangle (4,4);
\draw[mur] (4,3) rectangle (5,4);
\draw[mur] (6,3) rectangle (7,4);

\draw[mur] (4,4) rectangle (5,5);
\draw[mur] (4,6) rectangle (5,7);
\draw[mur] (5,6) rectangle (6,7);
\draw[mur] (6,6) rectangle (7,7);
\draw[mur] (6,5) rectangle (7,6);
\draw[mur] (6,4) rectangle (7,5);

\draw[mur] (4,3) rectangle (5,2);
\draw[mur] (4,2) rectangle (5,1);

\draw[porte] (2,3.3) rectangle (3,3.7);
\node (d1) at (2.5,4) {\LARGE \texttt{d1}};


\draw[porte] (4.3,7) rectangle (4.7,8);
\node (d3) at (5.1,7.5) {\LARGE \texttt{d2}};

\draw[hidden] (0,3) rectangle (\t+1,\t+1);

\node at (2.5,1.5) {\LARGE \texttt{A}};
\node at (5.5,2.5) {\LARGE \texttt{B}};
\node at (7.5,2.5) {\LARGE \texttt{C}};
\node at (6.5,1.5) {\LARGE \texttt{O}};

\end{tikzpicture}
}
\hfill
\resizebox{0.48\columnwidth}{!}{
	\begin{tikzpicture}[xscale=1,yscale=1]

\tikzstyle{mur}=[fill, draw=white]
\tikzstyle{porte}=[fill=brown, draw=black]
\tikzstyle{porteO}=[fill=white, draw=black]
\tikzstyle{test}=[fill=green, draw=white]
\tikzstyle{hidden}=[fill=black!40, opacity=0.5, draw=black]

\def\t{8}

\foreach \x in {0,...,\t} {
  \draw[mur] (0,\x) rectangle (1,\x+1);
  \draw[mur] (\t,\x) rectangle (\t+1,\x+1);
  \draw[mur] (\x,0) rectangle (\x+1,1);
  \draw[mur] (\x,\t) rectangle (\x+1,\t+1);
}

\draw[mur] (1,3) rectangle (2,4);
\draw[mur] (3,3) rectangle (4,4);
\draw[mur] (4,3) rectangle (5,4);
\draw[mur] (6,3) rectangle (7,4);

\draw[mur] (4,4) rectangle (5,5);
\draw[mur] (4,6) rectangle (5,7);
\draw[mur] (5,6) rectangle (6,7);
\draw[mur] (6,6) rectangle (7,7);
\draw[mur] (6,5) rectangle (7,6);
\draw[mur] (6,4) rectangle (7,5);

\draw[mur] (4,3) rectangle (5,2);
\draw[mur] (4,2) rectangle (5,1);

\draw[porteO] (2,3.3) rectangle (3,3.7);
\node (d1) at (2.5,4) {\LARGE \texttt{d1}};


\draw[porteO] (4.3,7) rectangle (4.7,8);
\node (d2) at (5.1,7.5) {\LARGE \texttt{d2}};

\draw[hidden] (0,3) rectangle (\t+1,\t+1);

\node (B) at (5.5,2.5) {\LARGE \texttt{B}};
\node (A) at (2.5,1.5) {\LARGE \texttt{A}};
\node (C) at (7.5,2.5) {\LARGE \texttt{C}};
\node at (6.5,1.5) {\LARGE \texttt{O}};

\draw[thick,->] (A) -- (2.5,3.2);
\draw[thick,->] (d1) -- (2.5,7.5) -- (4.2,7.5);
\draw[thick,->] (d2) -- (7.5,7.5) -- (C);

\end{tikzpicture}
}
\end{center}
\caption{
A grid environment with doors that may be locked.
Walls are represented by black cells, doors by brown rectangles.
The grey zone correspond to a zone hidden to the observer.
The agent starts in \texttt{A} and has to go to \texttt{O}.
}
\label{portes}
\end{figure}
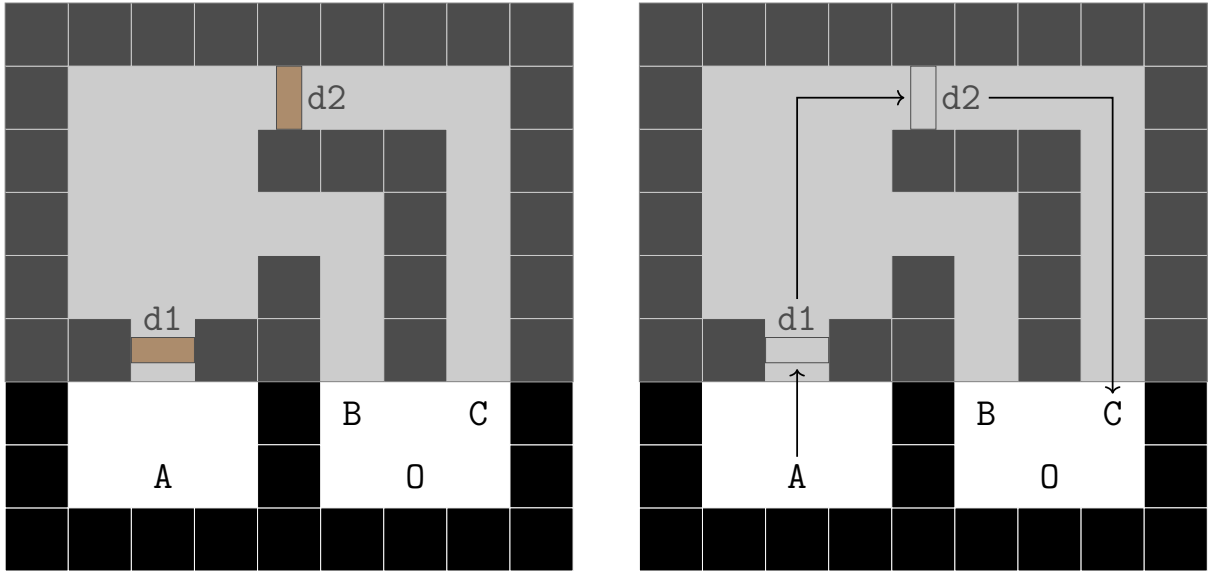

In the exemple of \Cref{portes}, by choosing a longer path in the hidden zone and reappearing in \texttt{C}, it tells the observer that doors \texttt{d1} and \texttt{d2} are not locked.
Becoming visible in \texttt{B} would allow achieving the objective, but without providing information to the observer about the state of door \texttt{d2}.
Solving this problem requires the agent to reason about
\begin{enumerate}
\item the consequences of its actions,
\item its visibility, which depends on its location,
\item the observer's possible inferences, and
\item in particular the doors whose state the agent wants the observer to know.
\end{enumerate}

\paragraph{Scenario \#2:}
In a second scenario, let us consider an agent responsible for tracking intruders in an environment an observer cannot perceive (see \Cref{intrus}).
By modeling the observer's reasoning process, the agent can leverage the observer's expectations to act, appearing in certain places and let the observer know about the presence and location of an intruder.

\begin{figure}
\begin{center}
\resizebox{0.48\columnwidth}{!}{
	\begin{tikzpicture}[xscale=1,yscale=1]

\tikzstyle{mur}=[fill, draw=white]
\tikzstyle{porte}=[fill=brown, draw=black]
\tikzstyle{test}=[fill=green, draw=white]
\tikzstyle{hidden}=[fill=black!40, opacity=0.5, draw=black]

\def\t{12}

\foreach \x in {0,...,\t} {
  \draw[mur] (0,\x) rectangle (1,\x+1);
  \draw[mur] (\t,\x) rectangle (\t+1,\x+1);
  \draw[mur] (\x,0) rectangle (\x+1,1);
  \draw[mur] (\x,\t) rectangle (\x+1,\t+1);
}

\draw[mur] (2,2) rectangle (3,3);
\draw[mur] (2,3) rectangle (3,4);
\draw[mur] (3,2) rectangle (4,3);
\draw[mur] (2,5) rectangle (3,6);
\draw[mur] (2,6) rectangle (3,7);
\draw[mur] (5,2) rectangle (6,3);
\draw[mur] (6,2) rectangle (7,3);
\draw[mur] (6,3) rectangle (7,4);

\draw[mur] (3,6) rectangle (4,7);

\draw[mur] (5,6) rectangle (6,7);
\draw[mur] (6,6) rectangle (7,7);
\draw[mur] (6,5) rectangle (7,6);


\draw[mur] (7,2) rectangle (8,3);

\draw[mur] (9,2) rectangle (10,3);
\draw[mur] (10,2) rectangle (11,3);
\draw[mur] (10,3) rectangle (11,4);
\draw[mur] (10,5) rectangle (11,6);
\draw[mur] (10,6) rectangle (11,7);
\draw[mur] (9,6) rectangle (10,7);
\draw[mur] (7,6) rectangle (8,7);

\draw[mur] (10,7) rectangle (11,8);
\draw[mur] (10,9) rectangle (11,10);
\draw[mur] (10,10) rectangle (11,11);
\draw[mur] (9,10) rectangle (10,11);
\draw[mur] (7,10) rectangle (8,11);
\draw[mur] (6,10) rectangle (7,11);
\draw[mur] (6,9) rectangle (7,10);
\draw[mur] (6,7) rectangle (7,8);

\draw[mur] (5,10) rectangle (6,11);
\draw[mur] (3,10) rectangle (4,11);
\draw[mur] (2,10) rectangle (3,11);
\draw[mur] (2,9) rectangle (3,10);
\draw[mur] (2,7) rectangle (3,8);




\draw[fill=cyan] (1.500000,1.500000) circle (0.500000);
\draw[fill=red] (5.500000,9.500000) circle (0.500000);
\draw[hidden] (2,2) rectangle (11,11);


\end{tikzpicture}
}
\hfill
\resizebox{0.48\columnwidth}{!}{
	\begin{tikzpicture}[xscale=1,yscale=1]

\tikzstyle{mur}=[fill, draw=white]
\tikzstyle{porte}=[fill=brown, draw=black]
\tikzstyle{test}=[fill=green, draw=white]
\tikzstyle{hidden}=[fill=black!40, opacity=0.5, draw=black]

\def\t{12}

\foreach \x in {0,...,\t} {
  \draw[mur] (0,\x) rectangle (1,\x+1);
  \draw[mur] (\t,\x) rectangle (\t+1,\x+1);
  \draw[mur] (\x,0) rectangle (\x+1,1);
  \draw[mur] (\x,\t) rectangle (\x+1,\t+1);
}

\draw[mur] (2,2) rectangle (3,3);
\draw[mur] (2,3) rectangle (3,4);
\draw[mur] (3,2) rectangle (4,3);
\draw[mur] (2,5) rectangle (3,6);
\draw[mur] (2,6) rectangle (3,7);
\draw[mur] (5,2) rectangle (6,3);
\draw[mur] (6,2) rectangle (7,3);
\draw[mur] (6,3) rectangle (7,4);

\draw[mur] (3,6) rectangle (4,7);

\draw[mur] (5,6) rectangle (6,7);
\draw[mur] (6,6) rectangle (7,7);
\draw[mur] (6,5) rectangle (7,6);


\draw[mur] (7,2) rectangle (8,3);

\draw[mur] (9,2) rectangle (10,3);
\draw[mur] (10,2) rectangle (11,3);
\draw[mur] (10,3) rectangle (11,4);
\draw[mur] (10,5) rectangle (11,6);
\draw[mur] (10,6) rectangle (11,7);
\draw[mur] (9,6) rectangle (10,7);
\draw[mur] (7,6) rectangle (8,7);

\draw[mur] (10,7) rectangle (11,8);
\draw[mur] (10,9) rectangle (11,10);
\draw[mur] (10,10) rectangle (11,11);
\draw[mur] (9,10) rectangle (10,11);
\draw[mur] (7,10) rectangle (8,11);
\draw[mur] (6,10) rectangle (7,11);
\draw[mur] (6,9) rectangle (7,10);
\draw[mur] (6,7) rectangle (7,8);

\draw[mur] (5,10) rectangle (6,11);
\draw[mur] (3,10) rectangle (4,11);
\draw[mur] (2,10) rectangle (3,11);
\draw[mur] (2,9) rectangle (3,10);
\draw[mur] (2,7) rectangle (3,8);




\draw[fill=cyan] (4.500000,11.500000) circle (0.500000);
\draw[fill=red] (5.500000,9.500000) circle (0.500000);
\draw[hidden] (2,2) rectangle (11,11);

\draw[thick] (1.5,1.5) -- (1.5,11.5);
\draw[thick,->] (1.5,11.5) -- (3.8,11.5);
\draw[thick,dashed,->] (4.5,10.8) -- (4.5,9.5);

\end{tikzpicture}
}
\end{center}

\caption{
Grid environment with an intruder.
Wall are represented by black cells.
The grey area corresponds to a zone hidden to the observer.
The intruder is represented by a red circle and the agent by a blue circle.
}
\label{intrus}
\end{figure}
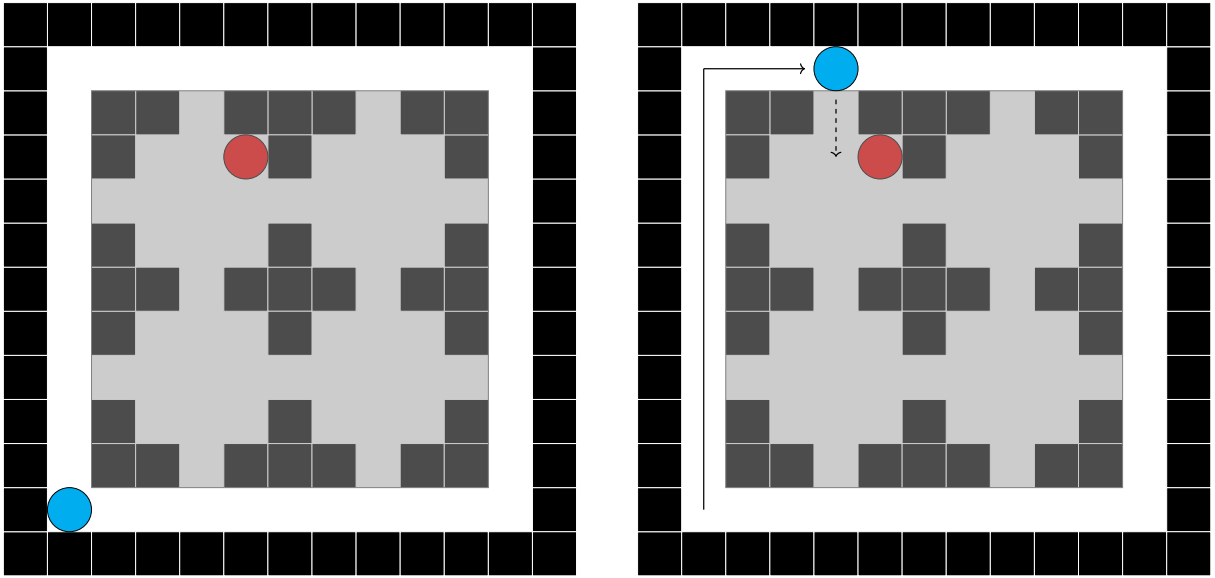

In the example illustrated by \Cref{intrus}, the observer expects the agent to try getting close to the intruder.
The agent can thus inform the observer about the intruder's location by choosing among the possible trajectories bringing as close as possible to the intruder, a path that is often visible to the observer.

\paragraph{Scenario \#3:}
Finally, in complex tasks that require achieving several intermediate sub-tasks (/objectives), the agent may try to convey information about the progress of the ongoing sub-task by following longer paths but
\begin{enumerate}
\item which are partially visible to the observer, and
\item which leave less ambiguity about its intermediate objective.
\end{enumerate}
By trying to make it easier to infer the intermediate objectives it attempts to achieve, the agent can thus transmit the status of the current sub-task, what can be crucial in a collaborative scenario (that would require in return a particular action from the human).

\medskip

These various situations show that the PO-OAMDP formalism allows broadening the family of problems covered by conveying information not only about the agent's behavior.
This framework allows modeling problems close to active information gathering, as formalized by $\rho$-POMDPs \citep{AraBufThoCha-nips10}.
In $\rho$-POMDPs, an agent partially observes its (own) environment, and has to act as well as possible to obtain relevant observations and maximize some information measure about its target variables (such as its location).
The main difference with PO-OAMDPs is that, in the latter, the agent wants to control the information acquired by a third party (the observer), not its own information (which is complete).
This requires in particular a model of the observer that the agent will take advantage of to indirectly control the observer's belief.

\end{document}